\documentclass[9pt,twoside]{extarticle}

\usepackage[papersize={17cm,24cm},margin=22mm,top=17mm,headsep=5mm]{geometry}
\usepackage[bottom]{footmisc}
\usepackage[square,numbers]{natbib}
\usepackage{listings}
\usepackage[utf8]{inputenc} 
\usepackage[T1]{fontenc}    
\usepackage{url}            
\usepackage{booktabs}       
\usepackage{amsfonts}       
\usepackage{nicefrac}       
\usepackage{microtype}      
\usepackage{multirow}
\usepackage{float}
\usepackage{graphicx}
\usepackage{dsfont}
\usepackage{bm}
\usepackage[dvipsnames]{xcolor}
\usepackage{wrapfig}
\usepackage{tikz}
\usepackage{xspace}
\usepackage{amsmath,amsfonts,amssymb,amsthm}
\usepackage{mathtools}
\usepackage{verbatim}
\usepackage[colorinlistoftodos, textwidth=26mm, shadow,color=blue!30!white]{todonotes}
\usepackage[colorlinks,linkcolor = Aquamarine, urlcolor  = violet, citecolor = YellowOrange, anchorcolor = violet]{hyperref}
\usepackage{appendix}
\usepackage{fancyhdr}
\fancypagestyle{plain}{
  \fancyhead[C,R]{}
  \fancyfoot{} 
  }

\usepackage{algorithmic}
\usepackage[ruled,vlined]{algorithm2e}
\usepackage{mathrsfs}
\usepackage{times}
\newcommand{\BEAS}{\begin{eqnarray*}}
\newcommand{\EEAS}{\end{eqnarray*}}
\newcommand{\BEA}{\begin{eqnarray}}
\newcommand{\EEA}{\end{eqnarray}}
\newcommand{\BA}{\begin{align}}
\newcommand{\EA}{\end{align}}
\newcommand{\BAS}{\begin{align*}}
\newcommand{\EAS}{\end{align*}}
\newcommand{\BEQ}{\begin{equation}}
\newcommand{\EEQ}{\end{equation}}
\newcommand{\BIT}{\begin{itemize}}
\newcommand{\EIT}{\end{itemize}}
\newcommand{\BNUM}{\begin{enumerate}}
\newcommand{\ENUM}{\end{enumerate}}

\newcommand{\E}{\mathbb{E}}
\DeclareMathOperator\Var{Var}

\newcommand{\R}{\mathbb{R}}
\DeclareMathOperator\tr{Tr}
\newcommand{\ba}{B^{-1}A}
\newcommand{\eps}{\varepsilon}
\newtheorem{theorem}{Theorem}
\newtheorem{corollary}{Corollary}
\newtheorem{lemma}{Lemma}
\newtheorem{proposition}{Proposition}
\newcommand{\mysec}[1]{Section~\ref{sec:#1}}
\newcommand{\eq}[1]{Eq.~(\ref{eq:#1})}
\newcommand{\myfig}[1]{Figure~\ref{fig:#1}}

\newcommand{\myapp}[1]{Appendix~\ref{sec:#1}}

\newcommand{\alg}{Gen-Oja\xspace}

\newcommand{\nocontentsline}[3]{}
\newcommand{\tocless}[2]{\bgroup\let\addcontentsline=\nocontentsline#1{#2}\egroup}

\title{\alg: A Two-time-scale approach for Streaming CCA}
\allowdisplaybreaks
\author{
Kush Bhatia\footnote{Equal contribution.}\phantom{$^\ast$}\\
University of California, Berkeley\\
\texttt{kushbhatia@berkeley.edu}\\
\and
Aldo Pacchiano$^\ast$\\
University of California, Berkeley\\
\texttt{pacchiano@berkeley.edu}\\
\and
Nicolas Flammarion\\
University of California, Berkeley\\
\texttt{flammarion@berkeley.edu}\\
\and
Peter L. Bartlett\\
University of California, Berkeley\\
\texttt{peter@berkeley.edu}\\
\and
Michael I. Jordan\\
\hspace{.8cm} University of California, Berkeley\\
\hspace{.8cm} \texttt{jordan@cs.berkeley.edu}\\
}
\date{}
\begin{document}
\maketitle

\begin{abstract}
In this paper, we study the problems of principal Generalized Eigenvector computation and Canonical Correlation Analysis in the stochastic setting. We propose a simple and efficient algorithm, \alg, for these problems. We prove the global convergence of our algorithm, borrowing ideas from the theory of fast-mixing Markov chains and {two-time-scale stochastic approximation}, showing that it achieves the optimal rate of convergence. In the process, we develop tools for understanding stochastic processes with Markovian noise which might be of independent interest.
\end{abstract}

\tocless\section{Introduction}
\label{sec:intro}
Cannonical Correlation Analysis (CCA) and the Generalized Eigenvalue Problem are two fundamental problems in machine learning and statistics, widely used for feature extraction in applications including regression \cite{kakade2007multi}, clustering \cite{chaudhuri2009multi} and classification \cite{karampatziakis2013discriminative}.

Originally introduced by Hotelling in~\cite{hotelling1936relations}, CCA is a statistical tool for the analysis of multi-view data that can be viewed as a ``correlation-aware" version of Principal Component Analysis (PCA). Given two multidimensional random variables, the objective in CCA is to obtain a pair of linear transformations that maximize the correlation between the transformed variables. 

Given access to samples $\{(x_i, y_i)_{i=1}^n\}$ of zero mean random variables $X, Y\in\R^{d}$ with an unknown joint distribution $P_{XY}$, CCA can be used to discover features expressing similarity or dissimilarity between $X$ and $Y$. Formally, CCA aims to find a pair of vectors $u,v \in \mathbb{R}^d$ such that projections of  $X$ onto $v$ and $Y$ onto $u$ are maximally correlated. In the population setting, the corresponding objective is given by:
\begin{equation}
\max v^\top \mathbb{E}[  XY^\top ] u\qquad \text{s.t.}\quad {{  v^\top \mathbb{E}[XX^\top ] v = 1 \text{ and } u^\top \mathbb{E}[ YY^\top ] u=1 } }.
\end{equation}

In the context of covariance matrices, the objective of the generalized eigenvalue problem is to obtain the direction $u \text{ or } v \in \mathbb{R}^d$ maximizing discrepancy between $X$ and $Y$ and can be formulated as,
\begin{equation}
\arg\max_{{v \neq 0}} \frac{ v^\top \mathbb{E}[XX^\top] v}{v^\top \mathbb{E}[YY^\top]v} \text{ and } \arg\max_{{u \neq 0}} \frac{u^\top \mathbb{E}[YY^\top]u}{u^\top \mathbb{E}[XX^\top ] u}.
\end{equation}

More generally, given symmetric matrices $A, B$, with $B$ positive definite, the objective of the principal generalized eigenvector problem is to obtain a unit norm vector $w$ such that $Aw = \lambda Bw$ for $\lambda$ maximal. 

CCA and the generalized eigenvalue problem are intimately related. In fact, the CCA problem can be cast as a special case of the generalized eigenvalue problem by solving for $u$ and $v$ in the following objective:
\begin{equation}\label{eq:ccagep}
\underbrace{\begin{pmatrix}
0 & \E[XY^\top]\\
\E[YX^\top] &0
\end{pmatrix}}_{A} \begin{pmatrix} v \\ u \end{pmatrix} = \lambda \underbrace{ \begin{pmatrix}
\E[XX^\top] & 0\\
0 &\E[YY^\top]
\end{pmatrix}}_{B}\begin{pmatrix} v \\ u \end{pmatrix}.
\end{equation}



The optimization problems underlying both CCA and the generalized eigenvector problem are non-convex in general. While they admit closed-form solutions, even in the offline setting a direct computation requires $\mathcal{O}(d^3)$ flops which is infeasible for large-scale datasets. Recently, there has been work on solving these problems by leveraging fast linear system solvers \cite{GeJinKak16, AllLi17b} while requiring complete knowledge of the matrices $A$ and $B$.

In the stochastic setting, the difficulty increases because the objective is to maximize a ratio of expectations, in contrast to the standard setting of stochastic optimization \cite{robbins1951}, where the objective is the maximization of an expectation. There has been recent interest in understanding and developing efficient algorithms with provable convergence guarantees for such non-convex problems. \cite{JaiJinKakNetSid16} and \cite{Sha16} recently analyzed the convergence rate of Oja's algorithm \cite{oja1982}, one of the most commonly used algorithm for streaming PCA.

In contrast, for the stochastic generalized eigenvalue problem and CCA problem, the focus has been to translate algorithms from the offline setting to the online one. For example, \cite{GaoGarSreWanWan17} proposes a streaming algorithm for the stochastic CCA problem which utilizes a streaming SVRG method to solve an online least-squares problem. Despite being streaming in nature, this algorithm requires a non-trivial initialization and, in contrast to the spirit of streaming algorithms, updates its eigenvector estimate only after every few samples. This raises the following challenging question:
\begin{changemargin}{0.4cm}{0.4cm}
\textit{Is it possible to obtain an efficient and provably convergent counterpart to Oja's Algorithm for computing the principal generalized eigenvector in the stochastic setting?}
\end{changemargin}
In this paper, we propose a simple, globally convergent, \emph{two-line} algorithm, \alg,  for the stochastic principal generalized eigenvector problem and, as a consequence, we obtain a natural extension of Oja's algorithm for the streaming CCA problem. \alg is an iterative algorithm which works by updating two coupled sequences at every time step. In contrast with existing methods \cite{JaiJinKakNetSid16}, at each time step the algorithm can be seen as performing a step of Oja's method, with a noise term which is neither \emph{zero mean} nor \emph{conditionally independent}, but instead is Markovian in nature. The analysis of the algorithm borrows tools from the theory of fast mixing of Markov chains \cite{dieuleveut2017} as well as two-time-scale stochastic approximation \cite{BenMetPri90,Bo1,Bo2} to obtain an optimal (up to dimension dependence) fast convergence rate of $\tilde{\mathcal{O}}(1/n)$. Our main contribution can summarized in the following informal theorem (made formal in Section~\ref{sec:main_thm}).

\paragraph{Main Result (informal).} With probability greater than $4/5$, one can obtain an $\epsilon$-accurate estimate of the generalized eigenvector in the stochastic setting using $\tilde{\mathcal{O}}(1/\epsilon)$ unbiased independent samples of the matrices. The multiplicative pre-factors depend polynomially on the inverse eigengap and the dimension of the problem.\\

\textbf{Notation}: We denote by $\lambda_i(M)$ and $\sigma_i(M)$ the $i^{th}$ largest eigenvalue and singular value of a square matrix $M$. For any positive semi-definite matrix $N$, we denote inner product in the $N$-norm by $\langle \cdot, \cdot \rangle_N$ and the corresponding norm by $\| \cdot \|_N$. We let $\kappa_N = \frac{\lambda_{\max}(N)}{\lambda_{\min}(N)}$ denote the condition number of $N$. We denote the eigenvalues of the matrix $B^{-1}A$ by ${\lambda_1 > \lambda_2 \geq \ldots \geq\lambda_d}$ with $(u_i)_{i=1}^d$ and $(\tilde{u}_i)_{i=1}^d$ denoting the corresponding right and left eigenvectors of $\ba$ whose existence is guaranteed by Lemma~\ref{lemma::right_eigenvectors_orthogonal} in \myapp{eigenprop}. We use $\Delta_\lambda$ to denote the eigengap $\lambda_1-\lambda_2$.

\tocless\section{Problem Statement}
\label{sec:prob}
In this section, we focus on the problem of estimating principal generalized eigenvectors in a stochastic setting. The generalized eigenvector, $v_i$, corresponding to a system of matrices $(A,B)$, where $A \in \R^{d\times d}$ is a symmetric matrix and $B \in \R^{d\times d}$ is a symmetric positive definite matrix, satisfies
\begin{equation}\label{eq:geneigen}
Av_i = \lambda_i Bv_i.
\end{equation}
The principal generalized eigenvector $v_1$ corresponds to the vector with the largest value\footnote{Note that we consider here the largest \emph{signed} value of $\lambda_i$} of $\lambda_i$, or, equivalently, $v_1$ is the principal eigenvector of the non-symmetric matrix $B^{-1}A$. The vector $v_1$ also corresponds to the maximizer of the generalized Rayleigh quotient given by
\vspace*{-2pt}
\begin{equation}\label{eq:geneigenRay}
v_1= \arg \max_{v\in\R^d} \frac{ v^\top A v }{v^\top B v}.
\end{equation}
\vspace*{-2pt}
In the stochastic setting, we only have access to a sequence of matrices $A_1,\dots,A_n\in \R^{d\times d}$ and $B_1,\dots,B_n\in \R^{d\times d}$ 
assumed to be drawn i.i.d. from an unknown underlying distribution, such that $\E[A_i]=A$ and $\E[B_i]=B$ and the objective is to estimate $v_1$ given access to $\mathcal{O}(d)$ memory. 

In order to quantify the error between a vector and its estimate, we define the following generalization of the sine with respect to the $B$-norm as, 
\begin{equation}
\sin^2_B(v,w)= 1-\Big(\frac{v^\top B w}{\Vert v \Vert_B\Vert w \Vert_B}\Big)^2.
\end{equation}

\tocless\section{Related Work}
\label{sec:rel}
\paragraph{PCA.} There is a vast literature dedicated to the development of computationally efficient algorithms for the PCA problem in the offline setting (see \cite{MusMus15, GarHazJinKakMusNetSid16} and references therein). In the stochastic setting, sharp convergence results were obtained recently by \cite{JaiJinKakNetSid16} and \cite{Sha16} for the principal eigenvector computation problem using Oja's algorithm and later extended to the streaming k-PCA setting by \cite{AllLi17}. They are able to obtain a $\mathcal{O}(1/n)$ convergence rate when the eigengap of the matrix is positive and a $\mathcal{O}(1/\sqrt{n})$ rate is attained in the gap free setting. 
\vspace*{-4pt}
\paragraph{Offline CCA and generalized eigenvector.} Computationally efficient optimization algorithms with finite convergence guarantees for CCA and the generalized eigenvector problem based on Empirical Risk Minimization (ERM) on a fixed dataset have recently been proposed in ~\cite{GeJinKak16,WanWanGarSre16,AllLi17b}. These approaches work by reducing the CCA and generalized eigenvector problem to that of solving a PCA problem on a modified matrix $M$ (e.g., for CCA, $M = B^{\frac{-1}{2}}AB^{\frac{-1}{2}}$). This reformulation is then solved by using an approximate version of the Power Method that relies on a linear system solver to obtain the approximate power method step. \cite{GeJinKak16, AllLi17b} propose an algorithm for the generalized eigenvector computation problem and instantiate their results for the CCA problem. \cite{LuFo14, MaLuFos15, WanWanGarSre16} focus on the CCA problem by optimizing a different objective:
\[
\min \frac{1}{2}\hat \E \vert \phi^\top x_i -\psi^\top y_i \vert^2+\lambda_x\Vert \phi\Vert_2^2+\lambda_y\Vert \psi\Vert_2^2 \quad \text{ s.t. } \quad \Vert \phi\Vert_{\hat\E[xx^\top]}   =\Vert \psi\Vert_{\hat\E[yy^\top]}=1,
\]
where $\hat \E$ denotes the empirical expectation. The proposed methods utilize the knowledge of complete data in order to solve the ERM problem, and hence is unclear how to extend them to the stochastic setting. 
\vspace*{-4pt}
\paragraph{Stochastic CCA and generalized eigenvector.} There has been a dearth of work for solving these problems in the stochastic setting owing to the difficulties mentioned in \mysec{intro}. Recently, \cite{GaoGarSreWanWan17} extend the algorithm of \cite{WanWanGarSre16} from the offline to the streaming setting by utilizing a streaming version of the SVRG algorithm for the least squares system solver. Their algorithm, based on the shift and invert method, suffers from two drawbacks: a) contrary to the spirit of streaming algorithms, this method does not update its estimate at each iteration -- it requires to use logarithmic samples for solving an online least squares problem, and, b) their algorithm critically relies on obtaining an estimate of $\lambda_1$ to a small accuracy for which it requires to burn a few samples in the process. In comparison, \alg takes a \emph{single} stochastic gradient step for the inner least squares problem and updates its estimate of the eigenvector after each sample. Perhaps the closest to our approach is \cite{AroMarMiaSre17}, who propose an online method by solving a convex relaxation of the CCA objective with an inexact stochastic mirror descent algorithm. Unfortunately, the computational complexity of their method is $\mathcal{O}(d^2)$ which renders it infeasible for large-scale problems.

\tocless\section{\alg}
\begin{algorithm}[t!]
\label{alg:spge}
	\DontPrintSemicolon
	\KwIn{Time steps $T$, step size $\alpha_t$ (Least Squares), $\beta_t$ (Oja)}
    \textbf{Initialize:} $(w_0, v_0 ) \leftarrow $ sample uniformly from the unit sphere in $\mathbb{R}^d$, $\bar{v}_0 = v_0$\;
	\For{$t = 1, \ldots, T$}{
	   Draw sample $(A_t,B_t)$\;
       $w_t \leftarrow w_{t-1} - \alpha_t (B_t w_{t-1} - A_t v_{t-1})$\;
       $ v_t'\leftarrow  v_{t-1} + \beta_t w_t$ \label{eq:updatev}\;
       $v_t \leftarrow \frac{v_t'}{\| v_t\|_2}$\;
	}
  \KwOut{Estimate of Principal Generalized Eigenvector: ${v}_T$}
	\caption{\alg for Streaming $A v = \lambda B v$}
\end{algorithm}
In this section, we describe our proposed approach for the stochastic generalized eigenvector problem (see \mysec{prob}). Our algorithm \alg, described in Algorithm \ref{alg:spge}, is a natural extension of the popular Oja's algorithm used for solving the streaming PCA problem. The algorithm proceeds by iteratively updating two coupled sequences $(w_t,v_t)$ at the same time: $w_t$ is updated using one step of stochastic gradient descent with constant step-size to minimize $ w^\top Bw - 2w^\top Av_t$ and $v_t$ is updated using a step of Oja's algorithm. \alg has its roots in the theory of two-time-scale stochastic approximation, by viewing the sequence $w_t$ as a fast mixing Markov chain and $v_t$ as a slowly evolving one. In the sequel, we describe the evolution of the Markov chains $(w_t)_{t\geq 0}, (v_t)_{t\geq 0}$, in the process outlining  the intuition underlying \alg and understanding the key challenges which arise in the convergence analysis. 
\vspace*{-4pt}
\paragraph{Oja's algorithm.} \alg is closely related to the Oja's algorithm \cite{oja1982} for the streaming PCA problem. Consider a special case of the problem, when each $B_t=I$. In the offline setting, this reduces the generalized eigenvector problem to that of computing the principal eigenvector of A. With the setting of step-size $\alpha_t=1$, \alg recovers the Oja's algorithm given by
\vspace*{-1pt}
\[
v_t=\frac{v_{t-1}+\beta_tA_t v_{t-1}}{\Vert v_{t-1}+\beta_tA_t v_{t-1}. \Vert}
\]
This algorithm is exactly a projected stochastic gradient ascent on the Rayleigh quotient $v^\top A v$  (with a step size $\beta_t$). Alternatively, it can be interpreted as a randomized power method on the matrix $(I+\beta_t A)$\cite{hardt2014}. 
\vspace*{-4pt}
\paragraph{Two-time-scale approximation.}
The theory of two-time-scale approximation forms the underlying basis for \alg. It considers coupled iterative systems where one component changes much faster than the other \cite{Bo1,Bo2}. More precisely, its objective is to understand classical systems of the type:
\begin{eqnarray}
\label{eq:tts_x}
x_{t}&=&x_{t-1}+\alpha_{t}\left[h \left(x_{t-1},y_{t-1} \right)+\xi^1_{t}\right] \\
\label{eq:tts_y}y_{t}&=&y_{t-1}+\beta_{t}\left[g \left(x_{t-1},y_{t-1} \right)+\xi^2_{t}\right],
\end{eqnarray}  
where $g$ and $h$ are the update functions and $(\xi^1_{t},\xi^2_{t})$ correspond to the noise vectors at step $t$ and typically assumed to be martingale difference sequences. 

In the above model, whenever the two step sizes $\alpha_t$ and $\beta_t$ satisfy $\beta_t/\alpha_t\to 0$, the sequence $y_t$ moves on a slower timescale than $x_t$. For any fixed value of $y$ the dynamical system given by $x_t$,
\begin{equation}
x_{t}=x_{t-1}+\alpha_{t}[h \left(x_{t-1},y \right)+\xi^1_{t}],
\end{equation}
converges to to a solution $x^*(y)$. In the coupled system, since the state variables $x_t$ move at a much faster time scale, they can be seen as being close to $x^*(y_t)$, and thus, we can alternatively consider:
 \begin{equation}
y_{t}=y_{t-1}+\beta_{t}\left[g\left(x_*(y_{t-1}),y_{t-1}\right)+\xi^2_{t}\right].
\end{equation}
If the process given by $y_t$ above were to converge to $y^*$, under certain conditions, we can argue that the coupled process $(x_t,y_t)$ converges to $(x^*(y^*),y^*)$. Intuitively, because $x_t$ and $y_t$ are evolving at different time-scales, $x_t$ views the process $y_t$ as quasi-constant while $y_t$ views $x_t$ as a process rapidly converging to $x^*(y_t)$.

\alg can be seen as a particular instance of the coupled iterative system given by Equations \eqref{eq:tts_x} and \eqref{eq:tts_y} where the sequence $v_t$ evolves with a step-size $\beta_t \approx \frac{1}{t}$, much slower than the sequence $w_t$, which has a step-size of $\alpha_t \approx \frac{1}{\log(t)}$. Proceeding as above, the sequence $v_t$ views $w_t$ as having converged to $B^{-1}Av_t + \xi_t$, where $\xi_t$ is a noise term, and the update step for $v_t$ in \alg can be viewed as a step of Oja's algorithm, albeit with Markovian noise.

While previous works on the stochastic CCA problem required to use logarithmic independent samples to solve the inner least-squares problem in order to perform an approximate power method (or Oja) step, the theory of two-time-scale stochastic approximation suggests that it is possible to obtain a similar effect by evolving the sequences $w_t$ and $v_t$ at two different time scales.
\vspace*{-4pt}
\paragraph{Understanding the Markov Process $\{w_t\}$.} In order to understand the process described by the sequence $w_t$, we consider the homogeneous Markov chain $(w^{v}_t)$ defined by 
\begin{equation}\label{eq:mctc}
w^{v}_t=w^{v}_{t-1}-\alpha(B_tw^v_{t-1} -A_t v),
\end{equation}
for a constant vector $v$ and we denote its $t$-step kernel by $\pi_{v}^t$ 
\cite{MeyTwe09}. 
This Markov process is an iterative linear model and has been extensively studied by \cite{Stei99,DiaFre99, BacMou13}. It is known that for any step-size $\alpha \leq 2/R^2$, the Markov chain $(w^{v}_t)_{t\geq 0}$ admits a unique stationary distribution, denoted by $\nu_v$. In addition,
\begin{equation}
\label{eq:mcmc_mixing}
W_2^2(\pi^t_v(w_0,\cdot),\nu_v)\leq (1-2\mu\alpha(1-\alpha R^2_B/2))^t \int_{\R^d}\Vert w_0-w\Vert_2^2 d\nu_v(w),
\end{equation}
where $W_2^2(\lambda, \nu)$ denotes the Wasserstein distance of order 2 between probability measures $\lambda$ and $\nu$ (see, e.g., \cite{Vil08} for more properties of $W_2$). Equation \eqref{eq:mcmc_mixing} implies that the iterative linear process described by \eqref{eq:mctc} mixes exponentially fast to the stationary distribution. This forms a crucial ingredient in our convergence analysis where we use the fast mixing to obtain a bound on the expected norm of the Markovian noise (see Lemma~\ref{lem:tatanos}). 

Moreover, one can compute the mean $\bar{w}^v$ of the process $w_t$ under the stationary distribution by taking expectation under $\nu_v$ on both sides in equation \eqref{eq:mctc}. Doing so, we obtain, $\bar w^v= B^{-1}Av$. Thus, in our setting, since the $v_t$ process evolves slowly, we can expect that $w_t \approx B^{-1}Av_t$, allowing \alg to mimic Oja's algorithm.

\tocless\section{Main Theorem}\label{sec:main_thm}
In this section, we present our main convergence guarantee for \alg when applied to the streaming generalized eigenvector problem. We begin by listing the key assumptions required by our analysis:
\vspace{-1ex}
\begin{description}
\item[(A1)] The matrices $(A_i)_{i\geq 0}$ satisfy $\E[A_i]=A$ for a symmetric matrix $A\in\R^{d\times d}$.
\item[(A2)] The matrices $(B_i)_{i\geq 0}$ are such that each $B_i \succcurlyeq 0$ is symmetric
and satisfies $\E[B_i]=B$ for a symmetric matrix $B\in\R^{d\times d}$ with  $B\succcurlyeq \mu I$ for $\mu > 0$.
\item[(A3)] There exists $R\geq0$ such that $\max\{ \Vert A_i\Vert, \Vert B_i\Vert \}\leq R$ almost surely.
\end{description}

Under the assumptions stated above, we obtain the following convergence theorem for \alg with respect to the $\sin^2_B$ distance, as described in Section \ref{sec:prob}.
\begin{theorem}[Main Result]\label{thm:main_thm_paper}
Fix any $\delta > 0$ and $\epsilon_1 > 0$. Suppose that the step sizes are set to $\alpha_t = \frac{c}{\log(d^2\beta+t)} $ and $\beta_t = \frac{\gamma}{\Delta_\lambda(d^2\beta+t)}$ for $\gamma > 1/2\ ,\ c>1$ and
\begin{small}
\begin{equation*}
  \beta = \max \left(\frac{20\gamma^2 \lambda_1^2}{\Delta_\lambda^2d^2\log\left(\frac{1+\delta/100}{1+\epsilon_1} \right)}, \frac{200\left(     \frac{R}{\mu}+\frac{R^3}{\mu^2}+ \frac{R^5}{\mu^3}\right)\log\left(    1+\frac{R^2}{\mu}+ \frac{R^4}{\mu^2}    \right)}{\delta \Delta_\lambda^2} \right).
\end{equation*}
\end{small}
Suppose that the number of samples $n$ satisfy
\begin{small}
\begin{equation*}
\frac{d^2\beta+n}{\log^{\frac{1}{\min(1, 2\gamma\lambda_1/\Delta_\lambda)}}(d^2\beta+n)} \geq \left( \frac{cd}{\delta_1\min(1, \lambda_1)}\right)^{\frac{1}{\min(1, 2\gamma\lambda_1/\Delta_\lambda)}}(d^3\beta+1)\exp\left( \frac{c\lambda_1^2}{d^2}\right)
\end{equation*}
\end{small}
Then, the output $v_n$ of Algorithm \ref{alg:spge} satisfies,
\begin{small}
  \begin{equation*}
\sin_B^2(u_1, v_n) \leq \frac{(2+\epsilon_1)cd \|\sum_{i=1}^d\tilde{u}_i\tilde{u}_i^\top \|_2\log\left(\frac{1}{\delta}\right)}{\delta^2\|\tilde{u}_1 \|_2^2}\left( \frac{c\gamma^2\log^3(d^2\beta+n)}{\Delta_\lambda^2(d^2\beta+n+1)} + \frac{cd}{\Delta_\lambda}\left(\frac{d^2\beta + \log^3(d^2\beta)}{d^2\beta + n +1} \right)^{2\gamma}\right),
  \end{equation*}
  \end{small}
  with probability at least $1-\delta$ with $c$ depending polynomially on parameters of the problem $\lambda_1, \kappa_B, R, \mu$. The parameter  $\delta_1$ is set as $\delta_1 = \frac{\epsilon_1}{2(2+\epsilon_1)}$. 
\end{theorem}

The above result shows that with probability at least $1-\delta$, \alg converges in the $B$-norm to the right eigenvector, $u_1$, corresponding to the maximum eigenvalue of the matrix $B^{-1}A$. Further, \alg exhibits an $\tilde{\mathcal{O}}({1}/{n})$ rate of convergence, which is known to be optimal for stochastic approximation algorithms even with convex objectives \citep{NemYud83}.

\textbf{Comparison with Streaming PCA.} In the setting where $B = I$, and $A \succeq 0$ is a covariance matrix, the principal generalized eigenvector problem reduces to performing PCA on the $A$. When compared with the results obtained for streaming PCA by \cite{JaiJinKakNetSid16}, our corresponding results differ by a factor of dimension $d$ and problem dependent parameters $\lambda_1, \Delta_\lambda$. We believe that such a dependence is not inherent to \alg but a consequence of our analysis. We leave this task of obtaining a dimension free bound for \alg as future work. 

\textbf{Gap-independent step size}: While the step size for the sequence $v_n$ in \alg depends on eigen-gap, which is a priori unknown, one can leverage recent results as in \cite{TriFlaBacJor18} to get around this issue by using a streaming average step size.

\section{Proof Sketch}
In this section, we detail out the two key ideas underlying the analysis of \alg to obtain the convergence rate mentioned in Theorem \ref{thm:main_thm_paper}: a) controlling the non i.i.d. Markovian noise term which is introduced because of the coupled Markov chains in \alg and b) proving that a noisy power method with such Markovian noise converges to the correct solution.

\paragraph{Controlling Markovian perturbations.}
In order to better understand the sequence $v_t$, we rewrite the update as,
\begin{equation} \label{eq:vt_redefined}
v'_t = v_{t-1} + \beta_tw_{t} = v_{t-1} + \beta_t(B^{-1}Av_{t-1} + \xi_t),
\end{equation} 
where $\xi_t = w_t - B^{-1}Av_{t-1}$ is the prediction error which is a Markovian noise. Note that the noise term is neither \emph{mean zero} nor a \emph{martingale difference} sequence. Instead, the noise term $\xi_t$ is dependent on all previous iterates, which makes the analysis of the process more involved. This framework with Markovian noise has been extensively studied by \cite{BenMetPri90,AndMouPri05}. 

From the update in Equation \eqref{eq:vt_redefined}, we observe that \alg is performing an Oja update but with a controlled Markovian noise. However, we would like to highlight that classical techniques in the study of stochastic approximation with Markovian noise (as the \emph{Poisson Equation} \cite{BenMetPri90,MeyTwe09}) were not enough to provide adequate control on the noise to show convergence.

In order to overcome this difficulty, we leverage the fast mixing of the chain $w_t^v$  for understanding the Markovian noise. While it holds that $\E[\|\xi_t\|_2] = \mathcal{O}(1)$ (see Appendix \ref{sec:moment}), a key part of our analysis is the following lemma, the proof of which can be found in Appendix \ref{app:tatanos}).  
\begin{lemma}\label{lem:tatanos}.
For any choice of ${k > 4\frac{\lambda_1(B)}{\mu \alpha } \log(\frac{1}{\beta_{t+k}})}$, and assuming that $\Vert w_s\Vert \leq W_s$ for $t\leq s\leq t+k$  we have that
$$\|\E [\xi_{t+k}|\mathcal{F}_{t}] \|_2 = \mathcal{O}(\beta_t k^2 \alpha_t W_{t+k})$$
\end{lemma}

Lemma~\ref{lem:tatanos} uses the fast mixing of $w_t$ to show that $\|\E[\xi_{t}]|\mathcal{F}_{t-r}\|_2 = \tilde{\mathcal{O}}(\beta_t)$ where $r = \mathcal{O}(\log t)$, i.e., the magnitude of the expected noise is small conditioned on $\log(t)$ steps in the past.


\paragraph{Analysis of Oja's algorithm.}
The usual proofs of convergence for stochastic approximation define a Lyapunov function and show that it decreases sufficiently at each iteration. Oftentimes control on the per step rate of decrease can then be translated into a global convergence result. Unfortunately in the context of PCA, due to the non-convexity of the Raleigh quotient, the quality of the estimate $v_t$ cannot be related to the previous $v_{t-1}$. Indeed $v_t$ may become orthogonal to the leading eigenvector. Instead \cite{JaiJinKakNetSid16} circumvent this issue by leveraging the randomness of the initialization and adopt an operator view of the problem. We take inspiration from this approach in our analysis of \alg. Let $G_i=w_i v_{i-1}^\top$ and 
$H_t=\prod_{i=1}^t (I+\beta_i G_i)
$,
\alg's update can be equivalently written as 
\[
v_t=\frac{H_t v_0}{\Vert H_t v_0\Vert_2^2},
\]
pushing, for the analysis only, the normalization step at the end. This point of view enables us to analyze the improvement of $H_t$ over $H_{t-1}$ since allows one to interpret Oja's update as one step of power method on $H_t$ starting on a random vector $v_0$. We present here an easy adaptation of \cite[][Lemma 3.1]{JaiJinKakNetSid16} that takes into account the special geometry of the generalized eigenvector problem and the asymmetry of $\ba$. The proof can be found in \myapp{proofsin}.
\begin{lemma}\label{lem:sin}
Let $H\in \R^{d\times d}$, $(u_i)_{i=1}^d$ and $(\tilde{u}_i)_{i=1}^d$ be the corresponding right and left eigenvectors of $\ba$  and $w\in\R^d$ chosen uniformly on the sphere, then with probability $1-\delta$ (over the randomness in the initial iterate)
\begin{equation}\label{eq:lemmasin}
\sin_B^2(u_i, Hw)\leq  \frac{C\log(1/\delta)}{\delta}\frac{\tr (HH^\top \sum_{j\neq i} \tilde u_j \tilde u_j^\top )}{ \tilde u_i^\top HH^\top \tilde u_i },
\end{equation}
for some universal constant $C>0$.
\end{lemma}
This lemma has the virtue of highly simplifying the challenging proof of convergence of Oja's algorithm. Indeed we only have to prove that $H_t$ will be close to $\prod_{i=1}^t (I+\beta_i \ba)$ for $t$ large enough which can be interpreted as an analogue of the law of large numbers for the multiplication of matrices. This will ensure that $\tr (H_tH_t^\top \sum_{j\neq i} \tilde u_j \tilde u_j^\top )$ is relatively small compared to $ \tilde u_i^\top H_tH_t^\top \tilde u_i$ and be enough with Lemma~\ref{lem:sin} to prove Theorem~\ref{thm:main_thm_paper}. The proof  follows the  line of \cite{JaiJinKakNetSid16} with two additional tedious difficulties: the Markovian noise is neither unbiased nor independent of the previous iterates, and the matrix $\ba$ is no longer symmetric, which is precisely why we consider the left eigenvector $\tilde u_i$ in the right-hand side of \eq{lemmasin}. We highlight two key steps:
\begin{itemize}
\vspace*{-5pt}
\item First we show that $\E \tr (H_tH_t^\top \sum_{j\neq i} \tilde u_j \tilde u_j^\top )$ grows as $\mathcal{O}(\exp(2\lambda_2\sum_{i=1}^t\beta_i))$, which implies by Markov's inequality the same bound on $\tr (H_tH_t^\top \sum_{j\neq i} \tilde u_j \tilde u_j^\top )$ with constant probability. See Lemmas~\ref{lemma2} for more details.
\vspace*{-4pt}
\item Second we show that  $\Var \tilde u_i^\top H_tH_t^\top \tilde u_i$ grows as $\mathcal{O}(\exp(4\lambda_1\sum_{i=1}^t\beta_i))$ and $\E \tilde u_i^\top HH^\top \tilde u_i$ grows as $\mathcal{O}(\exp(2\lambda_1\sum_{i=1}^t\beta_i))$  which implies by Chebyshev's inequality the same bound for  $\tilde u_i^\top HH^\top \tilde u_i$ with constant probability. See Lemmas~\ref{lem:lemma3} and~\ref{lemma::vari}  for more details.
\end{itemize}

\tocless\section{Application to Canonical Correlation Analysis}
Consider two random vectors $X\in\R^{d}$ and $Y\in\R^{d}$ with joint distribution $P_{XY}$. The objective of canonical correlation analysis in the population setting is to find the canonical correlation vectors $\phi,\psi\in\R^{d,d}$ which maximize the correlation 
\[
\max_{\phi,\psi}\frac{\E[(\phi^\top X)(\psi^\top Y)]}{\sqrt{\E[(\phi^\top X)^2]\E[(\psi^\top Y)^2]}}.
\]
This problem is equivalent to maximizing $\phi^\top \E[XY^\top] \psi$ under the constraint $\E[(\phi^\top X)^2]=\E[(\psi^\top Y)^2]=1$ and admits a closed form solution: if we define $T=\E[XX^\top]^{-1/2}\E[XY^\top]\E[YY^\top]^{-1/2}$, then the solution is $(\phi_*,\psi_*)=(\E[XX^\top]^{-1/2}a_1 \E[YY^\top]^{-1/2}b_1)$ where $a_1,b_1$ are the left and right principal singular vectors  of $T$. By the KKT conditions, there exist $\nu_1,\nu_2\in\R$ such that this solution satisfies the stationarity equation 
\[\E[XY^\top]\psi=\nu_1  \E[XX^\top] \phi \quad \text{ and } \quad \E[YX^\top]\phi=\nu_2  \E[YY^\top] \psi.
\]
Using the constraint conditions we conclude that $\nu_1=\nu_2$. This condition 
 can be written (for $\lambda=\nu_1$) in the matrix form of \eq{ccagep}. 
As a consequence, finding the largest generalized eigenvector for the matrices $(A,B)$ will recover the canonical correlation vector $(\phi,\psi)$. 
Solving the associated generalized streaming eigenvector problem, we obtain the following result for estimating the canonical correlation vector whose proof easily follows from Theorem~\ref{thm:main_thm_paper} (setting $\gamma=6$). 
\begin{theorem}\label{theorem:cca}
Assume that $\max \{ \Vert X\Vert , \Vert Y\Vert \}\leq R$ a.s., $ \min\{ \lambda_{\min}(\E[XX^\top]),\lambda_{\min}(\E[YY^\top])\}=\mu>0$ and $\sigma_1(T)-\sigma_2(T)=\Delta>0$.   Fix any $\delta > 0$, let $\epsilon_1 \geq 0$, and suppose the step sizes are set to $\alpha_t = \frac{1}{2R^2\log(d^2\beta + t)} $ and  $\beta_t = \frac{6}{\Delta (d^2\beta+t)}$ and
\begin{small}
  \begin{equation*}
    \beta = \max \left(\frac{720\sigma_1^2}{\Delta^2d^2\log\left(\frac{1+\delta/100}{1+\epsilon_1} \right)}, \frac{200\left(     \frac{R}{\mu}+\frac{R^3}{\mu^2}+ \frac{R^5}{\mu^3}\right)\frac{1}{\delta}\log(    1+\frac{R^2}{\mu}+ \frac{R^4}{\mu^2}    )}{\Delta^2} \right)
  \end{equation*}
  \end{small}
Suppose that the number of samples $n$ satisfy
\begin{small}
\begin{equation*}
\frac{d^2\beta+n}{\log^{\frac{1}{\min(1, 12\lambda_1/\Delta_\lambda)}}(d^2\beta+n)} \geq \left( \frac{cd}{\delta_1\min(1, \lambda_1)}\right)^{\frac{1}{\min(1, 12\lambda_1/\Delta_\lambda)}}(d^3\beta+1)\exp\left( \frac{c\lambda_1^2}{d^2}\right)
\end{equation*}
\end{small}
Then the output $(\phi_t,\psi_t)$ of Algorithm \ref{alg:spge} applied to $(A,B)$ defined above satisfies,
\begin{small}
  \begin{equation*}
    \sin_B^2((\phi_*,\psi_*), (\phi_t,\psi_t)) \leq \frac{(2+\epsilon_1)cd^2 \log\left(\frac{1}{\delta}\right)}{\delta^2\|\tilde{u}_1 \|_2^2} \frac{\log^3(d^2\beta+n)}{\Delta^2(d^2\beta+n+1)} ,
  \end{equation*}
  \end{small}
  with probability at least $1-\delta$ with $c$ depending on parameters of the problem and independent of $d$ and $\Delta$ where $\delta_1 = \frac{\epsilon_1}{2(2+\epsilon_1)}$.
\end{theorem}
We can make the following observations:
\begin{itemize}
\item The convergence guarantee are comparable with the sample complexity obtained by the  ERM ($t=\tilde{\mathcal{O}} (d/(\eps \Delta^2)$ for sub-Gaussian variables and  $t=\tilde{\mathcal{O}} (1/(\eps \Delta^2\mu ^2)$ for bounded variables)\cite{GaoGarSreWanWan17} and matches the lower bound $t= \mathcal{O} (d/(\eps \Delta^2))$ known for sparse CCA~\cite{GaoMaZhou15}.
\item The sample complexity in \cite{GaoGarSreWanWan17} is better in term of the dependence on $d$. They obtain the same rates as the ERM. The comparison with  \cite{AroMarMiaSre17} is meaningless since they are in the gap free setting and their computational complexity is $\mathcal{O}(d^2)$.
\end{itemize}

\tocless\section{Simulations}
 \begin{figure}[t]
\centering
\begin{minipage}[c]{.32\linewidth}
\includegraphics[width=\linewidth]{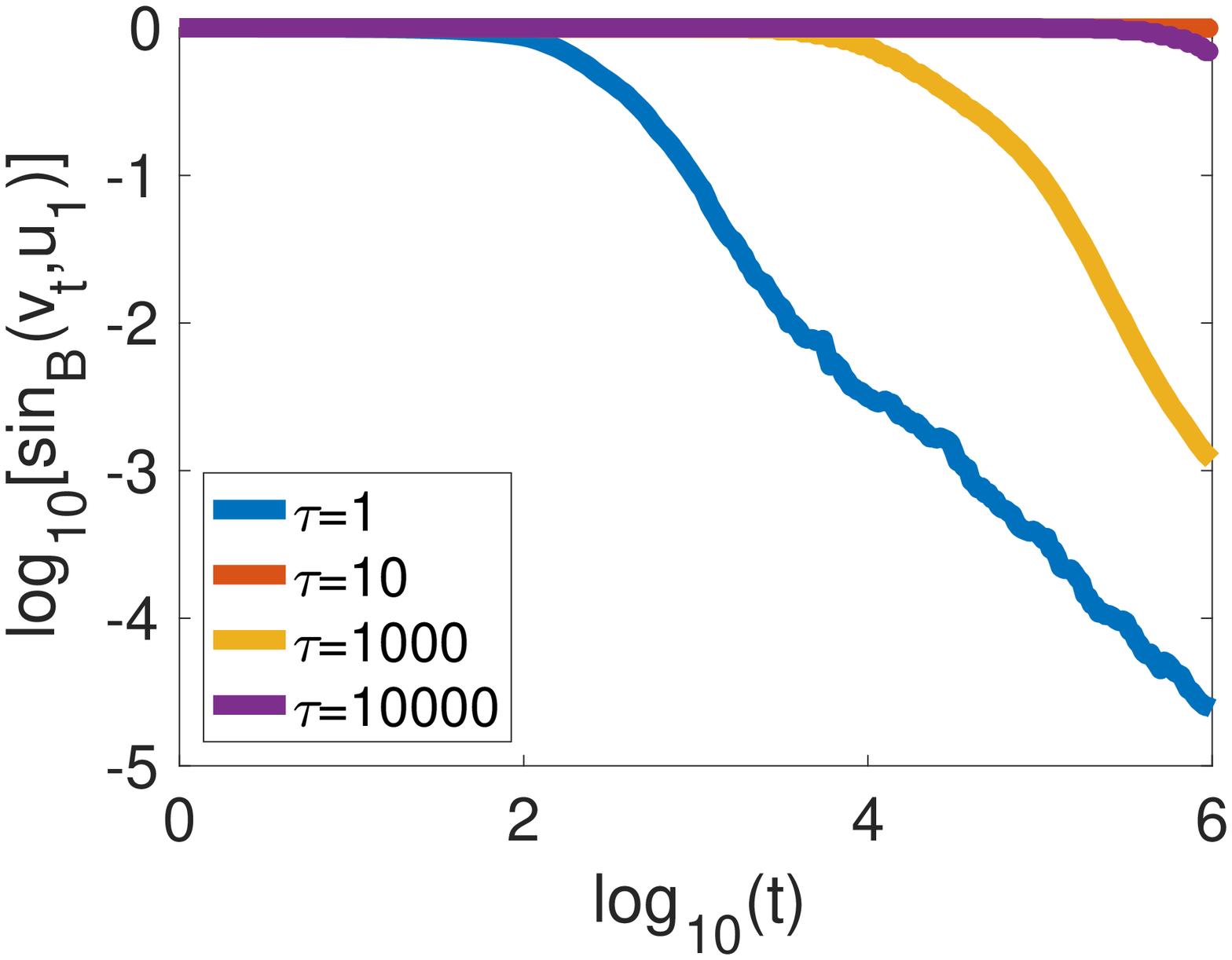}
   \end{minipage}
    \hspace*{-1pt}
   \begin{minipage}[c]{.32\linewidth}
\includegraphics[width=\linewidth]{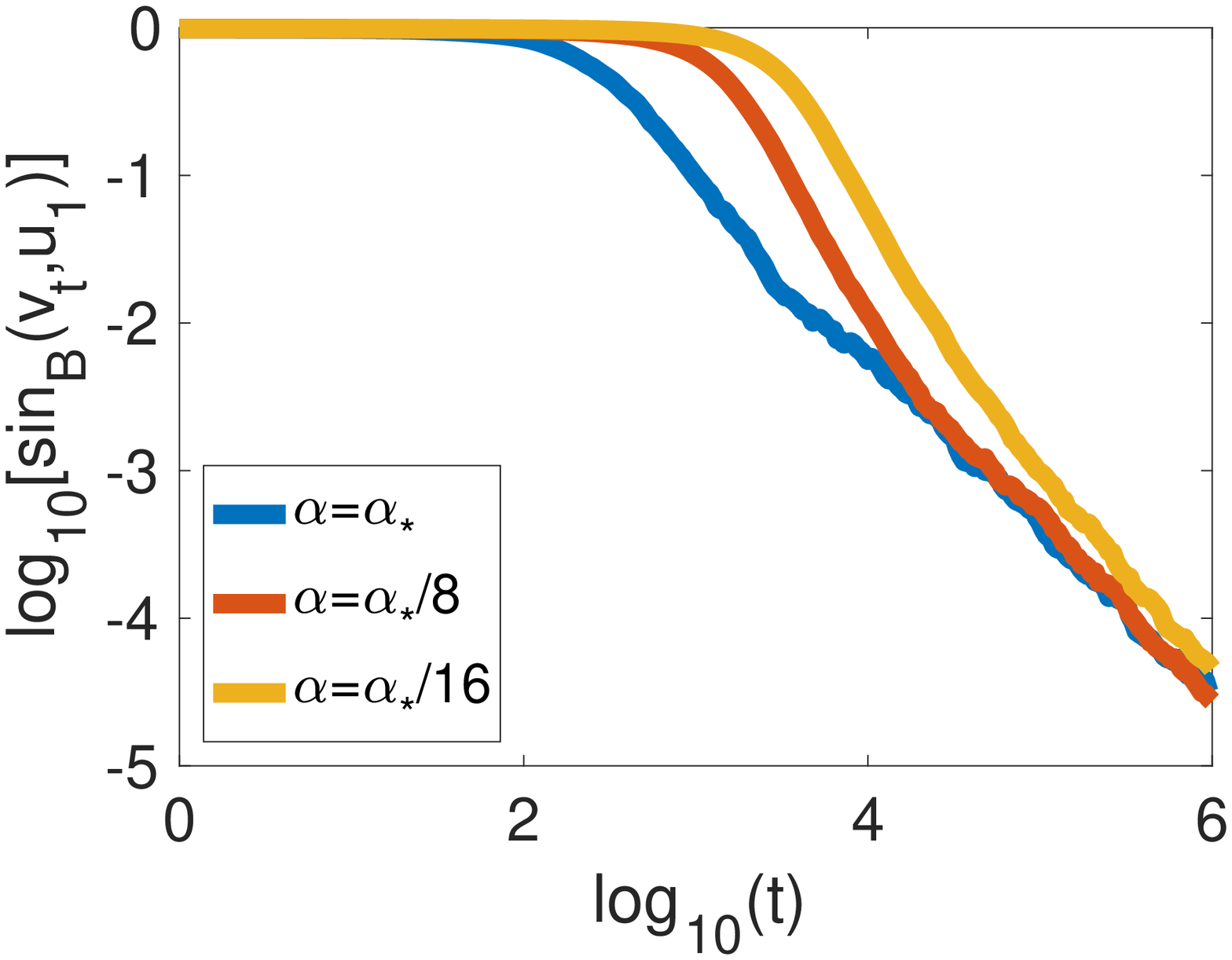}
   \end{minipage}
    \hspace*{-1pt}
   \begin{minipage}[c]{.32\linewidth}
\includegraphics[width=\linewidth]{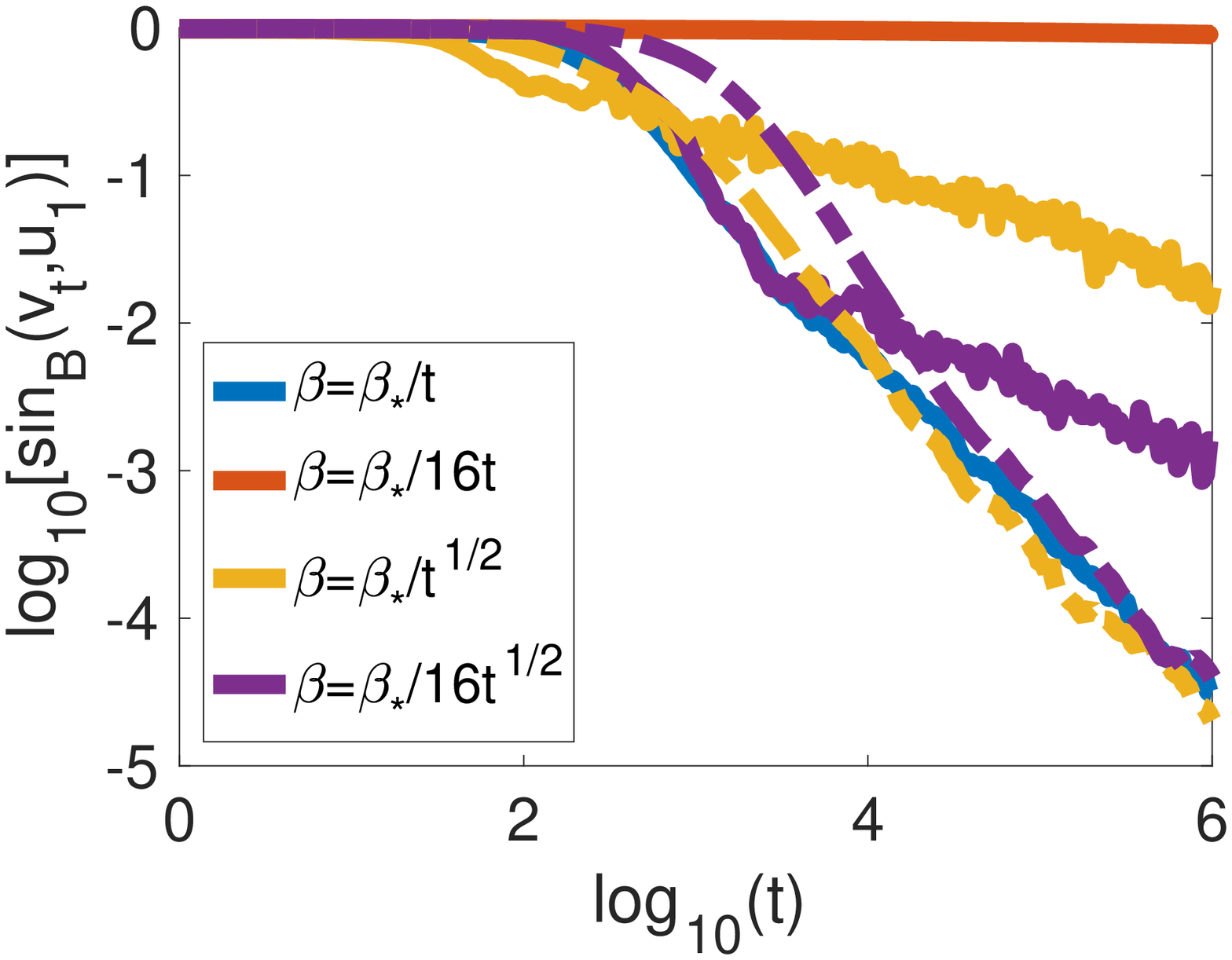}
   \end{minipage}
  \caption{Synthetic Generalized Eigenvalue problem. Left: Comparison with two-steps methods. Middle: Robustness to step size $\alpha_t$. Right: Robustness to step size $\beta_t$ (Streaming averaged \alg  is dashed).}
     \label{fig:synthetic}
\end{figure}
Here we illustrate the practical utility of \alg on a synthetic, streaming generalized eigenvector problem. We take $d=20$ and $T=10^6$. The streams $(A_t,B_t)\in (\R^{d\times d})^2$ are normally-distributed with covariance matrix $A$ and $B$ with random eigenvectors and eigenvalues decaying as $1/i$, for $i=1,\dots,d$. Here $R^2$ denotes the radius of the streams with $R^2=\max\{\tr A,\tr B\}$. All results are averaged over ten repetitions.
\vspace*{-5pt}
\paragraph{Comparison with two-steps methods.}
In the left plot of \myfig{synthetic} we compare the behavior of \alg to  different two-steps algorithms. Since the method by \cite{AroMarMiaSre17} is of complexity $\mathcal{O}(d^2)$, we compare \alg to a method which alternates between one step of Oja's algorithm and $\tau$ steps of averaged stochastic gradient descent with constant step size $1/2R^2$. \alg is converging at rate $\mathcal{O}(1/t)$ whereas the other methods are very slow. For $\tau=10$, the solution of the inner loop is too inaccurate and the steps of Oja are inefficient. For $\tau=10000$, the output of the sgd steps is very accurate but there are too few Oja iterations to make any progress. $\tau=1000$ seems an optimal parameter choice but this method is slower than \alg by an order of magnitude.
\vspace*{-5pt}
\paragraph{Robustness to incorrect step-size $\alpha$.} In the middle plot of \myfig{synthetic} we compare the behavior of \alg for step size $\alpha\in \{\alpha_*, \alpha_*/8, \alpha_*/16\}$ where $\alpha_*=1/R^2$. We observe that \alg converges at a rate $\mathcal{O}(1/t)$ independently of the choice of $\alpha$.
\vspace*{-5pt}
\paragraph{Robustness to incorrect step-size $\beta_t$.}
In the right plot of \myfig{synthetic} we compare the behavior of \alg for step size $\beta_t\in \{\beta_*/t, \beta_*/16t, \beta_*/\sqrt{i}, \beta_*/16\sqrt{i}\}$ where $\beta_*$ corresponds to the minimal error after one pass over the data. We observe that \alg is not robust to the choice of the constant for step size $\beta_t\propto 1/t$. If the constant is too small, the rate of convergence is arbitrary slow. We observe that considering the streaming average of \cite{TriFlaBacJor18} on \alg with a step size $\beta_t\propto 1/\sqrt{t}$ enables to recover the fast $\mathcal{O}(1/t)$ convergence while being robust to constant misspecification.

\tocless\section{Conclusion}
We have proposed and analyzed a simple online algorithm to solve the streaming generalized eigenvector problem and applied it to CCA. This algorithm, inspired by two-time-scale stochastic approximation achieves a fast $\mathcal{O}(1/t)$ convergence. Considering recovering the $k$-principal generalized eigenvector (for $k>1$) and obtaining a slow convergence rate $\mathcal{O}(1/\sqrt{t})$ in the gap free setting are promising future directions. Finally, it would be worth considering removing the dimension dependence in our convergence guarantee.

\tocless\section*{Acknowledgements}
We gratefully acknowledge the support of the NSF through grant IIS-1619362. AP acknowledges Huawei's support through a BAIR-Huawei PhD Fellowship. This work was supported in part by the Mathematical Data Science program of the Office of Naval Research under grant number N00014-18-1-2764. This work was partially supported by AFOSR through grant FA9550-17-1-0308. 
\begin{small}
\bibliography{refs}
\end{small}
\newpage

\appendix

\tableofcontents
\newpage
\section{Proof of Lemma~\ref{lem:sin}}
\label{sec:proofsin}
We prove here the Lemma~\ref{lem:sin} which is an easy adaptation of \cite[][Lemma 3.1]{JaiJinKakNetSid16}. We first recall it.
\begin{lemma}\label{lem:sinb}
Let $H\in \R^{d\times d}$, $(u_i)_{i=1}^d$ and $(\tilde{u}_i)_{i=1}^d$ the corresponding right and left eigenvectors of $\ba$  and $w\in\R^d$ chosen uniformly on the sphere, then with probability $1-\delta$ (over the randomness in the initial iterate)
\[
\sin^2_B(u_i, Hw)\leq  \frac{C\log(1/\delta)}{\delta}\frac{\tr (HH^\top \sum_{j\neq i} \tilde u_j \tilde u_j^\top )}{ \tilde u_i^\top HH^\top \tilde u_i },
\]
for some universal constant $C>0$.
\end{lemma}

\begin{proof}
We follow the proof of \cite{JaiJinKakNetSid16}. Given a $B$-normalized right eigenvector $u_i$ of $\ba$ and $w=\frac{g}{\Vert g\Vert_2}$ for $g\sim \mathcal{N}(0,I)$, we consider:
\[
\sin^2_B(u_i, Hw)= 1-\frac{(u_i^\top B Hw)^2 )}{w^\top H^\top  B H w } = \frac{g^\top H^\top B^{1/2}\left[ I-B^{1/2} u_i u_i^\top B^{1/2}\right]B^{1/2} H g}{g^\top H^\top B H g}.
\]
Moreover following Lemma \ref{lemma::right_eigenvectors_orthogonal} and denoting by $\hat u_i$ the corresponding orthonormal family of eigenvectors of the symmetric matrix $B^{-1/2}AB^{-1/2}$, we have that $u_i=B^{-1/2} \hat u_i$. This yields:
\[
\left[ I-B^{1/2} u_i u_i^\top B^{1/2}\right]=\left[ I-\hat u_i \hat u_i^\top \right] =\sum_{j\neq i } \hat u_j \hat u_j^\top
\]
Using now that the left eigenvectors of $\ba$ are given by  $\tilde u_i= B u_i$, we get
\[
\sin^2_B(u_i, Hw)= \frac{g^\top H^\top B^{1/2}\left[ \sum_{j\neq i } \hat u_j \hat u_j^\top \right]B^{1/2} H g}{g^\top H^\top B H g}=\frac{g^\top H^\top \left[ \sum_{j\neq i } \tilde u_j \tilde u_j^\top \right]  H g}{g^\top H^\top B H g}.
\]
We may bound the denominator by
\[
g^\top H^\top B H g \geq g^\top H^\top B^{1/2} \hat u_i \hat u_i^\top B^{1/2}  H g =  g^\top H^\top \tilde u_i \tilde u_i^\top  H g= (\tilde u_i^\top  H g)^2 \geq \frac{\delta}{C_1} \tilde u_i^\top  H H^\top \tilde u_i,
\]
where the last inequality follows as $\tilde u_i^\top  H g$ is a Gaussian random vector with variance $\Vert H^\top \tilde u_i\Vert_2^2$.
We can also bound the numerator as
\[
g^\top H^\top \left[ \sum_{j\neq i } \tilde u_j \tilde u_j^\top \right]  H g \leq C_2 \log(1/\delta) \tr [H^\top \sum_{j\neq i } \tilde u_j \tilde u_j^\top H],
\]
since $w^\top H^\top \left[ \sum_{j\neq i } \tilde u_j \tilde u_j^\top \right]  H w$ is a $\chi^2$ random variable with $\tr [H^\top \sum_{j\neq i } \tilde u_j \tilde u_j^\top H]$ degrees of freedom. Therefore it exists a universal constant $C>0$ such that
\[
\sin^2_B(u_i, Hw)\leq C\frac{\log(1/\delta)}{\delta} \frac{\tr [H^\top \sum_{j\neq i } \tilde u_j \tilde u_j^\top H]}{\tilde u_i^\top  H H^\top \tilde u_i},
\]
with probability $1-\delta$.
\end{proof}

\section{Deviation bounds for fast-mixing Markov Chain}
\label{app:tatanos}
In this section, we prove an upper bound on $\|\E [\epsilon_{t+k}|\mathcal{F}_{t}] \|_2$, where $\epsilon_t = (w_t - B^{-1}Aw_{t-1})v_{t-1}^\top$ and $\mathcal{F}_t = \sigma(w_0, \cdots, w_t)$ denotes the $\sigma$-algebra generated by $w_0, \cdots, w_t$. For the purpose of this section, we denote the pointwise upperbound on $\| w_t\|_2$ by $W_t$.
To begin with, we consider bounding the error term considering a fixed step-size $\alpha_t = \alpha$ in order to keep the analysis cleaner. In Lemma \ref{lem:final_thanos_alphat}, we bound the deviation of chains with step-size $\alpha_t = O(c/\log(d^2\beta + t))$ and fixed step size over a short horizon of length $O(\log^2(1/\beta_t))$

In order to prove the requisite bound, consider the following Markov chain given by,
\begin{equation}
\label{eq:mc}
\theta_{k+1} = \theta_k - \eta[f'(\theta_{(k)}) + \epsilon_{k+1}],
\end{equation}
where $f: \mathbb{R}^d \rightarrow \mathbb{R}$ is some strongly convex function. We make use of the following proposition highlighting the fast-mixing property of constant step-size stochastic gradient descent from \cite{dieuleveut2017}.
\begin{proposition}
  \label{prop:mc}
For any step size $\alpha \in (0, 2/L_{\theta})$, the markov chain given by $(\theta_k)_{k\geq 0}$  defined by recursion \eqref{eq:mc}, admits a unique stationary distribution $\pi \in \mathcal{P}(\mathbb{R}^d)$. In addition, for all $\theta \in \mathbb{R}^d, k \in \mathbb{N}$, we have,
\begin{equation}
  W_2^2(R^k(\theta, \cdot), \pi) \leq (1-2\mu_{\theta} \eta (1-\eta L_{\theta}/2))^k\int_{\mathbb{R}^d} \|\theta - \theta' \|_2^2 d\pi(\theta'),
\end{equation}
where $L_{\theta}$ and $\mu_{\theta}$ are the smoothness and the strong convexity parameters of $f$ respectively.
\end{proposition}
Now, consider the Markov chain given by
\begin{equation}
  \label{eq:wmc}
  w_t^{k+1} = w_t^{k} - \alpha(B_k w_t^k - A_k v_t),
\end{equation}
where $\E[B_k] = B, \E[A_k] = A, w_t^0 = w_t$ where $w_t$ is as given by Algorithm \ref{alg:spge}. Equation \eqref{eq:wmc} represents the update step for the $k^{th}$ step of a Markov chain starting at $w_t$ and performing stochastic gradient updates on $f_t(w) = 1/2 w^\top B w - w^\top A v_t$.

For this function $f_t$, the smoothness constant $L = \lambda_B$. Further, proposition \ref{prop:mc} guarantees the existence of a unique stationary distribution $\pi$ and we have that under the stationary distribution,
\begin{equation}
  \label{eq:sd}
  \E_{\pi} [w_t^k] = B^{-1}Av_t.
\end{equation}

\begin{lemma}
  \label{lem:mcmx}
  For the Markov chain given by \eqref{eq:wmc} with any step size $\alpha \in (0, 2/\lambda_B) $, for any ${k> \frac{\log(\frac{\lambda_1}{\epsilon})}{\mu \alpha \left(1-\frac{\alpha \lambda_B}{2}\right)}}$, we have
  \begin{equation*}
    \| \E[w_t^k - B^{-1}Av_t] | \mathcal{F}_t\|_2^2 \leq \epsilon
  \end{equation*}
\end{lemma}
\begin{proof}
We know from \eqref{eq:sd}, $ B^{-1}Av_t = \E_{\pi} [w_t^k]$. Now, we consider the term $\| \E[w_t^k - B^{-1}Av_t] | \mathcal{F}_t\|_2^2$,
\begin{align*}
  \| \E[w_t^k - B^{-1}Av_t] | \mathcal{F}_t\|_2^2 &= \| \E[w_t^k] - \E_{\pi}[w] | \mathcal{F}_t\|_2^2\\
  &= \|\E_{\Gamma(R^k(w_t, \cdot), \pi)}[w_t^k - w]|  \|_2^2\\
  &\stackrel{\zeta_1}{\leq} \E_{\Gamma(R^k(w_t, \cdot), \pi)}[\|w_t^k - w \|_2^2]\\
  &\stackrel{\zeta_2}{=} W_2^2(R^k(w_t, \cdot), \pi)\\
  &\stackrel{\zeta_3}{\leq} (1-2\mu \alpha (1-\alpha \lambda_B/2))^k\lambda_1^2,
\end{align*}
where $R^k(w_t, \cdot)$ denotes the $k$-step transition kernel of the Markov chain beginning from $w_t$, $\Gamma(R^k(w_t, \cdot), \pi)$ denotes any coupling of the distributions $R^k(w_t, \cdot)$ and $\pi$ and $\E_{\Gamma(\cdot , \cdot)}$ denotes the expectation under the joint distribution, conditioned on $\mathcal{F}_t$. Now, $\zeta_1$ follows from Jenson's inequality, $\zeta_2$ follows by setting $\Gamma(R^k(w_t, \cdot), \pi)$ to the coupling attaining the infimum in the wasserstein bound and $\zeta_3$ follows by using proposition \eqref{prop:mc}. The lemma now follows by setting ${k> \frac{\log(\frac{\lambda_1}{\epsilon})}{\mu \alpha \left(1-\frac{\alpha \lambda_B}{2}\right)}}$. \cite[see, e.g.,][for more properties of $W_2$]{Vil08}
\end{proof}

\textbf{Deviation bound for $\| v_t - v_{t+k}\|_2$}: We now bound the deviation of $v_{t+k}$ from $v_t$ if we execute $k$ steps of the algorithm sarting from $v_t$,
\begin{equation}
  \label{eq:vdiff}
  \| v_t - v_{t+k}\|_2 \leq \sum_{i=0}^{k-1} \|v_{t+i} - v_{t+i+1} \|_2.
\end{equation}
Now, for a single step of the algorithm, using the contractivity of the projection
\begin{equation*}
\| v_i - v_{i+1}\|_2 \leq \| v_i - \frac{v_{i+1}'}{\Vert v_{i+1}'\Vert} \|_2 \leq \| v_i - v_{i+1}'\|_2 \leq W_{i+1}\beta_{i+1}.
\end{equation*}
Using the above bound in \eqref{eq:vdiff}, we obtain,
\begin{equation}
  \label{eq:vbnd}
  \| v_t - v_{t+k}\|_2 \leq W_{t+k}\sum_{i=0}^{k-1} \beta_{t+i+1} \leq W_{t+k} k\beta_{t},
\end{equation}
by using the fact that  $\beta_t$ is a decreasing sequence.

\textbf{Deviation bound for Coupled Chains}: Consider the sequence $(w_{t+i})_{i = 0}^{k}$ as generated by Algorithm \ref{alg:spge}, assuming a constant step-size $\alpha$, and the sequence $(w_t^i)_{i = 1}^k$ generated by the recurrence \eqref{eq:wmc} in the case when both have the same randomness with respect to the sampling of the matrices $A_{t+i}, B_{t+i}$. We now obtain a bound on ${\|\E[w_t^k-w_{t+k}]|\mathcal{F}_t \|_2}$.
\begin{align}
  \begin{split}
\|\E[w_t^k-w_{t+k}]|\mathcal{F}_t \|_2 &= \| \E\left[ \E[(I-\alpha B_{t+k})(w_t^{k-1} - w_{t+k-1}) -\alpha A_{t+k}(v_t - v_{t+k-1})\right]|\mathcal{F}_{t+k-1} ]|\mathcal{F}_t\|_2\\
&= \| \E\left[ (I-\alpha B)(w_t^{k-1} - w_{t+k-1}) -\alpha A(v_t - v_{t+k-1})\right]|\mathcal{F}_t\|_2\\
&\vdots\\
&= \alpha \left\Vert \E\left[ \sum_{i=0}^{k-1} (I-\alpha B)^i A(v_t - v_{t+k-1-i})|\mathcal{F}_t\right] \right\Vert_2\\
&\leq \alpha  \E\left[ \sum_{i=0}^{k-1} \left\Vert(I-\alpha B)^i A(v_t - v_{t+k-1-i})\right\Vert_2|\mathcal{F}_t\right] \\
&\leq \alpha \lambda_{A} W_{t+k} k \sum_{i=0}^{k-1}\left(1-\alpha\mu\right)^i\beta_{t+k-1-i}\\
&\leq \frac{ \lambda_{A}  W_{t+k} k \beta_{t}}{\mu}, \label{eq:bndw}
\end{split}
\end{align}
where we expand the terms using the recursion and bound the geometric series by using that $\alpha \mu \leq 1$.
\begin{samepage}
\begin{lemma}\label{lem:const_step_thanos}
For any choice of ${k> \frac{\log(\frac{1}{\beta_t})}{2 \mu \alpha \left(1-\frac{\alpha \lambda_B}{2}\right)}}$, we have that
$$ \|\E [\epsilon_{t+k}|\mathcal{F}_{t}] \|_2 \leq \left(\frac{\lambda_{A} W_{t+k} k}{\mu} + \lambda_1 (1+2W_{t+k}k) + W_{t+k}^2k \right)\beta_{t} = O(W_{t+k}^2 k\beta_{t})$$
\end{lemma}
\end{samepage}
\begin{proof}
Consider the term $\|\E [\epsilon_{t+k}|\mathcal{F}_{t}] \|_2$,
\begin{align*}
  \|\E [\epsilon_{t+k}|\mathcal{F}_{t}] \|_2 &= \| \E[(w_{t+k} -   B^{-1}A v_{t+k-1})v_{t+k-1}^\top|\mathcal{F}_{t}]\|_2 \\
  &\leq \underbrace{\| \E[(w_{t+k} - B^{-1}A v_{t+k-1})v_{t}^\top|\mathcal{F}_{t}]\|_2}_{(I)} + \| \underbrace{\E[(w_{t+k} - B^{-1}A v_{t+k-1})(v_{t+k-1} - v_{t})^\top|\mathcal{F}_{t}]\|_2}_{(II)}.
\end{align*}
We first analyze term (I) in the expansion above.
\begin{align}
\| \E[(w_{t+k} - B^{-1}A v_{t+k-1})v_{t}^\top|\mathcal{F}_{t}]\|_2 &=  \| \E[(w_{t+k} - w_t^k) + (w_t^k - B^{-1}A v_{t})\nonumber\\
&\qquad + (B^{-1}A v_{t} - B^{-1}A v_{t+k-1}))v_{t}^\top|\mathcal{F}_{t}]\|_2 \nonumber \\
&\leq\| \E[(w_{t+k} - w_t^k) |\mathcal{F}_{t}]v_{t}^\top\|_2+ \| \E[(w_t^k - B^{-1}A v_{t+k-1}))|\mathcal{F}_{t}] v_{t}^\top\|_2 \nonumber\\
&\leq\| \E[(w_{t+k} - w_t^k) |\mathcal{F}_{t}]\|_2+ \| \E[(w_t^k - B^{-1}A v_{t}))|\mathcal{F}_{t}]\|_2 \nonumber\\
&\quad+ | \E[(B^{-1}A v_{t} - B^{-1}A v_{t+k-1}))|\mathcal{F}_{t}]\|_2\nonumber\\
&\stackrel{\zeta_1}{\leq} \frac{\lambda_{A} W_{t+k} k}{\mu}\beta_{t} + \lambda_1\beta_{t} + \lambda_1 W_{t+k} k \beta_{t}\nonumber \\
&= (\frac{\lambda_{A} W_{t+k} k}{\mu}  + \lambda_1 (1+W_{t+k}k))\beta_{t},\label{eq:t1}
\end{align}
where $\zeta_1$ follows from using lemma \ref{lem:mcmx} with ${k> \frac{\log(\frac{1}{\beta_t})}{2 \mu \alpha \left(1-\frac{\alpha \lambda_B}{2}\right)}}$, bound in \eqref{eq:vbnd} and bound in \eqref{eq:bndw}.

We now look at term (II) in the expansion.
\begin{align}
  \|\E[(w_{t+k} - B^{-1}A v_{t+k-1})(v_{t+k-1} - v_{t})^\top|\mathcal{F}_{t}]\|_2 &\leq (W_{t+k}+\lambda_1)\|v_{t+k-1} - v_{t} \| \nonumber\\
  &\leq W_{t+k}(W_{t+k}+\lambda_1)k\beta_{t}\label{eq:t2}.
\end{align}
 Combininig the bounds in \eqref{eq:t1} and \eqref{eq:t2}, we get the desired result.
\end{proof}

The bound we proved above hold for any fixed fixed step-size $\alpha$. However, in order to obtain the sharpest convergence result for our algorithm, we would require the step size $\alpha_t = \frac{c}{\log (d^2\beta + t)}$ for some constant $\beta$. We provide the following lemma which accomodates for this change.


In order to get a bound on the noise term with a logarithmically decaying step size, in addition to the previous analysis, we consider processes $(\hat{w}_{t+i})_{i =1}^k$ and $(\hat{v}_{t+i})_{i =1}^k$ which evolve with the same random matrices $A_{t+i}$ and $B_{t+i}$, but with a step size of $\alpha_{t+i} = \alpha_t = \frac{c}{\log (d^2\beta + t)}$.

\textbf{Pointwise bound on $\|\hat{w}_{t+k}\|_2$}: We can obtain a pointwise bound on $\|\hat{w}_{t+k}\|_2$ using the simple recursive evaluation:
\begin{align}
  \|\hat{w}_{t+k} \| &\leq \|I-\alpha_t B_{t+k}\|_2 \|\hat{w}_{t+k-1}\|_2 + \alpha_t \lambda_A \nonumber \\
  &\leq W_t + k\alpha_t \lambda_A,\label{eq:hatw_pointwise_bnd}
\end{align}
where the final inequality follows from recursing on $\|\hat{w}_{t+k-1} \|$ and using the assumption that $B_i \succeq 0$.

\textbf{Deviation bound for $\|v_{t+k} - \hat{v}_{t+k}\|_2$}: We can obtain a bound on this quantity as follows:
\begin{align}
  \|v_{t+k} - \hat{v}_{t+k} \|_2 &\leq \|v_{t+k} - v'_{t+k}\|_2 + \| \hat{v}_{t+k}' - \hat{v}_{t+k} \|_2 + \|v'_{t+k} - \hat{v}_{t+k}'\|_2 \nonumber\\
  &\leq 2\beta_{t+k}W_{t+k} + 2\beta_{t+k} \|\hat{w}_{t+k}\|_2 + \| \beta_{t+k}(w_{t+k} - \hat{w}_{t+k})\|_2 + \|v_{t+k-1} - \hat{v}_{t+k-1}\|_2 \nonumber\\
  &\leq 2\left( \sum_{i=1}^k \beta_{t+i}(W_{t+i} + \| \hat{w}_{t+i}\|_2) \right) +\sum_{i=1}^k \beta_{t+i} \| w_{t+i} - \hat{w}_{t+i}\|_2 \nonumber \\
  &\leq 3 \beta_t k  (2 W_{t+k} + k\alpha_t \lambda_A),
\end{align}
where the final bound is obtained using $\|w_{t+k}\|_2 \leq W_{t+k}$ and $\|\hat{w}_{t+k} \|_2 \leq W_{t+k} +  k\alpha_t \lambda_A$ from Equation \eqref{eq:hatw_pointwise_bnd}

\begin{samepage}
\begin{lemma}\label{lem:final_thanos_alphat}
  For any choice of ${k> \frac{\log(\frac{1}{\beta_t})}{2 \mu \alpha_t \left(1-\frac{\alpha_t \lambda_B}{2}\right)}}$ and $\alpha_t \in (0, 2/\lambda_B)$ of the form $\alpha_t = \frac{c}{\log (d^2\beta + t)}$, we have that
\begin{align*}
  \|\E [\epsilon_{t+k}|\mathcal{F}_{t}] \|_2 &\leq \left(\frac{\lambda_{A} W_{t+k} k}{\mu} + \lambda_1 (1+2W_{t+k}k) + W_{t+k}^2k \right)\beta_{t} \\
  &\quad +  \frac{\lambda_B W_{t+k} k \alpha_t \beta_t}{c\mu \gamma } +  \frac{\lambda_A k \alpha_t \beta_t}{c \mu \gamma} + \frac{3\lambda_A \beta_t k  (2 W_{t+k} + k\alpha_t \lambda_A)}{\mu}\\
  &\quad + (2W_{t+k} + k\alpha_t \lambda_A)W_{t+k}k\beta_t.
\end{align*}
In other words, we get that $\|\E [\epsilon_{t+k}|\mathcal{F}_{t}] \|_2 = O(\beta_t k^2 \alpha_t W_{t+k})$.
\end{lemma}
\end{samepage}
\begin{proof}
In continuation from Lemma \ref{lem:const_step_thanos}, we consider bounding the deviation of the process $\hat{w}_{t+k}$ from the process $w_{t+k}$. The extra components in the error term $\epsilon_t$ remain the same and we ignore them for clarity of this lemma.
\begin{small}
\begin{equation}\label{eq:log_step_decomp}
\|\E[(w_{t+k} - \hat{w}_{t+k})v_{t+k-1}^\top|\mathcal{F}_t]\|_2 \leq \underbrace{\|\E[(w_{t+k} - \hat{w}_{t+k})v_{t}^\top |\mathcal{F}_t]\|_2}_{\text{(I)}} + \underbrace{\|\E[(w_{t+k} - \hat{w}_{t+k})(v_{t+k-1}-v_{t})^\top |\mathcal{F}_t]\|_2}_{\text{(II)}}
\end{equation}
\end{small}
We proceed by first analyzing term (I) in Equation \eqref{eq:log_step_decomp}.
\begin{align}
  \|\E[(w_{t+k} - \hat{w}_{t+k})v_{t}^\top |\mathcal{F}_t]\|_2 &= \|\E[ \E [((I - \alpha_{t+k}B_{t+k})w_{t+k-1} + \alpha_{t+k}A_{t+k}v_{t+k-1}) \nonumber \\
  &\quad - ((I-\alpha_tB_{t+k})\hat{w}_{t+k-1} + \alpha_t A_{t+k}\hat{v}_{t+k-1})|\mathcal{F}_{t+k-1}]v_t^\top|\mathcal{F}_t]\|_2\nonumber \\
  &= \|\E[((I - \alpha_{t+k}B)w_{t+k-1} + \alpha_{t+k}Av_{t+k-1})\nonumber\\
  &\qquad - ((I-\alpha_tB)\hat{w}_{t+k-1} + \alpha_t A \hat{v}_{t+k-1})|\mathcal{F}_t]v_t^\top\|_2\nonumber\\
  &= \| \E[(\alpha_t - \alpha_{t+k})Bw_{t+k-1} + (I-\alpha_t B )(w_{t+k-1} - \hat{w}_{t+k-1}) \nonumber \\
  &\quad + (\alpha_{t+k} - \alpha_t )Av_{t+k-1} + \alpha_tA(v_{t+k-1} - \hat{v}_{t+k-1}))|\mathcal{F}_t]v_t^\top\|_2 \nonumber \\
  &\leq \left\Vert \E\left[ \sum_{i=1}^k (\alpha_t - \alpha_{t+i})(I - \alpha_t B)^{k-i}Bw_{t+i-1} \vert \mathcal{F}_t\right] v_t^\top\right\Vert_2\nonumber \\
  &\quad + \left\Vert \E\left[ \sum_{i=1}^k (\alpha_{t+i} - \alpha_{t})(I - \alpha_t B)^{k-i}Av_{t+i-1} \vert \mathcal{F}_t\right] v_t^\top\right\Vert_2\nonumber \nonumber \\
  &\quad + \alpha_{t}\left\Vert \E\left[ \sum_{i=1}^k (I - \alpha_t B)^{k-i}A(v_{t+i-1} - \hat{v}_{t+i-1}) \vert \mathcal{F}_t\right] v_t^\top\right\Vert_2 \nonumber\\
  &\leq \frac{(\alpha_t - \alpha_{t+k})\lambda_BW_{t+k}}{\alpha_t \mu} + \frac{(\alpha_t - \alpha_{t+k})\lambda_A}{\alpha_t \mu} + \frac{\lambda_A \|v_{t+k-1} - \hat{v}_{t+k-1} \|_2}{\mu}\nonumber\\
  &\leq \frac{\lambda_B W_{t+k} k \alpha_t}{c\mu(d\beta + t)} +  \frac{\lambda_A k \alpha_t}{c \mu (d\beta + t)} + \frac{3\lambda_A \beta_t k  (2 W_{t+k} + k\alpha_t \lambda_A)}{\mu}\nonumber \\
  &\leq \frac{\lambda_B W_{t+k} k \alpha_t \beta_t}{c\mu b } +  \frac{\lambda_A k \alpha_t \beta_t}{c \mu b} + \frac{3\lambda_A \beta_t k  (2 W_{t+k} + k\alpha_t \lambda_A)}{\mu}\label{eq:decomp_I}
\end{align}
where the second last inequality follows using Jensen's ineuality along with a trinagle inequality and using the fact that $B \succeq \mu I$ and the last equality follows from using the form of $\beta_t = \frac{b}{d^2\beta+t}$ for some constant $b$.

We now consider term (II) in Equation \eqref{eq:log_step_decomp}.
\begin{align}
  \|\E[(w_{t+k} - \hat{w}_{t+k})(v_{t+k-1}-v_{t})^\top |\mathcal{F}_t]\|_2 \leq (2W_{t+k} + k\alpha_t \lambda_A)W_{t+k}k\beta_t \label{eq:decomp_II},
\end{align}
by using Jensen's inequality along with bound \eqref{eq:vbnd}. Combining \eqref{eq:decomp_I} and \eqref{eq:decomp_II} with \eqref{eq:log_step_decomp}, and using Lemma \ref{lem:const_step_thanos}, we obtain the desired result.
\end{proof}

Note that in order to prove the final convergence for Algorithm \ref{alg:spge}, we use the form of the step sizes $\alpha_t$ and $\beta_t$ as mentioned in this section.



In the following sections we denote by ${r_t = \frac{1}{2 \mu \alpha_t \left(1-\frac{\alpha_t \lambda_B}{2}\right)}}\log^2(\frac{1}{\beta_t})$ and $\mathcal{A}_t$ to be such that:

\begin{equation}
\mathcal{A}_t r_t \beta_t \geq \|\mathbb{E}\left[  \epsilon_{t+k} | \mathcal{F}_t\right] \|
\end{equation}

When $\alpha_t = \frac{c}{\log(d^2\beta + t)}$, $r_t $ will be $O(\log^3(1/\beta_t))$ and when $\alpha_t$ is contant, $r_t$ will be $O(\log^2(1/\beta_t))$.
%

\section{Controlling Markov Chain $w_t$}\label{sec:moment}
For the purpose of this section, we stick with bounds $R_A, R_B$ the maximum of which equals $R$ in the main paper. In this section we provide a bound on the norm of the markov chain $w_t$. We start by showing the $p$ moments of the norms of $w_t$ are bounded as long as $\alpha_t = \alpha$ a small enough constant $\forall t$. Ultimately we will use a time dependent $\alpha_t$ as defined in the previous section, but for warm up we start by showing some lemmas that bring out the behavior of $w_t$ when $\alpha_t = \alpha$ for all $t$. The proofs for a moving $\alpha_t$ will follow a similar though technically involved arguments. 

\begin{lemma}\label{lemma::wt_moment_bound}
For $\alpha\leq 1/R_B^2$ we have
\[
\E[\Vert w_{t} \Vert_{2}^2] \leq \big[ (1-\mu \alpha/2 )^t \Vert w_{0} \Vert_{2}+2\frac{R_A^2}{\mu}\big]^2.
\]
If, in addition we assume that $\alpha \leq \frac{2}{R_B^2(p-2)}$ for $p\geq 3$ we have:
\[
\E \big[\Vert w_{t} \Vert_2^p \big] \leq \Big[(1-\mu \alpha/4 )^t \Vert w_{0} \Vert_2+4\frac{R_A^2}{\mu}\Big]^p.
\]
\end{lemma}
\begin{proof}
We first expand $w_{t+1}=(I-\alpha B_{t+1})w_t +\alpha A_{t+1}  v_t $ and use the Minkowski inequality on $L_2$-norm (denoted by $\Vert\Vert_{L_2}$) to obtain:
\[
\Vert w_{t+1}\Vert_{L_2}\leq \Vert (I-\alpha B_t)w_t \Vert_{L_2}+\Vert \alpha A_{t+1}  v_t\Vert_{L_2}
\]
We directly have that $\Vert \alpha A_{t+1}  v_t\Vert_{L_2}\leq \alpha R_A^2$ almost surely and we can directly compute for $\alpha<1/R^2_B$:
\BEAS
 \Vert (I-\alpha B_{t+1})w_t \Vert_{L_2}^2&=& \E [w_t^\top  (I-\alpha B_{t+1})^2w_t] = \E [w_t^\top  (I-2\alpha B_{t+1} +\alpha^2 B_{t+1}^2)w_t] \\
 &\overset{(1)}{\leq}& \E [w_t^\top  (I-\alpha B_{t+1} )w_t]\leq  \E [w_t^\top  (I-\alpha \E[B_{t+1}|\mathcal{F_t}] )w_t] \leq (1-\alpha\mu) \E [\Vert w_t\Vert_2^2],
 \EEAS
where $(1)$ follows as $B_{t+1}\preccurlyeq R_B^2I$.
We obtain expanding the recursion ( and using $\sqrt{1-x}\leq 1-x/2$ for $x\geq0$):
\[
\Vert w_{t} \Vert_{L_2} \leq (1-\alpha \mu/2 )^t \Vert w_{0} \Vert_{L_2}+\alpha R_A^2\sum_{i=0}^{t-1}(1-\alpha \mu/2)^i.
\]
We conclude
\[
\Vert w_{t} \Vert_{L_2} \leq (1-\mu \alpha/2 )^t \Vert w_{0} \Vert_{L_2}+2\frac{R_A^2}{\mu}.
\]

We consider now $p\geq 3$. We expand again $w_{t+1}=(I-\alpha B_{t+1})w_t +\alpha A_{t+1}  v_t $ and use now the Minkowski inequality on $L_p$-norm on $(\R^d,\Vert \Vert_2)$ (denoted by $\Vert\Vert_{L_p}$ and defined by $\Vert x\Vert_{L_p}=( \E[ \Vert x\Vert_2^p])^{1/p}$) to obtain:
\[
\Vert w_{t+1}\Vert_{L_p}\leq \Vert (I-\alpha B_t)w_t \Vert_{L_p}+\Vert \alpha A_{t+1}  v_t\Vert_{L_p}
\]
We then compute for $\alpha<1/R^2_B$
\BEAS
 \Vert (I-\alpha B_{t+1})w_t \Vert_{L_p}^p&=& \E [(w_t^\top  (I-\alpha B_{t+1})^2w_t)^{p/2}] = \E [(w_t^\top  (I-2\alpha B_{t+1} +\alpha^2 B_{t+1}^2)w_t)^{p/2}]  \\
 &\leq& \E [(w_t^\top  (I-\alpha B_{t+1} )w_t)^{p/2}] \leq  \E [\Vert w_t\Vert_2^p\left(1-\alpha\frac{ w_t^\top B_{t+1}w_t}{\Vert w_t\Vert_2^2}\right)^{p/2}] \\
 &\overset{(1)}{\leq}& \E [\Vert w_t\Vert_2^p\left(1-p\alpha\frac{ w_t^\top B_{t+1}w_t}{2\Vert w_t\Vert_2^2} + \alpha^2\frac{p(p-2)}{8}\frac{ (w_t^\top B_{t+1}w_t)^2}{\Vert w_t\Vert_2^4}\right)] \\
 &\overset{(2)}{\leq}& \E [\Vert w_t\Vert_2^p\left(1-p\alpha\frac{ w_t^\top B_{t+1}w_t}{2\Vert w_t\Vert_2^2} + \alpha^2R^2_B\frac{p(p-2)}{8}\frac{ w_t^\top B_{t+1}w_t}{\Vert w_t\Vert_2^2}\right)]\\
  &\leq& \E [\Vert w_t\Vert_2^p\left(1-\frac{p\alpha}{2}(1-\alpha R_B^2\frac{p-2}{4})\frac{ w_t^\top B_{t+1}w_t}{\Vert w_t\Vert_2^2} \right)]\\
    &\overset{(3)}{\leq}& \E [\Vert w_t\Vert_2^p\left(1-\frac{p\alpha}{2}(1-\alpha R_B^2\frac{p-2}{4})\mu \right)],
\EEAS
where $(1)$ follows as $(1-x)^p\leq (1-px+p(p-1)/2x^2)$ for $x\in[0,1]$, $(2)$ follows as $ w_t^\top B_{t+1}w_t\leq R_B^2\Vert w_t\Vert_2^2$ and $(3)$ follows as $\E[B_{t+1}|\mathcal{F}_t]=B\succcurlyeq \mu I$.
Then using $(1-x)^{1/p}\leq 1-x/p$ for $x\geq0$) yields
\BEAS
\Vert (I-\alpha B_{t+1})w_t \Vert_{L_p}&\leq&\Vert w_{t}\Vert_{L_p} \left(1-\frac{\alpha}{2}(1-\alpha R_B^2 \frac{p-2}{4})\mu \right).
\EEAS
Moreover
\[
\Vert \alpha A_{t+1}  v_t\Vert_{L_p} \leq  \alpha R_A^2\quad \text{ a.s.}
\]
And therefore
\begin{equation}\label{eq::wt_moment_important_equation}
\Vert w_{t+1}\Vert_{L_p}\leq\Vert w_{t}\Vert_{L_p} \left(1-\frac{\alpha}{2}(1-\alpha R_B^2 \frac{p-2}{4})\mu \right) +\alpha R_A^2.
\end{equation}
Let us denote by $\delta= \frac{\alpha}{2}(1-\alpha R_B^2\frac{p-2}{4})\mu$, then we directly obtain expanding the recursion:
\[
\Vert w_{t} \Vert_{L_p} \leq (1-\delta)^t \Vert w_{0} \Vert_{L_p}+\alpha R_A^2\sum_{i=0}^{t-1}(1-\delta)^i.
\]
We conclude for $\alpha \leq \frac{2}{R_B^2(p-2)}$
\[
\Vert w_{t} \Vert_{L_p} \leq (1-\mu \alpha/4 )^t \Vert w_{0} \Vert_{L_p}+4\frac{R_A^2}{\mu}.
\]
This concludes the proof.
\end{proof}
As a corollary, we conclude that:
\begin{corollary}\label{corollary::w_t_moments}
If $p \geq 3$, $w_0$ is sampled from the unit sphere, and $\alpha$ satisfies $\alpha \leq \min(\frac{2}{R_B^2(p-2)}, \frac{4}{\mu})$ then:
\begin{equation}
\mathbb{E}\left[ \|  w_t \|_2^p    \right] \leq \left(  1 + 4\frac{R_A^2}{\mu}   \right)^p
\end{equation}
\end{corollary}
We can leverage corollary \ref{corollary::w_t_moments} to obtain the following control on the norms of $w_t$. As a warm up first we show that polynomial control on the norms of $w$ is possible.
\begin{lemma}\label{lemma::pointwise_poly_control}
Let $\eta > 0$ and $b > 0$. If:
\begin{equation}
p = \frac{1+a}{b}, \qquad c \geq \frac{\left(   1 + 4\frac{R_A^2}{\mu} \right)}{\eta^{1/p}}  \left( \sum_{j=1}^\infty \frac{1}{j^{1+a}}  \right)^{1/p}
\end{equation}
Then whenever $\alpha \leq \min(\frac{2}{R_B^2(p-2)}) $, we have that with probability $1-\eta$, $\| w_t \| \leq c t^b$ for all $t \leq n$.
\end{lemma}
\begin{proof}
By Corollary \ref{corollary::w_t_moments} and Markov's inequality:
\begin{equation*}
\Pr\left( \|w_t \|^p \geq c^p t^{b p}   \right) \leq \frac{\mathbb{E}\left[  \| w_t\|^p \right] }{c^p t^{b p}} \leq \left( \frac{1+ 4R_A^2/\mu}{c}\right)^p \frac{1}{t^{b p}}  \leq \eta \left( \frac{1}{\sum_{j=1}^\infty \frac{1}{j^{1+a} }}\right) \frac{1}{t^{b p}}
\end{equation*}
The first inequality follows by Markov, the second by Corollary \ref{corollary::w_t_moments} and the third by the definition of $c$, and $p$.Applying the union bound to all $w_t$ from $t = 1$ to $\infty$ yields the desired result.
\end{proof}

The lemma above implies that for any probability level $\eta$ , whenever the step size $\alpha_t$ is a small enough constant, independent of time $t$, by picking $\alpha$ small enough, we can show pointwise control on the norms of $\|w_t\|$ with constant probability so that at time $t$, $\|w_t\| \leq c t^b$.

Notice that for a fixed $a$, $\sum_{j=1}^\infty \frac{1}{j^{1+a} }$ converges, and that in case $a \geq 1$, $\sum_{j=1}^\infty \frac{1}{j^{1+a}} < 10$ (an absolute constant).

We now proceed to show that in fact for any $\delta > 0$, there is a constant $\mathcal{C}(\delta, \mu, R_B, R_A, \log(d))$ such that with probability $1-\delta$, $w_t < \mathcal{B}(\delta, \mu, R_B, R_A, \log(d))$ for all $t$ whenever the step size is  $\alpha_t = \frac{c}{\log(d^2\beta + t)} $ with $\beta \geq 0$.

We start with the following observation:
\begin{lemma}\label{lemma::boundedness_w_t_1}
Let $t_0 \in \mathbb{N}$ and $t_1 = 2t_0$. Assume $\| w_{t_0} \| \leq \mathcal{B}$. Then for all $t_0+k \in [t_0 + \frac{8 \log(\mathcal{B}) \log(d^2 \beta  + t_0)}{\mu c}, \cdots, t_1]$, the following holds:
\begin{equation*}
\mathbb{E}\left[  \|  w_{t_0+k}  \|^{c_1 \log(t_1)}  \right]  \leq (1+\frac{8R_A^2}{\mu})^{c_1 \log(t_1)}
\end{equation*}
Where $\alpha_{t_0+k} = \frac{c}{\log(d^2\beta + t_0+k)}$, $t_0 \geq 2$. And $c,c_1$ are positive constants such that $c \leq \frac{1}{R_B^2 c_1}$.
\end{lemma}
\begin{proof}
Mimicking the proof of Lemma \ref{lemma::wt_moment_bound}, the same result of said Lemma holds up to Equation \ref{eq::wt_moment_important_equation} even if the step size $\alpha_{ t_0+m} = \frac{c}{\log(d^2\beta + t_0+m)}$, therefore for any $m$:
\begin{equation*}
\| w_{t_0 + m+1} \|_{L_p} \leq \| w_{t_0 + m}\|_{L_p}\left(1 - \frac{\alpha_{t_0+m}}{2}\left(1-\alpha_{t_0 + m} R_B^2 \frac{p-2}{4}   \right)\mu\right) + \alpha_{t_0 + k} R_A^2
\end{equation*}
Let $\delta_{t_0+m} = \frac{\alpha_{t_0 + m}  }{2} \left( 1-\alpha_{t_0 + m} R_B^2 \frac{p-2}{4}  \right)\mu$, we obtain the recursion:
\begin{equation*}
\| w_{t_0 + m+1} \|_{L_p} \leq \| w_{t_0 + m}\|_{L_p} (1 - \delta_{t_0 + m}) + \alpha_{t_0 + m} R_A^2
\end{equation*}
Which for any $k$ can be expanded to:
\begin{equation*}
\| w_{t_0 +k}  \|_{L_p} \leq \prod_{m=0}^{k-1} (1-\delta_{t_0+m}) \| w_{t_0} \|_{L_p} + R_A^2 \sum_{m'=0}^{k-1} \alpha_{t_0+m'} \prod_{j=m'+1}^{k-1}(1-\delta_{t_0+j})
\end{equation*}
We now show that we can substitute all instances of $\delta_{t_0+k}$ in the upper bound with a fixed quantity, which will allow us to bound the whole expression afterwards.

Notice that $\alpha_{t_0 +k}$ is decreasing and that  $\delta_{t_0+k } \geq \frac{\alpha_{t_1}}{2} \left(1-2\alpha_{t_1} R_B^2 \frac{p-2}{4}  \right) \mu$. The later follows because by assumption $\alpha_{t_0 + k} = \frac{c}{\log(d^2\beta + t_0+k)} \leq 2\frac{c}{\log(d^2\beta + t_1)} = 2 \alpha_{t_1}$ (recall that $t_1 = 2t_0$, implying this is true as long as $t_0 \geq 2$) and therefore $\alpha_{t_1} \leq \alpha_{t_0 + k} \leq 2 \alpha_{t_1}$.

Define $\delta_{t_1}' :=\frac{\alpha_{t_1}}{2} \left(1-2\alpha_{t_1} R_B^2 \frac{p-2}{4}  \right) \mu $. As a consequence:
\begin{align*}
\| w_{t_0 +k}  \|_{L_p} &\leq \prod_{i=0}^{k-1} (1-\delta_{t_1}') \| w_{t_0} \|_{L_p} +2 R_A^2 \alpha_{t_1} \sum_{m'=0}^{k-1}  (1-\delta_{t_1}')^{m'} \\
&\leq \prod_{i=0}^{k-1} (1-\delta_{t_1}') \| w_{t_0} \|_{L_p} +2 R_A^2 \alpha_{t_1} \frac{1}{\delta_{t_1}'} \\
&= (1-\delta_{t_1}')^k \| w_{t_0}\|_{L_p} + 2R_A^2 \alpha_{t_1} \frac{1}{\delta_{t_1}'}
\end{align*}

If $\alpha_{t_1}  <\frac{1}{R_B^2 (p-2)}$, then $\delta'_{t_1} > \frac{\alpha_{t_1}}{4}\mu $. Then:
\begin{equation*}
\| w_{t_0+k} \|_{L_p} \leq (1-\mu \alpha_{t_1}/4)^{k} \| w_{t_0} \|_{L_p} + 8\frac{R_A^2}{\mu}
\end{equation*}
And therefore:
\begin{equation*}
\mathbb{E}\left[  \|  w_{t_0+k} \|^p \right] \leq \left( (1-\mu \alpha_{t_1}/4)^k \| w_{t_0}\|_{L_p} + 8\frac{R_A^2}{\mu}   \right)^p
\end{equation*}
Notice that $(1-\mu \alpha_{t_1}/4)^k \leq \exp( -\frac{\mu \alpha_{t_1} k}{4}  ) $ and therefore $(1-\mu \alpha_{t_1}/4)^k \| w_{t_0}\|_{L_p} \leq 1$ whenever $-\mu \alpha_{t_1} k/4 + \log(\mathcal{B}) \leq 0$. Since $2\log(d^2\beta + t_0) \geq \log(d^2\beta + t_1)$ (because $t_0 \geq 2$), the relationship $(1-\mu \alpha_{t_1}/4)^k \| w_{t_0}\|_{L_p} \leq 1$ holds (at least) whenever $ k \geq \frac{8 \log(\mathcal{B}) \log(d^2\beta+t_0)}{\mu c}$.

Recall that $p = c_1 \log(t_1)$. Since the above conditions require $\alpha_{t_1} < \frac{1}{R_B^2 (p-2)}$ to hold, it is enough to ensure that:
\begin{equation*}
\alpha_{t_1} = \frac{c}{\log( d^2\beta + t_1)} \leq \frac{1}{R_B^2p} = \frac{1}{R_B^2 c_1 \log(t_1)} < \frac{1}{R_B^2(p-2)} =  \frac{1}{R_B^2(c_1 \log(t_1) -2)}
\end{equation*}
It is enough to take $c \leq \frac{1}{R_B^2 c_1}$ to satisfy the bound. Putting all these relationships together:
\begin{equation*}
\mathbb{E}\left[  \|  w_{t_0+k} \|^p \right] \leq \left( 1 + 8\frac{R_A^2}{\mu}   \right)^p
\end{equation*}
For $p = c_1 \log(t_1)$ and for all $k$ such that $k \in [ \frac{8 \log(\mathcal{B}) \log(d^2\beta + t_0)}{\mu c}, \cdots, t_0]$.
\end{proof}

As a consequence of Lemma \ref{lemma::boundedness_w_t_1}, we have the following corollary:
\begin{corollary}\label{corollary::boundedness_w_t_1}
Let $t_0 \in \mathbb{N}$ and $t_1 = 2t_0$. Assume $\| w_{t_0} \| \leq \mathcal{B}$. Then for all $t_0+k \in [t_0 + \frac{8 \log(\mathcal{B}) \log(d^2\beta + t_0)}{\mu c}, \cdots, t_1]$, the following holds:
\begin{equation*}
\mathbb{E}\left[  \|  w_{t_0+k}  \|^{c_1 \log(t_0+k)}  \right]  \leq (1+\frac{8R_A^2}{\mu})^{c_1 \log(t_0+k)}
\end{equation*}
Where $\alpha_{t_0+k} = \frac{c}{\log(d^2\beta + t_0+k)}$, $t_0 \geq 2$. And $c,c_1$ are positive constants such that $c \leq \frac{1}{R_B^2 c_1}$.
\end{corollary}
The proof of this result follows the exact same template as the proof of Lemma \ref{lemma::boundedness_w_t_1}, the only difference is the subtitution of $p$ with the desired $c_1 \log(t_0 + k)$ wherever necessary. 

Now we proceed to show that having control up to the $c_1 \log(t)$ moments for $\| w_t \|$ implies boundedness of $w_t$ with high probability:
\begin{lemma}\label{lemma::boundedness_w_t_2}
Assume $\mathbb{E}\left[ \| w_t \|^{c_1 \log(t)}  \right] \leq (1+ 8 \frac{R_A^2}{\mu})^{c_1 \log(t)}$, and $\delta > 0$, then for $\mathcal{B} \geq 2 \left( 1+\frac{8R_A^2 }{\mu} \right) \frac{1}{\delta}$, we have:
\begin{equation*}
\Pr\left( \| w_t \| \geq \mathcal{B} \right) \leq \frac{1}{t^{c_1}} \delta^{c_1 \log(t)}
\end{equation*}
Where $\log$ is base $2$.
\end{lemma}
\begin{proof}
The proof follows from a simple application of Markov's inequality:
\begin{align*}
\Pr\left( \| w_t \| \geq  \mathcal{B} \right) &\leq \Pr\left( \| w_t \|^{c_1 \log(t) } \geq  \mathcal{B}^{c_1 \log(t)} \right) \\
& \leq \frac{1}{t^{c_1}} \delta^{c_1 \log(t)}
\end{align*}
This concludes the proof.
\end{proof}

We now show that if there is $t_0$ for which $\| w_t \| \leq \mathcal{B}$, for some large enough constant $\mathcal{B}$, then by leveraging Lemmas \ref{lemma::boundedness_w_t_1} and \ref{lemma::boundedness_w_t_2} then we can say that with any constant probability a large chunk of the $w_t$ are bounded provided $\alpha$ is time dependent $\alpha_t$ with $\alpha_t = \frac{c}{\log(d^2\beta + t)}$ for some constant $c$.

\begin{lemma}\label{lemma::final_boundedness_w_t}
Let $\delta >0$, define $\eta := \frac{\sum_{j=1}^\infty \frac{1}{j^2}}{\delta}$,  and let the step size $\alpha_t = \frac{c}{\log(d^2\beta +t) }$ with $c > 0$ satisfying $c \leq \frac{1}{2R_B^2}$. Assume there exists $t_0 \geq 2$ such that $\| w_{t_0}\| \leq \mathcal{B}$ with $\mathcal{B} \geq 2\left( 1+\frac{8R_A^2}{\mu}  \right) \eta$. Define $t_1 = 2t_0$ and $t_{i+1} = 2t_{i}$ for all $i \geq 1$. With probability $1-\delta$ it holds that for all $t \geq t_0$ such that $t \in [ t_i + \frac{2\log(\mathcal{B})\log(d^2\beta + t_i)}{\mu}2R_B^2, \cdots, t_{i+1}]$ it follows that:
\begin{equation*}
\| w_t \| \leq  \mathcal{B}
\end{equation*}
\end{lemma}
\begin{proof}
The proof is a simple application of Lemmas \ref{lemma::boundedness_w_t_1} and \ref{lemma::boundedness_w_t_2}. Indeed, by Lemma \ref{lemma::boundedness_w_t_1} and the assumptions on $w_{t_0}$ and the step size, conditioning on the event that $w_{t_0} \leq \mathcal{B}$, the $2\log(t_1)$ moments (and in fact the $2\log(t)$ moments as well) of $\|w_t\|$ for $t \in  [ t_0 + \frac{2\log(\mathcal{B})\log(d^2\beta + t_0)}{\mu}2R_B^2, \cdots, t_{1}]$ are bounded by $(1+\frac{8R_A^2}{\mu})^{2\log(t_1)}$ (respectively $(1+\frac{8R_A^2}{\mu})^{2\log(t)}$ for the $2\log(t)$ moments). This in turn implies by Lemma \ref{lemma::boundedness_w_t_2}, that conditional on $\|w_{t_0}\|\leq \mathcal{B}$, for any $t \in [ t_0 + \frac{2\log(\mathcal{B})\log(d^2\beta + t_0)}{\mu}2R_B^2, \cdots, t_{1}]$ the probability that $\| w_t \|$ is larger than $\mathcal{B}$ is upper bounded by $\frac{1}{t^2}\frac{1}{\eta^{2\log(t)} }\leq\frac{1}{t^2}\frac{\delta}{\sum_{j=1}^\infty \frac{1}{j^2}}$ (this inequality follows because $\eta \geq 1$ and $2\log(t) \geq 1$ as well). Consequently, the probability that any $\|w_t\| > \mathcal{B}$ for $t \in [ t_0 + \frac{2\log(\mathcal{B})\log(d^2\beta + t_0)}{\mu}2R_B^2, \cdots, t_{1}]$ can be bounded by the union bound as:
\begin{equation*}
\frac{\delta}{\sum_{j=1}^\infty \frac{1}{j^2}} \sum_{t \in  [ t_0 + \frac{2\log(\mathcal{B})\log(d^2\beta + t_0)}{\mu}2R_B^2, \cdots, t_{1}]} \frac{1}{t^2}
\end{equation*}

Conditioning on $\| w_{t_1} \| \leq \mathcal{B}$ and repeating the argument, for all $i$, we obtain that the probability that there is any $t$ such that $\| w_t \| > \mathcal{B}$ and $t \in [ t_i + \frac{2\log(\mathcal{B})\log(d^2\beta + t_i)}{\mu}2R_B^2, \cdots,  t_{i+1}]$ is at most:
\begin{equation*}
\frac{\delta}{\sum_{j=1}^\infty \frac{1}{j^2}}  \sum_{i=0}^\infty \sum_{t \in  [ t_i + \frac{2\log(\mathcal{B})\log(d^2\beta + t_i)}{\mu}2R_B^2, \cdots, t_{i+1}]} \frac{1}{t^2} \leq \delta
\end{equation*}
This concludes the proof.
\end{proof}

Now we show that in fact, for any $\delta \in (0,1)$, then, with probability $1-\delta$, for all $t$, all $w_t$ are bounded (by a quantity that depends inversely on $\delta$). More formally:
\begin{lemma}\label{lemma::final_result_w_t_boundedness}
Define $R_A$ and $R_B$ such that $R_A = R_B \geq \frac{1}{2}$. Let $$\mathcal{B}= \max\left( 1+ \frac{1}{R_B},(1+\frac{8R_{A}^2}{\mu}) \frac{ \sum_{j=1}^\infty \frac{1}{j^2} }{\delta},  2, (5 + 72\cdot \frac{\log^2(1+d^2\beta)R_B^3}{\mu^2})^2  \right).$$ If $\alpha_t = \frac{c}{\log(d^2\beta + t)}$ with $c = \frac{1}{2R_B^2}$ and $\| w_0 \| =1$, then with probability $1-\delta$ for all $t$:
\begin{equation*}
\| w_t \| \leq \mathcal{B} + \frac{2\log(\mathcal{B}) R_B}{\mu} := \mathcal{C}(\delta, \mu, R_B, R_A, \log(d))
\end{equation*}
\end{lemma}
\begin{proof}
Let $t_0 = \max( \left( \frac{4}{3}*\frac{4\log(1+ d^2\beta) \log(\mathcal{B}) R_B^2}{\mu} \right)^2, 2 ) $. Define $t_1 = 2t_0$ and in general for all $i \geq 1$, $t_i = 2t_{i-1}$.
\begin{itemize}
\item We start by showing that $t_0 \geq 4 \frac{\log(\mathcal{B} ) \log(d^2\beta + t_0) R_B^2 }{\mu}$, which will allow us to show that the interval $[t_0 + 4 \frac{\log\log(\mathcal{B} ) \log(d^2\beta + t_0) R_B^2 }{\mu}, \cdots, t_1]$ is nonempty.
\end{itemize}

First notice that for all $t \geq 1$, (and in particular for all $t \geq 2$), we have that:
\begin{equation*}
\frac{t}{\log_2(t)} \geq \frac{3}{4}t^{\frac{1}{2}}
\end{equation*}
Therefore:
\begin{equation*}
\frac{t_0}{\log(t_0)} \geq \frac{3}{4} t_0^{1/2} \geq \max(   \left( \frac{4\log( 1+ d^2\beta) \log(\mathcal{B}) R_B^2}{\mu} \right), 1)  \geq  \frac{4 \log( 1 + d^2\beta) \log(\mathcal{B}) R_B^2}{\mu}
\end{equation*}
And therefore, since $\log(t_0) \log(1+d^2\beta) \geq \log(d^2\beta + t_0 )$:
\begin{equation*}
t_0 \geq  \frac{4 \log(t_0) \log( 1 + d^2\beta) \log(\mathcal{B}) R_B^2}{\mu} \geq \frac{4\log(  d^2\beta + t_0) \log(\mathcal{B}) R_B^2}{\mu}
\end{equation*}
Which implies the desired inequality.
\begin{itemize}
\item Now we see that $\| w_t \| \leq \mathcal{B}$ for all $t \leq t_0$.
\end{itemize}
We use a very rough bound on $w_t$. Recall that $w_{t} = (I  - \alpha_{t-1} B_{t}) w_{t-1}  + \alpha_{t-1} A_t v_t $. The following sequence of inequalities holds:
\begin{align*}
\| w_t \| &\leq \| I - \alpha_{t-1} B_{t} \| \| w_{t-1}\| + \frac{1}{2R_B^2} \| A_t \| \\
&\leq \| w_{t-1} \| + \frac{1}{2R_B}
\end{align*}
This holds as long as $\|I - \alpha_{t-1}B_t \| \leq 1$, which is true since by assumption $B_t \succeq 0$ for all $t$ and therefore $\|  \alpha_{t-1} B_t  \| \leq \frac{1}{2R_B^2} R_B = \frac{1}{2R_B}  \leq 1$. The last inequality follows because $R_B \geq \frac{1}{2}$. Consequently, $\| w_t \| \leq 1+ \frac{t}{2R_B}$ for $t \leq t_0$. We want to ensure $\mathcal{B} \geq  1+ \frac{t_0}{2R_B}$. Notice that:
\begin{equation*}
1 + \frac{t_0}{2R_B} =1 + \frac{  \max( \left( \frac{4}{3}\frac{4 \log(1+ d^2\beta)\log(\mathcal{B}) R_B^2}{\mu}   \right)^2, 1 )}{2R_B}
\end{equation*}

If $t_0=1$, this provides the condition $\mathcal{B} \geq 1+ \frac{1}{R_B}$. When the max defining $t_0$ is achieved at $\left( \frac{4}{3}\frac{4\log( 1 + d^2\beta)\log(\mathcal{B}) R_B^2}{\mu}   \right)^2$, we obtain the condition:
\begin{equation}\label{eq::wt_control_final_auxiliary}
1+ \left( \frac{4^4}{2*3^2}\frac{\log^2( 1+ d^2\beta)\log^2(\mathcal{B}) R_B^3}{\mu^2}   \right) \leq \mathcal{B}
\end{equation}

Since we already have $\mathcal{B} \geq 2$, it follows that $\log(\mathcal{B}) \geq 1$. And therefore, Equation \ref{eq::wt_control_final_auxiliary} is satisfied as long as:
\begin{equation*}
\log^2(\mathcal{B}) \left(  1+ \left( \frac{4^4}{2*3^2}\frac{\log^2(1+d^2\beta ) R_B^3}{\mu^2}   \right) \right) \leq \mathcal{B}
\end{equation*}
Notice that for all $x \geq 1$:
\begin{equation*}
\frac{x}{\log^2(x)} \geq \frac{1}{5}x^{1/2}
\end{equation*}

Therefore, picking $\mathcal{B} \geq (5 + 72\cdot \frac{\log^2(1+d^2\beta)R_B^3}{\mu^2})^2 \geq (5 + 5\frac{4^4}{2*3^2}\cdot \frac{\log^2(1+d^2\beta)R_B^3}{\mu^2})^2$ guarantees that Equation \ref{eq::wt_control_final_auxiliary} is satisfied, (since $\mathcal{B}$ is also greater than $1$).

\begin{itemize}
\item We can therefore invoke Lemma \ref{lemma::final_boundedness_w_t} to the sequence $\{t_i\}$ and conclude that with probability $1-\delta$ for all $t$ such that $t \in [t_i + \frac{4 \log(\mathcal{B}) \log(d^2\beta + t_i)R_B^2}{\mu}, \cdots, t_{i+1}]$ for some $i$, we have $\| w_t \| \leq \mathcal{B}$ simultaneously for all such $t$.  This uses the fact that $\mathcal{B} \geq \left(1+\frac{8R_{A}^2}{\mu}\right) \frac{ \sum_{j=1}^\infty \frac{1}{j^2} }{\delta}$.
\end{itemize}

\begin{itemize}
\item The final step is to show a bound on $w_t$ for the remaining blocks.
\end{itemize}

For the remaining blocks notice that if $\| w_{t_i}\| \leq \mathcal{B}$, then by a crude bound since $\alpha_t = \frac{c}{\log(d^2\beta + t)}$, with $c = \frac{1}{2R_B^2}$, at each step starting from $t_i$, $w_t$ grows by at most an additive $\frac{1}{\log(d^2\beta + t_i)}$ factor:
\begin{align*}
\| w_t \| &\leq \| I - \alpha_{t-1} B_{t} \| \| w_{t-1}\| + \frac{1}{2R_B^2\log(d^2\beta + t_i)} \| A_t \| \\
&\leq \| w_{t-1} \| + \frac{1}{2R_B \log(d^2\beta + t_i)}
\end{align*}

For all $t \in [t_i + 1 , \cdots, t_i + \frac{2\log(\mathcal{B})\log(d^2\beta + t_i)}{\mu}2R_B^2 ]$.

Since $\|w_{t_i} \| \leq \mathcal{B}$, we have that  $\| w_t \| \leq \mathcal{B} + \frac{2\log(\mathcal{B}) R_B}{\mu} $ for all $t \in [t_i + 1 , \cdots, t_i + \frac{2\log(\mathcal{B})\log(d^2\beta + t_i)}{\mu}2R_B^2 ]$.

As desired.
\end{proof}
\textbf{Notation for following sections}:
Throughout the following sections we use the following notation:

We use the assumption that $\| \mathbb{E}\left[ \epsilon_t | \mathcal{F}_{t-r_t} \right] \| \leq \mathcal{A}_t r_t \beta_{t}$ as proved in Section \ref{app:tatanos} where $r_t$ is the mixing time window at time $t$.

Also, as proved in Section \ref{sec:moment}, we have that $\| w_t \| \leq W_t$ and consequently:
\begin{equation}
\| \epsilon_t \| \leq  \| w_t - B^{-1}Av_{t-1}\| \leq W_t + \| B^{-1}A \| := B_{\epsilon_t}
\end{equation}
Additionally we also have that:
\begin{equation*}
\| G_t \| \leq \lambda_1 + B_{\epsilon_t} := \mathcal{G}_t
\end{equation*}
Notice that $B_{\epsilon_t}$ and $\mathcal{G}_t$ are of the same order.



\section{Analysis burn in times}\label{section::burn_in}

In order to provide a convergence analysis for Algorithm \ref{alg:spge}, we use Lemma \ref{lem:sinb} and bound each of the terms appearing in it. To obtain those bounds, we use a mixing time argument that allows us to bound the expected error accumulated by terms of the form $\beta_t\left( \epsilon_t H_{t-1}H_{t-1}^\top + H_{t-1}H_{t-1}^\top \epsilon_t \right)$.

To control terms of this kind we deal with the set $\{t\}$ such that $t \geq r_t$ and the set of $\{t\}$ such that $t  < r_t$ differently. Let $t_0  = \max$ $t$ such that $t <r_t$. This value $t_0$ is finite because $r_t$ grows polylogarithmically.

Recall that $r_t = O(\log^3(\frac{1}{\beta_t}))$ where $\beta_t = \frac{b}{d^2\beta+1}$. We define $r_t := \log^3(\frac{1}{\beta_t}) \mathcal{C}_r$. Where $\mathcal{C}_r$ is a constant capturing all the missing dependencies between $r_t$ and $A,B$. Let's start with an auxiliary lemma:

\begin{lemma}\label{lemma::burn_in_lemma_auxiliary}
Let $c > 0$ be some constant. If $x \geq 6!c $ then, $x^{\frac{1}{3}} \geq \log(c x)$.
\end{lemma}
\begin{proof}
Observe that $x^{\frac{1}{3}} \geq \log(c x)$ iff $\exp( x^{\frac{1}{3}} ) \geq c x$. Let's write the left hand side using its taylor series:
\begin{equation*}
\exp( x^{\frac{1}{3}} )  = \sum_{i=0}^\infty \frac{x^{\frac{i}{3}}}{i!}
\end{equation*}
Notice that $\sum_{i=0}^\infty \frac{x^{\frac{i}{3}}}{i!}  \geq \frac{x^2}{6!}$, which in turn implies that if $\frac{x^2}{6!}\geq cx $ and therefore $x \geq 6!c$, then $\exp( x^{\frac{1}{3}} ) \geq c x$, as desired.
\end{proof}
We provide an upper bound for $t_0$:
\begin{lemma}
The breakpoint $t_0$ satisfies:
\begin{equation*}
t_0 = \max( \mathcal{B}_1(b , \mathcal{C}_r) ,  \mathcal{C}_r \left(  \log(d^2\beta) + \log(b) - 1 \right)^3 )
\end{equation*}
Where $\mathcal{B}_1(b, \mathcal{C}_r):= 1440 \frac{\mathcal{C}_r^2}{b}$ is a constant dependent only on $b$ and $\mathcal{C}_r$.
\end{lemma}
\begin{proof}
We would like to show $t_0$ satisfies the property that for all $t\geq t_0$, it follows that $t \geq \mathcal{C}_r \log^3(\frac{1}{\beta_t})$. This is true iff $t^{\frac{1}{3}} - \mathcal{C}_r^{\frac{1}{3}}\log(\frac{1}{\beta_t}) \geq 0$. The following sequence of equalities holds:
\begin{align*}
t^{\frac{1}{3}} - \mathcal{C}_r^{\frac{1}{3}}\log(\frac{1}{\beta_t})  &= t^{\frac{1}{3}} - \mathcal{C}_r^{\frac{1}{3}}\log(d^2 \beta + t) + \log(b) \mathcal{C}_r^{\frac{1}{3}} \\
&= t^{\frac{1}{3}} - \mathcal{C}_r^{\frac{1}{3}}\log(\frac{d^2\beta +t}{t}) - \mathcal{C}_r^{\frac{1}{3}} \log(t) + \log(b) \mathcal{C}_r^{\frac{1}{3}}\\
&= t^{\frac{1}{3}} - \mathcal{C}_r^{\frac{1}{3}}\log(\frac{d^2\beta}{t} + 1) - \mathcal{C}_r^{\frac{1}{3}} \log(t) + \log(b) \mathcal{C}_r^{\frac{1}{3}}
\end{align*}
We now massage this expression by considering two cases and making use of the following inequality: For $\log(1+x) \leq \log(x) + 1$ if $x \geq 1$
\begin{itemize}
\item[Case 1]: $t \geq d^2\beta$
\end{itemize}
This implies that $\log( \frac{d^2\beta}{t} + 1  ) \leq \log(1+1) = 1$.
The following inequalities hold:
\begin{align*}
 t^{\frac{1}{3}} - \mathcal{C}_r^{\frac{1}{3}}\log(\frac{d^2\beta}{t} + 1) - \mathcal{C}_r^{\frac{1}{3}} \log(t) + \log(b) \mathcal{C}_r^{\frac{1}{3}} &\geq  t^{\frac{1}{3}} - \mathcal{C}_r^{\frac{1}{3}}- \mathcal{C}_r^{\frac{1}{3}} \log(t) + \log(b) \mathcal{C}_r^{\frac{1}{3}} \\
 &= t^{\frac{1}{3}} - \mathcal{C}_r^{\frac{1}{3}} \left( 1-\log(b) + \log(t)   \right) \\
  &= t^{\frac{1}{3}} - \mathcal{C}_r^{\frac{1}{3}} \left( \log\left(\frac{2}{b}\right) + \log(t)   \right) \\
&= t^{\frac{1}{3}} - \mathcal{C}_r^{\frac{1}{3}} \left( \log\left(\frac{2t}{b}\right)   \right)
\end{align*}
Let $t = \mathcal{C}_r h$. Substituting into the previous equation, we would like to find a condition for $h$ such that  $t^{\frac{1}{3}} - \mathcal{C}_r^{\frac{1}{3}} \left( \log\left(\frac{2t}{b}\right)   \right) = \mathcal{C}_r^{\frac{1}{3}}\left(   h^{\frac{1}{3}} - \log(\frac{2 \mathcal{C}_r h}{b} )  \right) \geq 0$. This follows as long as $h \geq  6! \frac{2\mathcal{C}_r}{b}  =   1440 \frac{\mathcal{C}_r}{b}$ by Lemma \ref{lemma::burn_in_lemma_auxiliary}. Let $\mathcal{B}_1(b , \mathcal{C}_r) = 1440 \frac{\mathcal{C}_r^2}{b}  $.

We conclude that as long as we have $t \geq \mathcal{B}_1(b, \mathcal{C}_r)$ for some constant $\mathcal{B}_1(b, \mathcal{C}_r)$ depending on $\gamma$ and $\mathcal{C}_r$, we can guarantee that $t^{\frac{1}{3}} - \mathcal{C}_r^{\frac{1}{3}} \left( \log\left(\frac{2t}{b}\right)   \right) \geq 0$.
\begin{itemize}
\item[Case 2]: $t < d^2\beta$.
\end{itemize}
This implies that $\log(\frac{d^2\beta}{t} + 1) \leq \log(\frac{d^2\beta}{t}) + 1$.
The following inequalities hold:
\begin{align*}
 t^{\frac{1}{3}} - \mathcal{C}_r^{\frac{1}{3}}\log(\frac{d^2\beta}{t} + 1) - \mathcal{C}_r^{\frac{1}{3}} \log(t) + \log(b) \mathcal{C}_r^{\frac{1}{3}}  &\geq  t^{\frac{1}{3}} - \mathcal{C}_r^{\frac{1}{3}}\log(\frac{d^2\beta}{t} ) - \mathcal{C}_r^{\frac{1}{3}}-\mathcal{C}_r^{\frac{1}{3}} \log(t) + \log(b) \mathcal{C}_r^{\frac{1}{3}}  \\
 &= t^{\frac{1}{3}} - \mathcal{C}_r^{\frac{1}{3}}\log(d^2\beta) - \mathcal{C}_r^{\frac{1}{3}} + \log(b) \mathcal{C}_r^{\frac{1}{3}}
\end{align*}
And therefore the last expression is greater than zero if $t \geq \mathcal{C}_r \left(  \log(d\beta) + \log(b) - 1 \right)^3$. As a consequence we get that as long as $t \geq t_0 = \max( \mathcal{B}_1(b , \mathcal{C}_r) ,  \mathcal{C}_r \left(  \log(d\beta) + \log(b) - 1 \right)^3 ) $ we have that $t \geq \mathcal{C}_r \log^3(\frac{1}{\beta_t})$ as desired.
\end{proof}
Throughout the next sections, we use $t_0$ to denote this breakpoint.

\section{Analysis for \alg}
In this section, we provide bounds on expectations of various terms appearing in Lemma \ref{lem:sinb} which are required to obtain a convergence bound for \alg.
\subsection{Upper Bound on Operator Norm of $\mathbb{E}\left[ H_t H_t^\top  \right] $}\label{section::lemma_1}
We start by showing an upper bound for $\|\mathbb{E}\left[ H_t H_t^\top  \right] \|$.
\begin{lemma}\label{lemma1}
For all $t \geq 0$:
\begin{equation*}
\|\mathbb{E}\left[ H_t H_t^\top  \right] \| \leq \exp\left( 2\sum_{i=1}^t \beta_i  \lambda_1 + \beta_i^2 d r_i (\mathcal{A}_i + B_{\epsilon_i}\mathcal{G}_i + B_{\epsilon_i}^2)\mathcal{C}^{(1)}   + \sum_{j=1}^{t_0} \beta_j 2dB_{\epsilon_j} \right)
\end{equation*}
Where $C^{(1)} $ is a constant. Assuming that for all $t \geq 0$,\ $\beta_t r_t \mathcal{G}_t < \frac{1}{4}$.
\end{lemma}
\begin{proof}
We start by substituting the identity: $H_t = (I + \beta_t G_t) H_{t-1} = (I  + \beta_t B^{-1}A + \beta_t \epsilon_t)H_{t-1}$ into the expectation:
\begin{align*}
\mathbb{E}\left[ H_t H_t^\top \right] &= \mathbb{E}\left[  (I  + \beta_t B^{-1}A + \beta_t \epsilon_t) H_{t-1}H_{t-1}^\top (I  + \beta_t B^{-1}A + \beta_t \epsilon_t)^{\top}  \right] \\
&= (I + \beta_t B^{-1}A )\mathbb{E}\left[  H_{t-1} H_{t-1}^\top\right] (I+\beta_t B^{-1}A ) ^\top + \beta_t \mathbb{E}\left[ \epsilon_t H_{t-1}H_{t-1}^\top + H_{t-1}H_{t-1}^\top \epsilon_t^\top \right]  \\
& \quad + \beta_t^2 \mathbb{E}\left[ \epsilon_t H_{t-1}H_{t-1}^\top\epsilon_t^\top\right]
\end{align*}

If we assume to have a series of upper bounds $\theta_{1} \leq \cdots \leq \theta_{t-1}$ such that:
\begin{equation}
\mathbb{E}\left[ B_{u}B_{u}^\top \right] \preceq \theta_{u}I
\end{equation}
The following inequality holds:
\begin{equation}\label{eq::lemma1_good_term_bound}
(I + \beta_t B^{-1}A )\mathbb{E}\left[  H_{t-1} H_{t-1}^\top\right] (I+\beta_t B^{-1}A )^\top \preceq \theta_{t-1} (I + \beta_t B^{-1}A ) (I + \beta_t B^{-1}A )^\top
\end{equation}
Furthermore, we show how that $(I + \beta_t B^{-1}A)(I + \beta_t B^{-1}A)^\top \preceq (1+\beta_t \lambda_1)^2I$:

Indeed, let $v$ be an eigenvector of $B^{-\frac{1}{2}}AB^{-\frac{1}{2}}$ with eigenvalue $\lambda$ and denote $\tilde{v} = B^{\frac{1}{2}}v$. We show that $\tilde{v}$ is an eigenvector of $(I+\beta_t B^{-1}A)(I + \beta_t B^{-1}A)^\top$ with eigenvalue $(1 + \beta_t \lambda)^2$:
\begin{align*}
\tilde{v}^\top ( I+\beta_tB^{-1}A)(I + \beta_t B^{-1}A)^\top \tilde{v} &= v^\top (  B^{\frac{1}{2}} +\beta_t B^{-\frac{1}{2}} A ) (  B^{\frac{1}{2}} +\beta_t B^{-\frac{1}{2}} A )^\top v  \\
&= v^\top (  B^{\frac{1}{2}} +\beta_t B^{-\frac{1}{2}} A )B^{-\frac{1}{2}}B^{\frac{1}{2}}B^{\frac{1}{2}}B^{-\frac{1}{2}} (  B^{\frac{1}{2}} +\beta_t B^{-\frac{1}{2}} A )^\top v \\
&= v^\top (  I +\beta_t B^{-\frac{1}{2}} A B^{-\frac{1}{2}} )B (  I+\beta_t B^{-\frac{1}{2}} A B^{-\frac{1}{2}} )^\top v \\
&= (1+\beta_t \lambda)^2 v^\top B v\\
&= (1+\beta_t\lambda)^2 \tilde{v}^\top \tilde{v}
\end{align*}
As a consequence, we conclude the set of eigenvalues of $(I+\beta_t B^{-1}A)(I + \beta_t B^{-1}A)^\top$ equals $\{(1+\beta_t \lambda_i)^2\}_{i=1}^d$, since the set of eigenvalues of $B^{-\frac{1}{2}}AB^{-\frac{1}{2}}$ equals $\{\lambda_i\}_{i=1}^d$, the set of eigenvalues of $B^{-1}A$. Therefore we conclude that
\begin{equation}\label{eq::lemma1_bound_on_square}
(I+\beta_t B^{-1}A)(I + \beta_t B^{-1}A)^\top \preceq (1+\beta_t \lambda_1)^2 I
\end{equation}
We proceed to bound the remaining terms.
\begin{align}
\begin{split}
\mathbb{E}\left[ \epsilon_t  H_{t-1}H_{t-1}^\top \epsilon_t^\top\right] &\leq \mathbb{E}\left[ \| \epsilon_t \| \| H_{t-1}H_{t-1}^\top\| \| \epsilon_t^\top\| \right] \\
&\leq B_{\epsilon_t}^2 \mathbb{E}\left[ \| H_{t-1}H_{t-1}^\top\|  \right]\\
&\leq B_{\epsilon_t}^2 \mathbb{E}\left[ \tr( H_{t-1}H_{t-1}^\top  )  \right] \\
&\leq d B_{\epsilon_t}^2 \| \mathbb{E}\left[ H_{t-1}H_{t-1}^\top\right] \|\\
&\leq dB_{\epsilon_t}^2 \theta_{t-1}
\end{split}\label{eq::lemma1_square_bias_bound}
\end{align}

The first step is a consequence of Cauchy Schwartz, the second step because of the uniform boundedness of $\epsilon_t$ and the last step is true because $H_{t-1}H_{t-1}^\top$ is a positive semidefinite matrix.

\textbf{Terms with a single $\epsilon_t$}: Let $H_{t-1} = \prod_{j =t-r_t + 1}^ {t-1} (I + \beta_j G_j) H_{t-r_t}$.

Define $H_{t-1}^{t-r_t+ 1} := \prod_{j = t-r_t + 1}^ {t-1} (I + \beta_j G_j)$ and $L_{t-1}^{t-r_t + 1} := H_{t-1}^{t-r_t + 1 } - I$.

In order to control this term we start by bounding $\|L_{t-1}^{t-r_t +1}\|$. For this we use a crude bound.
\begin{equation}
L_{t-1}^{t-r_t - 1} = \sum_{k = 1}^{r_t}  \left(   \sum_{ i_1>  \cdots >  i_k \in [ t-r_t - 1, \cdots, t-1]  }  \left[   \prod_{j = 1}^k \beta_{i_j} G_{i_j}  \right] \right)
\end{equation}
For any $k \in [1, \cdots, r_t ]$:
\begin{align*}
\left\|\sum_{ i_1>  \cdots >  i_k \in [ t-r_t - 1, \cdots, t-1]  }  \left[   \prod_{j = 1}^k \beta_{i_j} G_{i_j}  \right] \right\| &\leq \sum_{ i_1>  \cdots >  i_k \in [ t-r_t - 1, \cdots, t-1]  }  \left[   \prod_{j = 1}^k \|\beta_{i_j} G_{i_j}\|  \right]  \\
&\leq \sum_{ i_1>  \cdots >  i_k \in [ t-r_t - 1, \cdots, t-1]  } \mathcal{G}_t^k \beta_{t-r_t}^k  \\
&\leq \left[ r_t \mathcal{G}_t \beta_{t-r_t} \right]^k
\end{align*}
The first follows from the triangle inequality, the second because of the uniform boundedness assumptions at the beginning of the section and the third because $\binom{r_t}{k} \leq r_t^k$.


For all $t \geq 0$, since the step size condition holds:
\begin{equation*}
\left[r_t \mathcal{G}_t\beta_{t-r_t}\right]^k \leq \left[ 2r_t \mathcal{G}_t\beta_t \right]^k \leq  2r_t \mathcal{G}_t\beta_t  \leq \frac{1}{2}
\end{equation*}
Putting these rough bounds together we conclude that:
\begin{equation}\label{equation::L_bound}
\| L_{t-1}^{t-r_t - 1} \| \leq \sum_{k=1}^{r_t}  \left[ 2r_t \mathcal{G}_t\beta_t \right]^k =  \left[ 2r_t \mathcal{G}_t\beta_t \right] \frac{ 1- \left[2r_t \mathcal{G}_t\beta_t \right]^k }{ 1- \left[ 2r_t \mathcal{G}_t\beta_t \right] } \leq 2  \left[ 2r_t \mathcal{G}_t\beta_t \right] = 4r_t \mathcal{G}_t\beta_t,
\end{equation}
where we have used that $1/(1-x)\leq 2x$ for $x\in [0,1/2]$.
We can write $H_t = (I + L_{t-1}^{t-r_t+1} ) H_{t-r_t} = H_{t-r_t} + L_{t-1}^{t-r_t+1} H_{t-r_t}$. Substituting this equation into
$\mathbb{E}\left[ \epsilon_t H_{t-1}H_{t-1}^\top + H_{t-1}H_{t-1}^\top \epsilon_t^\top \right]$ gives:
\begin{small}
\begin{align*}
\mathbb{E}\left[ \epsilon_t H_{t-1}H_{t-1}^\top + H_{t-1}H_{t-1}^\top \epsilon_t^\top \right] &=  \mathbb{E}\left[ \epsilon_t (H_{t-r_t} + L_{t-1}^{t-r_t+1}H_{t-r_t}  )( H_{t-r_t} + L_{t-1}^{t-r_t+1}H_{t-r_t})^\top \right]  \\
&\quad+\mathbb{E}\left[  (H_{t-r_t} + L_{t-1}^{t-r_t+1}H_{t-r_t}  )( H_{t-r_t} + L_{t-1}^{t-r_t+1}H_{t-r_t})^\top  \epsilon_t^\top\right] \\
&= \mathbb{E}\left[ \epsilon_t H_{t-r_t}H_{t-r_t}^\top   \right] + \mathbb{E}\left[ \epsilon_t H_{t-r_t}  H_{t-r_t}^\top (L_{t-1}^{t-r_t+1})^\top   \right] \\
&\quad+ \mathbb{E}\left[  \epsilon_t L_{t-1}^{t-r_t +1}H_{t-r_t} H_{t-r_t}^\top   \right] + \mathbb{E}\left[ \epsilon_t L_{t-1}^{t-r_t+1}H_{t-r_t}H_{t-r_t}^\top ( L_{t-1}^{t-r_t + 1} )^\top  \right] \\
&\quad+\mathbb{E}\left[  H_{t-r_t}H_{t-r_t}^\top   \epsilon_t^\top\right] + \mathbb{E}\left[ H_{t-r_t}  H_{t-r_t}^\top (L_{t-1}^{t-r_t+1})^\top  \epsilon_t^\top \right] \\
&\quad+ \mathbb{E}\left[   L_{t-1}^{t-r_t +1}H_{t-r_t} H_{t-r_t}^\top   \epsilon_t^\top \right] + \mathbb{E}\left[  L_{t-1}^{t-r_t+1}H_{t-r_t}H_{t-r_t}^\top ( L_{t-1}^{t-r_t + 1} )^\top \epsilon_t^\top  \right]
\end{align*}
\end{small}
We focus first on bounding the terms of this expansion containing $L_{t-1}^{t-r_t+1}$. We analyze the term $\mathbb{E}\left[  \epsilon_t L_{t-1}^{t-r_t +1}H_{t-r_t} H_{t-r_t}^\top   \right]$.
\begin{align*}
\|\mathbb{E}\left[  \epsilon_t L_{t-1}^{t-r_t +1}H_{t-r_t} H_{t-r_t}^\top   \right] \| &\leq \mathbb{E}\left[  \|\epsilon_t\| \|L_{t-1}^{t-r_t+1}\| \| H_{t-r_t}H_{t-r_t}^\top\|  \right]  \\
&\leq B_{\epsilon_t}  4r_t \mathcal{G}_t\beta_t\mathbb{E}\left[ \| H_{t-r_t}H_{t-r_t}^\top\|   \right] \\
&\leq B_{\epsilon_t} 4r_t \mathcal{G}_t\beta_t\mathbb{E}\left[ \tr( H_{t-r_t}H_{t-r_t} ) \right] \\
&\leq B_{\epsilon_t} 4r_t \mathcal{G}_t\beta_t d \| \mathbb{E}\left[ H_{t-r_t}H_{t-r_t} \right] \|
\end{align*}
All other terms containing $L_{t-1}^{t-r_t+1}$ can be bounded in the same way. Combining these terms, we obtain the following bound for the sum of all these terms:
\begin{align*}
\| \mathbb{E}\left[ \epsilon_t H_{t-1}H_{t-1}^\top + H_{t-1}H_{t-1}^\top \epsilon_t^\top \right]\| &\leq B_{\epsilon_t} 4\mathcal{G}_tdr_t\beta_t\left(  4 + 8\mathcal{G}_t r_t\beta_t \right) \\
&= 16B_{\epsilon_t} \mathcal{G}_t d r_t \beta_t + 32 d B_{\epsilon_t} \mathcal{G}_t^2 r_t \beta_t^2 \\
&\leq 8dB_{\epsilon_t} \mathcal{G}_t \beta_t + 16 B_{\epsilon_t} \mathcal{G}_t d \beta_t r_t \\
&= 8d B_{\epsilon_t} \mathcal{G}_t \beta_t(2r_t+1)
\end{align*}
The last inequality holds because of the step size condition. It remains to bound the terms $\mathcal{E}\left[ \epsilon_t H_{t-r_t}H_{t-r_t}^\top  \right]$ and $\mathcal{E}\left[  H_{t-r_t}H_{t-r_t}^\top \epsilon_t^\top\right]$.

By assumption, we know $\| \mathbb{E}\left[ \epsilon_t  | \mathcal{F}_{t-r_t}\right] \| \leq \mathcal{A}_t \beta_t  r_t $ and therefore:
\begin{align*}
\| \mathbb{E}\left[  \epsilon_t H_{t-r_t}H_{t-r_t}^\top \right] \|  &\leq \mathbb{E}\left[ \|  \mathbb{E}\left[ \epsilon_t|\mathcal{F}_{t-r_t}  \right]  H_{t-r_t}H_{t-r_t}^\top   \|    \right]\\
&\leq \mathbb{E}\left[ \|  \mathbb{E}\left[ \epsilon_t|\mathcal{F}_{t-r_t}  \right] \| \|H_{t-r_t}H_{t-r_t}^\top   \|    \right] \\
&\leq \mathcal{A}_t r_t\beta_{t} \mathbb{E}\left[  \|H_{t-r_t}H_{t-r_t}^\top\|\right]\\
&\leq  \mathcal{A}_t r_t\beta_{t} \mathbb{E}\left[  \tr(H_{t-r_t}H_{t-r_t}^\top)\right] \\
&\leq  d \cdot \mathcal{A}_t r_t\beta_{t} \mathbb{E}\left[  \| H_{t-r_t}H_{t-r_t}^\top\| \right]  \\
&\leq  d \cdot \mathcal{A}_t r_t\beta_{t} \theta_{t-r_t} \\
&\leq  d \cdot \mathcal{A}_t r_t\beta_{t} \theta_{t-1}
\end{align*}
Combining the last bounds we get that whenever $t > t_0$:
\begin{equation}\label{eq::lemma1_final_bias_bound}
\| \mathbb{E}\left[ \epsilon_t H_{t-1}H_{t-1}^\top + H_{t-1}H_{t-1}^\top \epsilon_t^\top \right] \| \leq \left(  8dB_{\epsilon_t}\mathcal{G}_t \beta_t (2r_t+1) + d\mathcal{A}_t r_t \beta_t\right) \theta_{t-1}
\end{equation}

Also, whenever $t \leq t_0$, we have that,
\begin{equation}\label{eq::lemma1_final_bias_bound_burn_in}
\| \mathbb{E}\left[ \epsilon_t H_{t-1}H_{t-1}^\top + H_{t-1}H_{t-1}^\top \epsilon_t^\top \right] \| \leq 2 d B_{\epsilon_t} \theta_{t-1}
\end{equation}
Combining the bound of equation \ref{eq::lemma1_final_bias_bound} with equations \ref{eq::lemma1_good_term_bound},  \ref{eq::lemma1_square_bias_bound}, and \ref{eq::lemma1_final_bias_bound_burn_in} yields for $t > t_0$ 
\begin{align*}
\|\mathbb{E}\left[  H_t H_t^\top\right]\| &\leq \theta_{t-1} \| (I+ \beta_t B^{-1}A)(I + \beta_t B^{-1}A)^\top   \| + \theta_{t-1} \beta_t^2  \left(  8dB_{\epsilon_t}\mathcal{G}_t (2r_t+1) + d\mathcal{A}_t r_t   \right) \\
&+ \theta_{t-1} \beta_t^2 dB_{\epsilon}^2\\
&\leq \theta_{t-1} \left( 1 +2 \beta_t \lambda_1 + \beta_t^2 ( \Lambda_1   +  dB_\epsilon^2) + \beta_t^2  \left(  8dB_{\epsilon_t}\mathcal{G}_t (2r_t+1) + d\mathcal{A}_t r_t     \right) \right)
\end{align*}
where $\Lambda_1 = \lambda_1^2$. This gives us a recursion of the form:
\begin{equation}
\theta_t = \theta_{t-1} \left( 1+2\beta_t  \lambda_1 + \beta_t^2 d r_t (\mathcal{A}_t + B_{\epsilon_t}\mathcal{G}_t)\mathcal{C}^{(1)}     \right)
\end{equation}
where $\mathcal{C}^{(1)}$ is the smallest constant depending on $\Lambda_1$ such that:
\begin{equation}
dr_t(\mathcal{A}_t + B_{\epsilon_t}\mathcal{G}_t + B_{\epsilon_t}^2)\mathcal{C}^{(1)} \geq \Lambda_1 + dB_\epsilon^2 +  8dB_{\epsilon_t}\mathcal{G}_t (2r_t+1) + d\mathcal{A}_t r_t
\end{equation}
Similarly, whenever $t \leq t_0$, we have that
\begin{align*}
\|\mathbb{E}\left[  H_t H_t^\top\right]\| &\leq \theta_{t-1} \| (I+ \beta_t B^{-1}A)(I + \beta_t B^{-1}A)^\top   \| + \theta_{t-1} \beta_t * 2 d B_{\epsilon_t}  \\
&+ \theta_{t-1} \beta_t^2 dB_{\epsilon}^2\\
&\leq \theta_{t-1} \left( 1 + 2\beta_t \lambda_1 + \beta_t^2 ( \Lambda_1   +  dB_\epsilon^2) + \beta_t  *2 dB_{\epsilon_t}   \right) \\
&\leq \theta_{t-1} \left( 1 + 2\beta_t \lambda_1 + \beta_i^2 d r_i (\mathcal{A}_i + B_{\epsilon_i}\mathcal{G}_i + B_{\epsilon_i}^2)\mathcal{C}^{(1)} + \beta_t  *2 dB_{\epsilon_t}   \right)
\end{align*}
Using the inequality $(1+x) \leq \exp(x)$ for $x \geq 0$, and noting that $\theta_0 =1$ we obtain the desired result:
\begin{equation*}
\theta_t \leq \exp( \sum_{i=1}^t 2 \beta_i \lambda_1 + \beta_i^2 d r_i (\mathcal{A}_i + B_{\epsilon_i}\mathcal{G}_i + B_{\epsilon_i}^2)\mathcal{C}^{(1)}   + \sum_{j=1}^{t_0} \beta_j 2dB_{\epsilon_j}  )
\end{equation*}
\end{proof}

\subsection{Orthogonal Subspace: Upper Bound on Expectation of $\tr(V_\perp^\top H_t H_t^\top V_\perp   ) $ }\label{section::lemma_2}
In this section, we provide a bound on $\mathbb{E}\left[ \tr(V_\perp^\top H_t H_t^\top V_\perp   )   \right]$.

\begin{lemma}\label{lemma2}
For all $t>0$ and $\beta_t$ is such that $\beta_t \mathcal{G}_t r_t <\frac{1}{4}$ (which can be obtained by appropriately controlling the constant $\beta$ in the step size).,
\begin{align*}
&\mathbb{E}\left[ \tr(V_\perp^\top H_t H_t^\top V_\perp   )   \right]\\
&\qquad\leq \exp\left(   \sum_{j=1}^t 2\beta_j \lambda_2 + \beta_j^2  \lambda_2^2  \right) \left(  \tr(V_\perp V_\perp^\top) + d\|V_{\perp}V_{\perp}^\top\|_2 \sum_{i=1}^t \left( r_i \beta_i^2   \mathcal{S}_i \mathcal{C}^{(2)}      +  1(i \leq t_0) \beta_i *2dB_{\epsilon_i}   \right)\cdot\right.\\
&\qquad\qquad \left.\exp\left(  \sum_{j=1}^i  2\beta_j (\lambda_1 - \lambda_2) + \beta_j^2 (\mathcal{S}_j dr_j\mathcal{C}^{(1)} - \lambda_2^2) + \sum_{j=1}^{\min(i, t_0)}  \beta_j 2*d B_{\epsilon_j}    \right)          \right)
\end{align*}
where the $V_\perp $ matrix contains in its columns $\tilde{u}_2, \ldots, \tilde{u}_d$, where each $\tilde{u}_1 = Bu_i$ is the unnormalized left eigenvector of the matrix $B^{-1}A$ and $\mathcal{S}_i = (  \mathcal{A}_i + B_{\epsilon_i}\mathcal{G}_i +B_{\epsilon_i}^2  ) $ for all $i$.
\end{lemma}
\begin{proof}
Let $\gamma_t = \mathbb{E}\left[  \tr( V_\perp^\top H_t H_t^\top V_\perp  )  \right]$. By definition:
\begin{align*}
\gamma_t &= \tr( \mathbb{E}\left[  H_tH_t^\top\right] V_\perp V_\perp^\top )\\
&= \underbrace{\tr( \mathbb{E}\left[ H_{t-1}H_{t-1}^\top  \right] (I + \beta_t B^{-1}A)^\top V_\perp V_\perp^\top (I+\beta_t B^{-1}A))}_{  \spadesuit} \\
&\quad+ \underbrace{\tr( \beta_t \mathbb{E}\left[  \epsilon_t H_{t-1}H_{t-1}^\top + H_{t-1}H_{t-1}^\top \epsilon_t^\top    \right] V_\perp V_\perp^\top + \beta_t^2 \mathbb{E}\left[ \epsilon_t H_{t-1}H_{t-1}^\top \epsilon_t^\top )\right] V_\perp V_\perp^\top)}_{\square}
\end{align*}
We focus on term $\spadesuit$:
\begin{align*}
 \underbrace{(I + \beta_t B^{-1}A)^\top V_\perp V_\perp^\top (I+\beta_t B^{-1}A)}_{\spadesuit_0}  &=  \underbrace{V_\perp V_\perp^\top + \beta_t (B^{-1}A)^\top V_\perp V_\perp^\top + \beta_t V_\perp V_\perp^\top (B^{-1}A)}_{\spadesuit_1} \\
 &\quad + \beta_t^2 (B^{-1}A)^\top V_\perp V_\perp^\top B^{-1}A
\end{align*}
\textbf{Analysis of $\spadesuit_1$}: {We begin by noting that the columns of $V_\perp$ contain the vectors $\tilde{u}_i$ which are the unnormalized left eigenvectors of $B^{-1}A$ and therefore,}
\begin{equation*}
V_{\perp}^\top (B^{-1}A) =  V_\perp^\top \Lambda,
\end{equation*}
where $\Lambda$ is a diagonal matrix with $\Lambda_{i,i} = \lambda_{i+1}\ \forall i = 2 \ldots d$. Noting that $V_{\perp} V_\perp^\top \Lambda \preceq \lambda_2 V_{\perp} V_\perp^\top $, we obtain,
\begin{equation}\label{eq:lem2_eig1}
\spadesuit_1 \preceq V_\perp V_\perp^\top ( 1 +2 \beta_t \lambda_2).
\end{equation}
Following a similar argument, we obtain that,
\begin{equation}\label{eq:lem2_eig2}
(B^{-1}A)^\top V_\perp V_\perp^\top B^{-1}A \preceq \lambda_2^2 V_\perp V_\perp^\top.
\end{equation}
Combining Eqs \eqref{eq:lem2_eig1} and \eqref{eq:lem2_eig2}, we obtain,
\begin{equation*}
\spadesuit \leq \tr\left( \mathbb{E}[H_{t-1}H_{t-1}^T]  V_\perp V_\perp^\top (1+2\beta_t \lambda_2 + \beta_t^2 \lambda_2^2   ) \right)
\end{equation*}
The terms corresponding to $\square$ can also be bounded by bonding the operator norms of its two constituent expectations. In the same way as in Lemma \ref{lemma1}, let $H_{t-1 } = (I + L_{t-1}^{t-r_t+1} )  H_{t-r_t}$. {Note that $V_{\perp}V_{\perp}^\top \preceq \|V_{\perp}V_{\perp}^\top\|_2I$ and we bound the normalized term $\nicefrac{\square}{\|V_{\perp}V_{\perp}^\top\|_2}$}.
\begin{align*}
&\frac{\tr( \mathbb{E}\left[  \epsilon_t H_{t-1}H_{t-1}^\top + H_{t-1}H_{t-1}^\top \epsilon_t^\top \right] V_\perp V_\perp^\top )}{\|V_{\perp}V_{\perp}^\top\|_2} \leq \tr( \mathbb{E}\left[  \epsilon_t H_{t-1}H_{t-1}^\top + H_{t-1}H_{t-1}^\top \epsilon_t^\top \right]  ) \\
&=\underbrace{ \tr(\mathbb{E}\left[ H_{t-r_t}H_{t-r_t}^\top (\epsilon_t + \epsilon_t^\top)  \right]   )}_{\Gamma_1}  +  \underbrace{\tr(   \mathbb{E} \left[  H_{t-r_t}H_{t-r_t}^\top \left( (L_{t-1}^{t-r_t+1} )^\top \epsilon_t^\top + \epsilon_t L_{t-1}^{t-r_t+1}       \right)  \right]) }_{\Gamma_2}\\
&\quad +\underbrace{\tr(   \mathbb{E} \left[  H_{t-r_t}H_{t-r_t}^\top \left( (L_{t-1}^{t-r_t+1} )^\top \epsilon_t + \epsilon_t^\top L_{t-1}^{t-r_t+1}  \right)  \right])}_{\Gamma_3} \\
&\quad +\underbrace{\tr(\mathbb{E}\left[ H_{t-r_t}H_{t-r_t}^\top \left( (L_{t-1}^{t-r_t+1})^\top \epsilon_t L_{t-1}^{t-r_t+1} + (L_{t-1}^{t-r_t+1})^\top \epsilon_t^\top L_{t-1}^{t-r_t+1} \right) \right]    ) }_{\Gamma_4}
\end{align*}
Recall that $\| L_{t-1}^{t-r_t+1}\| \leq 4r_t \mathcal{G}_t\beta_t$. As a consequence:
\begin{align*}
(L_{t-1}^{t-r_t+1} )^\top \epsilon_t^\top + \epsilon_t L_{t-1}^{t-r_t+1}   &\preceq 2B_{\epsilon_t} *4r_t \mathcal{G}_t\beta_t I = 8 B_{\epsilon_t} r_t \mathcal{G}_t \beta_t I \\
(L_{t-1}^{t-r_t+1} )^\top \epsilon_t + \epsilon_t^\top L_{t-1}^{t-r_t+1}  &\preceq  2B_{\epsilon_t} *4r_t \mathcal{G}_t\beta_t I = 8 B_{\epsilon_t} r_t \mathcal{G}_t \beta_t I \\
(L_{t-1}^{t-r_t+1})^\top \epsilon_t L_{t-1}^{t-r_t+1} + (L_{t-1}^{t-r_t+1})^\top \epsilon_t^\top L_{t-1}^{t-r_t+1} &\preceq 2 (4r_t\mathcal{G}_t\beta_t)^2 B_{\epsilon _t}I \preceq  8 B_{\epsilon_t} r_t \mathcal{G}_t \beta_t  I
\end{align*}

The second inequality in the last line follows from the step size condition. Therefore:
\begin{align*}
\Gamma_2 + \Gamma_3 + \Gamma_4  &\leq  32B_{\epsilon_t} r_t \mathcal{G}_t\beta_t   \tr(  \mathbb{E}\left[ H_{t-r_t}H_{t-r_t}^\top    \right] )  \\
&\leq 32B_{\epsilon_t} r_t \mathcal{G}_t\beta_t  d \| \mathbb{E}\left[ H_{t-r_t}H_{t-r_t}^\top  \right]\| \\
&\leq 32B_{\epsilon_t} r_t \mathcal{G}_t\beta_t  d \theta_{t-r_t}\\
&\leq 32B_{\epsilon_t} r_t \mathcal{G}_t\beta_t  d \theta_{t-1}
\end{align*}
We proceed to bound $\Gamma_1$. We know that $\| \mathbb{E}\left[ \epsilon_t | \mathcal{F}_{t-r_t}  \right]\| = \mathcal{A}_t\beta_t r_t$ and therefore $\mathbb{E}\left[    \epsilon_t + \epsilon_t^\top | \mathcal{F}_{t-r_t}\right] \preceq  2\mathcal{A}_t \beta_t r_t  I $
\begin{align*}
\Gamma_1 = \tr( \mathbb{E}\left[  H_{t-r_t}H_{t-r_t}^\top (\epsilon_t + \epsilon_t^\top )  \right] )  &= \mathbb{E}\left[  \tr( H_{t-r_t}H_{t-r_t}^\top  \mathbb{E}\left[   \epsilon_t + \epsilon_t^\top | \mathcal{F}_{t-r_t}\right]  )        \right] \\
&\leq 2\mathcal{A}_t\beta_t r_t \mathbb{E}\left[ H_{t-r_t}H_{t-r_t}^\top  \right]\\
&\leq 2\mathcal{A}_t \beta_t r_t d \|\mathbb{E}\left[   H_{t-r_t} H_{t-r_t}^\top \right] \|\\
&\leq 2 \mathcal{A}_t \beta_t r_t d \theta_{t-r_t}\\
&\leq 2 \mathcal{A}_t \beta_t r_t d \theta_{t-1}
\end{align*}
The last inequalities follow from the same argument as in equation \ref{eq::lemma1_final_bias_bound}, where $\theta_{t-1}$ is the upper bound obtained in the previous lemma for $\| \mathbb{E} \left[ H_{t-1}H_{t-1}^\top  \right] \|$.

As a consequence, {whenever $t \geq t_0$} the first term in $\nicefrac{\square}{\|V_{\perp}V_{\perp}^\top\|_2}$ can be bounded by:
\begin{equation*}
\frac{\tr( \mathbb{E}\left[  \epsilon_t H_{t-1}H_{t-1}^\top + H_{t-1}H_{t-1}^\top \epsilon_t^\top \right] V_\perp V_\perp^\top )}{\|V_{\perp}V_{\perp}^\top\|_2} \leq 32B_{\epsilon_t} r_t \mathcal{G}_t\beta_t  d \theta_{t-1} + 2 \mathcal{A}_t \beta_t r_t d \theta_{t-1}
\end{equation*}
{For the case when $t <t_0$:}
\begin{align*}
\frac{\tr\left(\mathbb{E}\left[ \epsilon_t H_{t-1}H_{t-1}^\top  + H_{t-1}H_{t-1}^\top \epsilon_t^\top\right] V_\perp V_\perp^\top \right)}{\|V_{\perp}V_{\perp}^\top\|_2} &=    \frac{\mathbb{E}\left[ \tr\left((H_{t-1}H_{t-1}^\top  ) ( V_\perp V_\perp^\top \epsilon_t + \epsilon_t^\top V_{\perp}V_\perp^\top )       \right) \right]}{\|V_{\perp}V_{\perp}^\top\|_2} \\
&\leq 2B_{\epsilon_t} \tr( \mathbb{E}\left[  H_{t-1}H_{t-1}^\top \right]) \\
&\leq 2dB_{\epsilon_t} \| \mathbb{E}\left[ H_{t-1}H_{t-1}^\top \right]\| \\
&\leq 2d B_{\epsilon_t}\theta_{t-1}
\end{align*}
where the first inequality follows because $\| V_\perp V_\perp^\top \epsilon_t + \epsilon_t^\top V_{\perp}V_\perp^\top \| \leq 2B_{\epsilon_t}\|V_{\perp}V_{\perp}^\top\|_2$.

And the second term in $\nicefrac{\square}{\|V_{\perp}V_{\perp}^\top\|_2}$ can be bounded {for all $t$}:
\begin{align*}
\frac{\tr( \mathbb{E} \left[  \epsilon_t H_{t-1}H_{t-1}^\top \epsilon_t^\top  \right] V_\perp V_\perp^\top )}{\|V_{\perp}V_{\perp}^\top\|_2}  &\leq \tr( \mathbb{E} \left[  \epsilon_t H_{t-1}H_{t-1}^\top \epsilon_t^\top  \right] ) \\
&= \tr( \mathbb{E} \left[  H_{t-1}H_{t-1}^\top \epsilon_t^\top \epsilon_t  \right] ) \\
&\leq B_{\epsilon_t}^2 \tr( \mathbb{E} \left[  H_{t-1}H_{t-1}^\top  \right] ) \\
&\leq d B_{\epsilon_t}^2  \| \mathbb{E}\left[   H_{t-1}H_{t-1}^\top \right] \|\\
&\leq dB_{\epsilon_t}^2 \theta_{t-1}
\end{align*}
Let $C^{(2)}$ be a constant such that $d r_t (  \mathcal{A}_t + B_{\epsilon_t}\mathcal{G}_t  +B_{\epsilon_t}^2  )  \mathcal{C}^{(2)} \geq 32d B_{\epsilon_t}r_t\mathcal{G}_t + 2d\mathcal{A}_t r_t  + dB_\epsilon^2    $.

The last inequalities follow from the same argument as in equation  \ref{eq::lemma1_square_bias_bound}. {We conclude that {whenever $t > t_0$}:}
\begin{align*}
\square &= \beta_t \tr( \mathbb{E}\left[  \epsilon_t H_{t-1}H_{t-1}^\top + H_{t-1}H_{t-1}^\top \epsilon_t^\top   \right] V_\perp V_\perp^\top )  + \beta_t^2 \tr( \mathbb{E}\left[  \epsilon_t H_{t-1}H_{t-1}^\top \epsilon_t^\top \right] V_\perp V_\perp^\top)\\
&\leq d r_t \beta_t^2  (  \mathcal{A}_t + B_{\epsilon_t}\mathcal{G}_t  +B_{\epsilon_t}^2  ) \mathcal{C}^{(2)} \theta_{t-1} \|V_{\perp}V_{\perp}^\top\|_2
\end{align*}
{Combining $\spadesuit$ with $\square$, {whenever $t > t_0$}:}
\begin{equation*}
\gamma_t = \spadesuit + \square \leq \gamma_{t-1}(1+ 2\beta_t \lambda_2  + \beta_t \lambda_2^2)  + d r_t \beta_t^2 (  \mathcal{A}_t + B_{\epsilon_t}\mathcal{G}_t  +B_{\epsilon_t}^2  )  \mathcal{C}^{(2)} \theta_{t-1}\|V_{\perp}V_{\perp}^\top\|_2
\end{equation*}
On the other hand, {for $t \leq t_0$}:
\begin{align*}
\square &= \beta_t \tr( \mathbb{E}\left[  \epsilon_t H_{t-1}H_{t-1}^\top + H_{t-1}H_{t-1}^\top \epsilon_t^\top   \right] V_\perp V_\perp^\top )  + \beta_t^2 \tr( \mathbb{E}\left[  \epsilon_t H_{t-1}H_{t-1}^\top \epsilon_t^\top \right] V_\perp V_\perp^\top)\\
&\leq \left(   d r_t \beta_t^2  (  \mathcal{A}_t + B_{\epsilon_t}\mathcal{G}_t  +B_{\epsilon_t}^2  ) \mathcal{C}^{(2)} + \beta_t *d B_{\epsilon_t}   \right) \theta_{t-1}\|V_{\perp}V_{\perp}^\top\|_2
\end{align*}
And {consequently}:
\begin{equation*}
\gamma_t = \spadesuit + \square \leq \gamma_{t-1}(1+ 2\beta_t \lambda_2  + \beta_t^2 \lambda_2^2)  + \left( d r_t \beta_t^2 (  \mathcal{A}_t + B_{\epsilon_t}\mathcal{G}_t  +B_{\epsilon_t}^2  )  \mathcal{C}^{(2)} + \beta_t 2*d B_{\epsilon_t} \right) \theta_{t-1}\|V_{\perp}V_{\perp}^\top\|_2
\end{equation*}

Using the bound for $\theta_{t-1}$ in Lemma \ref{lemma1} and the inequality $1+x \leq e^x$:
\begin{align*}
\gamma_t &\leq \exp(   2\beta_t \lambda_2 + \beta_t^2 \lambda_2^2 )\gamma_{t-1} +\\
&\quad \|V_{\perp}V_{\perp}^\top\|_2\left( dr_t \beta_t^2  (  \mathcal{A}_t + B_{\epsilon_t}\mathcal{G}_t  +B_{\epsilon_t}^2  )  C^{(2)} + 1(t \leq t_0) \beta_t *2dB_{\epsilon_t} \right) \cdot\\
&\quad \exp\left(   \sum_{i=1}^{t-1} 2\beta_i \lambda_1 + d r_i \beta_i^2(  \mathcal{A}_i + B_{\epsilon_i}\mathcal{G}_i  +B_{\epsilon_i}^2  )  \mathcal{C}^{(1)}    + \sum_{j=1}^{\min(t,t_0)} \beta_j 2*d B_{\epsilon_j}  \right)
\end{align*}
After doing recursion we obtain the upper bound,
\begin{align*}
\gamma_t &\leq \sum_{i=1}^t \Big[\|V_{\perp}V_{\perp}^\top\|_2 \left(  d r_i \beta_i^2  (  \mathcal{A}_i + B_{\epsilon_i}\mathcal{G}_i  +B_{\epsilon_i}^2  ) C^{(2)}   + 1( i \leq t_0) \beta_i 2d B_{\epsilon_i}    \right) \exp\left(   \sum_{j=i+1}^t 2\beta_j \lambda_2 + \beta_j^2 \lambda_2^2\right) \cdot \\
&\quad  \exp\left( \sum_{j=1}^i 2\beta_j \lambda_1 + d r_j \beta_j^2 (  \mathcal{A}_j
+ B_{\epsilon_j}\mathcal{G}_j  +B_{\epsilon_j}^2  ) C^{(1)}   + \sum_{j=1}^{\min(i,t_0)} \beta_j 2*d B_{\epsilon_j} \right) \Big] \\
&\quad + \exp(   \sum_{j=1}^t 2\beta_j \lambda_2 + \beta_j^2  \lambda_2^2  ) \gamma_0
\end{align*}
Where $\gamma_0 = \tr(V_\perp V_\perp^\top)$. Let $\mathcal{S}_i = (  \mathcal{A}_i + B_{\epsilon_i}\mathcal{G}_i +B_{\epsilon_i}^2  ) $ 
\begin{align*}
\gamma_t \leq \exp(   \sum_{j=1}^t 2\beta_j \lambda_2 + \beta_j^2  \lambda_2^2  ) &\left(  \tr(V_\perp V_\perp^\top) + d\|V_{\perp}V_{\perp}^\top\|_2 \sum_{i=1}^t \left( r_i \beta_i^2   \mathcal{S}_i \mathcal{C}^{(2)}      +  1(i \leq t_0) \beta_i *2dB_{\epsilon_i}   \right)\cdot\right.\\
&\quad \left.\exp\left(  \sum_{j=1}^i  2\beta_j (\lambda_1 - \lambda_2) + \beta_j^2 (\mathcal{S}_j dr_j\mathcal{C}^{(1)} - \lambda_2^2) + \sum_{j=1}^{\min(i, t_0)}  \beta_j 2*d B_{\epsilon_j}    \right)          \right)
\end{align*}
\end{proof}


\subsection{Lower Bound on Expectation of $\tilde{u}_1^\top H_tH_t^\top \tilde{u}_1$}
\begin{lemma}\label{lem:lemma3}
For all $t\geq 0$ and $\beta_t \geq 0$ we have,
\begin{align}
\E [\tilde{u}_1^\top H_tH_t^\top \tilde{u}_1] &\geq \|\tilde{u}_1\|_2^2 \exp{\left( \sum_{i=1}^t 2\beta_i \lambda_1 - 4\beta_i^2 \lambda_1^2\right)} - d\|\tilde{u}_1\|_2^2\sum_{i=1}^{t} \Bigg( (\beta_i^2 r_t(\mathcal{A}_t + B_{\epsilon_t}\mathcal{G}_t + B_{\epsilon_t}^2  )\mathcal{C}^{(2)} \nonumber\\
 &\hspace{-4ex}\left. + \beta_i\mathbb{I}(t  \leq t_0)(B_{\epsilon_t})) \exp\left( \sum_{j=1}^{i-1} 2\beta_j \lambda_1 + \beta_j^2 dr_j( \mathcal{A}_j + B_{\epsilon_j}\mathcal{G}_j + B_{\epsilon_j}^2) \mathcal{C}^{(1)}  +\sum_{j=1}^{\min(t-1,t_0)} \beta_j 2dB_{\epsilon_j} \right) \right),
\end{align}
where $\tilde{u}_1$ is the unnormalized left eigenvector corresponding to the maximum eigenvalue $\lambda_1$ of $(B^{-1}A)^\top)$.
\end{lemma}
\begin{proof}
Let $\gamma_t \stackrel{\Delta}{=} \E [v^\top H_tH_t^\top v]$ where $v = \nicefrac{\tilde{u}_1}{\|\tilde{u}_1 \|_2}$ be the normalized left eigenvector and $\Sigma = B^{-1}A$. Since $H_t = (I+ \beta_t G_t)$, we can obtain a bound on $\gamma_t$ as,
\begin{align}
  \gamma_t &= \E [v^\top (I+ \beta_t G_t) H_{t-1}H_{t-1}^\top (I+\beta_t G_t)^\top v]\nonumber \\
  &= \E[v^\top (I+\beta_t \Sigma)H_{t-1}H_{t-1}^\top(I+\beta_t \Sigma)^\top v ] + \beta_t\E[v^{\top}(\epsilon_t H_{t-1}H_{t-1}^\top + H_{t-1}H_{t-1}^\top \epsilon_t^\top)v]\nonumber\\
  &\quad + \beta_t^2 \E[v^\top (\epsilon_t + \Sigma)H_{t-1} H_{t-1}^\top(\epsilon_t + \Sigma)^\top v] - \beta_t^2 \E[v^\top\Sigma H_{t-1} H_{t-1}^\top \Sigma^\top  v]\nonumber\\
  &\stackrel{\zeta_1}{\geq} \E[v^\top H_{t-1}H_{t-1}^\top v ] + \beta_t \E[v^\top \Sigma H_{t-1} H_{t-1}^\top v] + \beta_t \E[v^\top  H_{t-1} H_{t-1}^\top \Sigma^\top v]\nonumber\\
  &\quad+ \beta_t\E[v^{\top}(\epsilon_t H_{t-1}H_{t-1}^\top + H_{t-1}H_{t-1}^\top \epsilon_t^\top)v]\nonumber \\
  &\stackrel{\zeta_2}{=} (1+2\lambda_1\beta_t)\gamma_{t-1} + \beta_t\underbrace{\E[v^{\top}(\epsilon_t H_{t-1}H_{t-1}^\top + H_{t-1}H_{t-1}^\top \epsilon_t^\top)v]}_{(I)},
  \label{eq:lemma3_gamma_bnd}
\end{align}
where $\zeta_1$ follows since $(\epsilon_t + \Sigma)H_{t-1} H_{t-1}^\top(\epsilon_t + \Sigma)^\top$ is a positive semi-definite matrix and $\zeta_2$ follows since $v$ is the top left eigenvector of $\Sigma$. Now, in order to bound term (I), we note that
\begin{equation*}
\E[v^{\top}(\epsilon_t H_{t-1}H_{t-1}^\top + H_{t-1}H_{t-1}^\top \epsilon_t^\top)v] \geq -\|\E[ \epsilon_t H_{t-1}H_{t-1}^\top + H_{t-1}H_{t-1}^\top \epsilon_t^\top]\|_2.
\end{equation*}
Using the bound obtained in \eqref{eq::lemma1_final_bias_bound}, we get that {for $t>t_0$},
\begin{equation*}
  \E[v^{\top}(\epsilon_t H_{t-1}H_{t-1}^\top + H_{t-1}H_{t-1}^\top \epsilon_t^\top)v] \geq - \beta_t ( 8d B_{\epsilon_t} \mathcal{G}_t (2r_t + 1) + d\mathcal{A}_t r_t)\theta_{t-1},
\end{equation*}
and {for $t \leq t_0$}, we have from equation \eqref{eq::lemma1_final_bias_bound_burn_in}
\begin{equation*}
 \E[v^{\top}(\epsilon_t H_{t-1}H_{t-1}^\top + H_{t-1}H_{t-1}^\top \epsilon_t^\top)v] \geq - 2 d B_{\epsilon_t} \theta_{t-1},
\end{equation*}
Where $\theta_{t-1}$ is defined as in \ref{lemma1}. We next use the bound from lemma \ref{lemma1} to lower bound -$\theta_{t-1}$,
\begin{align*}
  \E[v^{\top}(\epsilon_t H_{t-1}H_{t-1}^\top + H_{t-1}H_{t-1}^\top \epsilon_t^\top)v]&\geq - \beta_t ( 8d B_{\epsilon_t} \mathcal{G}_t (2r_t + 1) + d\mathcal{A}_t r_t + \mathbb{I}(t  \leq t_0)(2 d B_{\epsilon_t})) \cdot \\
  &\hspace{-5ex} \exp\left( \sum_{i=1}^{t-1}2 \beta_i \lambda_1 + \beta_i^2 d r_i (\mathcal{A}_i + B_{\epsilon_i}\mathcal{G}_i + B_{\epsilon_i}^2)\mathcal{C}^{(1)}   + \sum_{j=1}^{\min(t-1,t_0)} \beta_j 2dB_{\epsilon_j} \right).
\end{align*}

Recall that in Lemma \ref{lemma2} we defined $C^{(2)}$ as a constant such that:
$d r_t (  \mathcal{A}_t + B_{\epsilon_t}\mathcal{G}_t  +B_{\epsilon_t}^2  )  \mathcal{C}^{(2)} \geq 32d B_{\epsilon_t}r_t\mathcal{G}_t + 2d\mathcal{A}_t r_t  + dB_\epsilon^2    $, therefore:
\begin{align*}
  \E[v^{\top}(\epsilon_t H_{t-1}H_{t-1}^\top + H_{t-1}H_{t-1}^\top \epsilon_t^\top)v] &\geq - \left( \beta_t d r_t (  \mathcal{A}_t + B_{\epsilon_t}\mathcal{G}_t  +B_{\epsilon_t}^2  )  \mathcal{C}^{(2)} + \mathbb{I}(t  \leq t_0)(2 d B_{\epsilon_t}) \right) \cdot\\
  &\hspace{-4ex}\exp\left( \sum_{i=1}^{t-1}2 \beta_i \lambda_1 + \beta_i^2 d r_i (\mathcal{A}_i + B_{\epsilon_i}\mathcal{G}_i + B_{\epsilon_i}^2)\mathcal{C}^{(1)}   + \sum_{j=1}^{\min(t-1,t_0)} \beta_j 2dB_{\epsilon_j} \right).
\end{align*}

Substituting the above in equation \eqref{eq:lemma3_gamma_bnd}, we obtain the following recursion,
\begin{align*}
\gamma_t \geq (1+2\lambda_1\beta_t)\gamma_{t-1} &-
\left( \beta_t d r_t (  \mathcal{A}_t + B_{\epsilon_t}\mathcal{G}_t  +B_{\epsilon_t}^2  )  \mathcal{C}^{(2)} + \mathbb{I}(t  \leq t_0)(2 d B_{\epsilon_t}) \right) \cdot\\
  &\quad \exp\left( \sum_{i=1}^{t-1} 2\beta_i \lambda_1 + \beta_i^2 d r_i (\mathcal{A}_i + B_{\epsilon_i}\mathcal{G}_i + B_{\epsilon_i}^2)\mathcal{C}^{(1)}   + \sum_{j=1}^{\min(t-1,t_0)} \beta_j 2dB_{\epsilon_j} \right).
\end{align*}
Using the inequality $1+x \geq \exp{(x-x^2)}$ for all $x\geq 0$, along with $\gamma_0 = 1$, we obtain,
\begin{align}
  \gamma_t \geq \exp{\left( \sum_{i=1}^t 2\beta_i \lambda_1 - 4\beta_i^2 \lambda_1^2\right)} & - d\sum_{i=1}^{t} \left( (\beta_i^2 r_t(\mathcal{A}_t + B_{\epsilon_t}\mathcal{G}_t + B_{\epsilon_t}^2  )\mathcal{C}^{(2)} + \beta_i\mathbb{I}(t  \leq t_0)(B_{\epsilon_t})) \right. \cdot\nonumber \\
  &\hspace{-3ex}\left. \exp\left( \sum_{j=1}^{i-1} 2\beta_j \lambda_1 + \beta_j^2 dr_j( \mathcal{A}_j + B_{\epsilon_j}\mathcal{G}_j + B_{\epsilon_j}^2) \mathcal{C}^{(1)}  +\sum_{j=1}^{\min(t-1,t_0)} \beta_j 2dB_{\epsilon_j} \right) \right)
\end{align}
which concludes the proof of the lemma.
\end{proof}

\subsection{Upper Bound on Variance of $\tilde{u}_1^\top H_tH_t^\top \tilde{u}_1$} \label{section::lemma_4_1}
In this section, we provide an upper bound on $\mathbb{E}\left[ (v^\top H_tH_t^\top v )^2  \right]$ which will be later used in order to lower bound the requisite term using the Chebychev Inequality. We first prove an upper bound on $\mathbb{E}\left[\tr( H_t H_t^\top H_t H_t^\top       )    \right]$ and use this in the next lemma to obtain the requisite bounds.
\begin{lemma}\label{lemma4_part1}
For all $t \geq 0$:
\begin{equation*}
\mathbb{E}\left[\tr( H_t H_t^\top H_t H_t^\top       )    \right] \leq d\exp( \sum_{i=1}^t 4\lambda_1 \beta_i +   dr_i(\mathcal{A}_i + B_{\epsilon_i}^2 + B_{\epsilon_i}\mathcal{G}_i) C^{(3)} \beta_i^2 +  \sum_{j=1}^{\min(t, t_0)}  \beta_j 2*(\frac{101}{100})^3B_{\epsilon_j} )
\end{equation*}
As long as $\beta_t$ satisfies that for all $t$, $\| I + \beta_t B^{-1}A\| \leq \frac{101}{100}$, $\beta_t r_t \mathcal{G}_t < \frac{1}{4}$, and $\beta_t r_t \mathcal{B}_{\epsilon_t} < \frac{1}{4}$.
\end{lemma}
\begin{proof}
We start by substituting the identity: $H_t = (I + \beta_t G_t) H_{t-1} = (I + \beta_t B^{-1}A + \beta_t \epsilon_t)H_{t-1}$.

Substituting this decomposition intro the trace we want to bound we obtain:
\begin{align*}
\mathbb{E}\left[ \tr( H_tH_t^\top H_t H_t^\top  ) \right]  &=\tr \mathbb{E}\left[   (I + \beta_t G_t) H_{t-1} H_{t-1}^\top (I + \beta_t G_t)^\top (I+\beta_t G_t) H_{t-1}H_{t-1}^\top (I+\beta_tG_t)^\top \right]\\
&= \tr \mathbb{E}\left[   \underbrace{H_{t-1} H_{t-1}^\top}_{\Gamma_1}\underbrace{(I + \beta_t G_t)^\top (I+\beta_t G_t)}_{\Gamma_2}\underbrace{ H_{t-1}H_{t-1}^\top}_{\Gamma_1} \underbrace{(I+\beta_tG_t)^\top (I+\beta_t G_t)}_{\Gamma_2}  \right] \\
&\leq \tr \mathbb{E}\left[ H_{t-1}H_{t-1}^\top H_{t-1}H_{t-1}^\top \underbrace{ (I+\beta_tG_t)^\top (I+\beta_t G_t)(I+\beta_tG_t)^\top (I+\beta_t G_t)}_{\spadesuit}\right]
\end{align*}
where last inequality follows from the trace inequality: $\tr\left( \Gamma_1 \Gamma_2 \Gamma_1 \Gamma_2    \right) \leq \tr\left(  \Gamma_1^2 \Gamma_2^2   \right)$.

Expanding $\spadesuit$ yields:
\begin{small}
\begin{align*}
&\spadesuit = \underbrace{(I + \beta_t B^{-1}A)^\top(I + \beta_t B^{-1}A)(I + \beta_tB^{-1}A)^\top(I + \beta_tB^{-1}A)}_{\spadesuit_1} \\
&+ \underbrace{\beta_t \left(  \epsilon_t^\top (I + \beta_tB^{-1}A)(I + \beta_tB^{-1}A)^\top(I + \beta_tB^{-1}A) + (I + \beta_tB^{-1}A)^\top \epsilon_t (I + \beta_tB^{-1}A)^\top(I + \beta_tB^{-1}A) \right)}_{\spadesuit_2^{(1)}} \\
&+ \underbrace{\beta_t\left( (I + \beta_tB^{-1}A)^\top(I + \beta_tB^{-1}A)\epsilon_t^\top (I + \beta_tB^{-1}A) + (I + \beta_tB^{-1}A)^\top(I + \beta_tB^{-1}A)(I + \beta_tB^{-1}A)^\top \epsilon_t\right) }_{\spadesuit_2^{(2)}}\\
&+ \spadesuit_3,
\end{align*}
\end{small}
where $\spadesuit_3$ contains all terms with at least two $\epsilon_t$. Additionally, $\spadesuit_3$ is a symmetric matrix with norm satisfying:
\begin{align*}
\|\spadesuit_3 \| & \stackrel{\gamma_1}{\leq} \beta_t^2 \binom{4}{2}\| \epsilon_t\|^2 \| I + \beta_t B^{-1}A\|^2  + \beta_t^3 \binom{4}{3} \| \epsilon_t\|^3 \| I + \beta_t B^{-1}A\| + \beta_t^4\binom{4}{4} \| \epsilon_t \|^4 \\
&\stackrel{\gamma_2}{\leq} \beta_t^2 *6 B_{\epsilon_t}^2  (\frac{101}{100})^2  + \beta_t^3 *4 *B_{\epsilon}^3 (\frac{101}{100})  + \beta_t^4 B_{\epsilon_t}^4 \\
&\stackrel{\gamma_3}{\leq} \beta_t^2 *6 B_{\epsilon_t}^2  (\frac{101}{100})^2 + \beta_t^2 B_{\epsilon_t}^2 (\frac{101}{100}) + \beta_t^2 B_{\epsilon_t}^2 \frac{1}{16}\\
&\leq 8\beta_t^2 B_{\epsilon_t}^2
\end{align*}
where the inequality $\gamma_1$ follows from triangle, and $\gamma_2, \gamma_3$ from the step size condition. Recall that:
\begin{equation*}
\tr \mathbb{E}\left[ H_{t-1}H_{t-1}^\top H_{t-1}H_{t-1}^\top \spadesuit \right] =\tr \mathbb{E}\left[ H_{t-1}H_{t-1}^\top H_{t-1}H_{t-1}^\top \left( \spadesuit_1 + (\spadesuit_2^{(1)} + \spadesuit_2^{(2)}) + \spadesuit_3   \right) \right]
\end{equation*}

Since, as shown in equation Equation \ref{eq::lemma1_bound_on_square} we have that $(I+\beta_t B^{-1}A)(I + \beta_t B^{-1}A)^\top \preceq (1+\beta_t \lambda_1)^2 I$. then, $\spadesuit_1 \preceq (1+\beta_t\lambda_1)^4 I$ (this is because $(I+\beta_t B^{-1}A)(I + \beta_t B^{-1}A)^\top $ and $(I+\beta_t B^{-1}A)^\top(I + \beta_t B^{-1}A)$ have the same eigenvalues. And therefore $\spadesuit_1 \preceq (1+ \beta_t \lambda_1)^4 I \preceq (1+4\beta_t \lambda_1 + 11 \beta_t^2 \max(\lambda_1^4, 1)) I $ and $\spadesuit_3 \preceq  8\beta_t^2 B_{\epsilon_t}^2 I $, thus implying:
\begin{equation*}
  \tr \mathbb{E}\left[ H_{t-1}H_{t-1}^\top H_{t-1}H_{t-1}^\top \left( \spadesuit_1 +  \spadesuit_3   \right) \right]   \leq  (1+4\beta_t \lambda_1 + 11 \beta_t^2 (\lambda_1^4 \vee 1) + 8\beta_t^2 B_{\epsilon_t}^2) \tr \mathbb{E}\left[ H_{t-1}H_{t-1}^\top H_{t-1}H_{t-1}^\top\right]
\end{equation*}

It only remains to bound the term $\underbrace{\tr \mathbb{E}\left[ H_{t-1}H_{t-1}^\top H_{t-1}H_{t-1}^\top(\spadesuit_2^{(1)} + \spadesuit_2^{(2)})\right] }_{\Gamma}$. Notice that $\spadesuit_2^{(1)} + \spadesuit_2^{(2)}$ is a symmetric matrix. Therefore,{whenever $t \leq t_0$:}
\begin{equation*}
\| \spadesuit_2^{(1)} + \spadesuit_2^{(2)}       \| \leq 2\beta_t B_{\epsilon_t} \|I+\beta_t B^{-1}A \|^3 \leq 2\beta_t (\frac{101}{100})^3 B_{\epsilon_t}
\end{equation*}
And also whenever {$t \leq t_0$ }:
\begin{align*}
\tr \mathbb{E}\left[ H_{t-1}H_{t-1}^\top H_{t-1}H_{t-1}^\top(\spadesuit_2^{(1)} + \spadesuit_2^{(2)})\right] &\leq 2\beta_t B_{\epsilon_t} \|I+\beta_t B^{-1}A \|^3 \\
&\leq 2(1.01)^3\beta_t \tr \mathbb{E}\left[ H_{t-1}H_{t-1}^\top H_{t-1}H_{t-1}^\top\right]
\end{align*}

We will use similar arguments to what we used in previous sections to bound these types of terms for the case when {$t > t_0$:}

Let $H_{t-1} = (I+L_{t-1}^{t-r_t+1})H_{t-r_t}$ as in Lemma \ref{lemma1}, therefore:
\begin{align*}
\tr \mathbb{E}\left[ H_{t-1}H_{t-1}^\top H_{t-1}H_{t-1}^\top(\spadesuit_2^{(1)} + \spadesuit_2^{(2)})\right]  &= \tr \mathbb{E}\big[  (I+L_{t-1}^{t-r_t+1}) H_{t-r_t}H_{t-r_t}^\top (I+L_{t-1}^{t-r_t+1})^\top \\
&\quad(I + L_{t-1}^{t-r_t+1}) H_{t-r_t}H_{t-r_t}^\top (I+L_{t-1}^{t-r_t+1})^\top  (\spadesuit_2^{(1)} + \spadesuit_2^{(2)}) \big]
\end{align*}
We can now expand the right hand side of the last equation into different types of terms. We start by bounding the term that does not contain any $L_{t-1}^{t-r_t+1}$ nor $(L_{t-1}^{t-r_t+1})^\top$. It is easy to see that $\| \mathbb{E}\left[     \spadesuit_2^{(1)} + \spadesuit_2^{(2)}| \mathcal{F}_{t-r_t}\right]\| \leq \beta_t^2 \mathcal{A}_t r_t *(\frac{101}{100})^3*4 $. This follows because $\| \mathbb{E}\left[ \epsilon_t  |\mathcal{F}_{t-r_t} \right]\| \leq  \mathcal{A}_t \beta_t r_t $, and an operator bound on each of the remaining $3$ terms in each of the four factors by $\| I + \beta_t B^{-1}A\|\leq \frac{101}{100}$. With these observations and using the fact that $\spadesuit_2^{(1)} + \spadesuit_2^{(2)}$ is a symmetric matrix, we can bound the following term:
\begin{align*}
&\mathbb{E}\left[   \tr\left(  H_{t-r_t}H_{t-r_t}^\top H_{t-r_t}H_{t-r_t}^\top (\spadesuit_2^{(1)} + \spadesuit_2^{(2)})  \right) | \mathcal{F}_{t-r_t}\right]  \\
&\qquad=  \tr\left(    H_{t-r_t}H_{t-r_t}^\top H_{t-r_t}H_{t-r_t}^\top  \mathbb{E}\left[\spadesuit_2^{(1)} + \spadesuit_2^{(2)}   | \mathcal{F}_{t-r_t}\right] \right)\\
&\qquad\leq \beta_t^2 \mathcal{A}_t r_t *\left( \frac{101}{100}  \right)^3*4 \tr(  H_{t-r_t}H_{t-r_t}^\top H_{t-r_t}H_{t-r_t}^\top  )\\
&\qquad\leq  \beta_t^2 \mathcal{A}_t r_t *5 \tr(  H_{t-r_t}H_{t-r_t}^\top H_{t-r_t}H_{t-r_t}^\top  )
\end{align*}

For the terms of $\Gamma$ containing $L_{t-1}^{t-r_t+1}$ components we use a simple bound. Notice that $\| \spadesuit_2^{(1)} + \spadesuit_2^{(2)} \| \leq \beta_t (\frac{101}{100})^3*4*B_{\epsilon_t}$. And recall just as in Equation \ref{equation::L_bound},  $\| L_{t-1}^{t-r_t+1} \| \leq  4r_t \mathcal{G}_t\beta_t $ and therefore $\|  (L_{t-1}^{t-r_t+1})^\top L_{t-1}^{t-r_t+1} \| \leq 16r_t^2\mathcal{G}_t^2\beta_t^2$. We look at the term containing four copies of $L_{t-1}^{t-r_t+1}$ terms:

Let $O_1 = \tr\left(   L_{t-1}^{t-r_t+1}   H_{t-r_t}H_{t-r_t}^\top (L_{t-1}^{t-r_t+1})^\top L_{t-1}^{t-r_t+1} H_{t-r_t}H_{t-r_t}^\top (L_{t-1}^{t-r_t+1})^\top  (  \spadesuit_2^{(1)} + \spadesuit_2^{(1)} ) \right) $.
\begin{align*}
O_1 &\leq \beta_t (\frac{101}{100} )^3 *4*B_{\epsilon_t} \tr\left(   L_{t-1}^{t-r_t+1}   H_{t-r_t}H_{t-r_t}^\top (L_{t-1}^{t-r_t+1})^\top L_{t-1}^{t-r_t+1} H_{t-r_t}H_{t-r_t}^\top (L_{t-1}^{t-r_t+1})^\top  \right)\\
&= \beta_t (\frac{101}{100} )^3 *4*B_{\epsilon_t}  \tr\left(  [   H_{t-r_t}H_{t-r_t}^\top (L_{t-1}^{t-r_t+1})^\top L_{t-1}^{t-r_t+1} H_{t-r_t}H_{t-r_t}^\top ] (L_{t-1}^{t-r_t+1})^\top  L_{t-1}^{t-r_t+1} \right)\\
&\leq \beta_t^3 (\frac{101}{100})^3 *4 B_{\epsilon_t}*16r_t^2 \mathcal{G}_t^2 \tr\left(     H_{t-r_t}H_{t-r_t}^\top (L_{t-1}^{t-r_t+1})^\top L_{t-1}^{t-r_t+1} H_{t-r_t}H_{t-r_t}^\top \right)\\
&= \beta_t^3 (\frac{101}{100})^3 *4 B_{\epsilon_t}*16r_t^2 \mathcal{G}_t^2 \tr\left( H_{t-r_t}H_{t-r_t}^\top    H_{t-r_t}H_{t-r_t}^\top (L_{t-1}^{t-r_t+1})^\top L_{t-1}^{t-r_t+1}  \right) \\
&\leq \beta_t^5 (\frac{101}{100})^3 *4 B_{\epsilon_t}*16^2 r_t^4 \mathcal{G}_t^4 \tr\left( H_{t-r_t}H_{t-r_t}^\top    H_{t-r_t}H_{t-r_t}^\top  \right) \\
&\leq \beta_t^2 17 B_{\epsilon_t}*r_t \mathcal{G}_t \tr\left( H_{t-r_t}H_{t-r_t}^\top    H_{t-r_t}H_{t-r_t}^\top  \right)
\end{align*}
where the last inequality follows from the step size conditions. We now look at the following term in $\Gamma$ that has three $L_{t-1}^{t-r_t+1}$ terms:
\begin{align*}
O_2  &=  \tr( H_{t-r_t} H_{t-r_t}^\top (L_{t-1}^{t-r_t+1})^\top L_{t-1}^{t-r_t+1} H_{t-r_t}H_{t-r_t}^\top (L_{t-1}^{t-r_t+1} )^\top  ( \spadesuit_2^{(1)} + \spadesuit_2^{(2)}) )   + \\
&\quad \tr(  L_{t-1}^{t-r_t+1} H_{t-r_t} H_{t-r_t}^\top (L_{t-1}^{t-r_t+1})^\top L_{t-1}^{t-r_t+1} H_{t-r_t}H_{t-r_t}^\top ( \spadesuit_2^{(1)} + \spadesuit_2^{(2)}) )
\end{align*}

Since $\| (L_{t-1}^{t-r_t+1} )^\top  ( \spadesuit_2^{(1)} + \spadesuit_2^{(2)})    +  ( \spadesuit_2^{(1)} + \spadesuit_2^{(2)}) L_{t-1}^{t-r_t+1}  \| \leq \beta_t^2  (\frac{101}{100}  )^3B_{\epsilon_t}r_t\mathcal{G}_t$. Using a similar series of inequalities as in the case above we obtain a bound:
\begin{align*}
O_2 & = \tr\big( H_{t-r_t} H_{t-r_t}^\top (L_{t-1}^{t-r_t+1})^\top L_{t-1}^{t-r_t+1} H_{t-r_t}H_{t-r_t}^\top (L_{t-1}^{t-r_t+1} )^\top  ( \spadesuit_2^{(1)} + \spadesuit_2^{(2)})    + \\
&\quad   L_{t-1}^{t-r_t+1} H_{t-r_t} H_{t-r_t}^\top (L_{t-1}^{t-r_t+1})^\top L_{t-1}^{t-r_t+1} H_{t-r_t}H_{t-r_t}^\top ( \spadesuit_2^{(1)} + \spadesuit_2^{(2)}) \big) \\
& = \tr\big( H_{t-r_t} H_{t-r_t}^\top (L_{t-1}^{t-r_t+1})^\top L_{t-1}^{t-r_t+1} H_{t-r_t}H_{t-r_t}^\top [  (L_{t-1}^{t-r_t+1} )^\top  ( \spadesuit_2^{(1)} + \spadesuit_2^{(2)})    +  ( \spadesuit_2^{(1)} + \spadesuit_2^{(2)}) L_{t-1}^{t-r_t+1}]\big)\\
&\leq \beta_t^2 *32*(\frac{101}{100} )^3 B_{\epsilon_t}\mathcal{G}_t r_t\tr\big( H_{t-r_t} H_{t-r_t}^\top (L_{t-1}^{t-r_t+1})^\top L_{t-1}^{t-r_t+1} H_{t-r_t}H_{t-r_t}^\top  \big) \\
&= \beta_t^2 *32*(\frac{101}{100} )^3 B_{\epsilon_t}\mathcal{G}_t r_t\tr\big(   H_{t-r_t}H_{t-r_t}^\top H_{t-r_t} H_{t-r_t}^\top (L_{t-1}^{t-r_t+1})^\top L_{t-1}^{t-r_t+1}\big)\\
&\leq \beta_t^4 16*32( \frac{101}{100})^3 B_{\epsilon_t}\mathcal{G}_t^3 r_t^3 \tr\big(   H_{t-r_t}H_{t-r_t}^\top H_{t-r_t} H_{t-r_t}^\top\big)\\
&\stackrel{\gamma_1}{\leq} \beta_t^2 *32(\frac{101}{100})^3 B_{\epsilon_t}\mathcal{G}_tr_t \tr\big(   H_{t-r_t}H_{t-r_t}^\top H_{t-r_t} H_{t-r_t}^\top\big)\\
&\leq \beta_t^2 *33 B_{\epsilon_t}\mathcal{G}_tr_t \tr\big(   H_{t-r_t}H_{t-r_t}^\top H_{t-r_t} H_{t-r_t}^\top\big)
\end{align*}
The last inequality $\gamma_1$ follows from the step size conditions.
\begin{align*}
\hspace{-5ex}O_2'  &=  \tr( L_{t-1}^{t-r_t+1} H_{t-r_t} H_{t-r_t}^\top (L_{t-1}^{t-r_t+1})^\top  H_{t-r_t}H_{t-r_t}^\top (L_{t-1}^{t-r_t+1} )^\top  ( \spadesuit_2^{(1)} + \spadesuit_2^{(2)}) )   + \\
&\quad \tr(  L_{t-1}^{t-r_t+1} H_{t-r_t} H_{t-r_t}^\top  L_{t-1}^{t-r_t+1} H_{t-r_t}H_{t-r_t}^\top (L_{t-1}^{t-r_t+1})^\top( \spadesuit_2^{(1)} + \spadesuit_2^{(2)}) )
\end{align*}
Since $\| (L_{t-1}^{t-r_t+1} )^\top (\spadesuit_2^{(1)} + \spadesuit_2^{(2)})L_{t-1}^{t-r_t+1}\| \leq \beta_t^3 (\frac{101}{100})^3*4*B_{\epsilon_t} r_t^2\mathcal{G}_t^2$:
\begin{align*}
O_2'  &=  \tr\big( L_{t-1}^{t-r_t+1} H_{t-r_t} H_{t-r_t}^\top (L_{t-1}^{t-r_t+1})^\top  H_{t-r_t}H_{t-r_t}^\top (L_{t-1}^{t-r_t+1} )^\top  ( \spadesuit_2^{(1)} + \spadesuit_2^{(2)})    + \\
&\quad   L_{t-1}^{t-r_t+1} H_{t-r_t} H_{t-r_t}^\top  L_{t-1}^{t-r_t+1} H_{t-r_t}H_{t-r_t}^\top (L_{t-1}^{t-r_t+1})^\top( \spadesuit_2^{(1)} + \spadesuit_2^{(2)}) \big) \\
&= \tr(  H_{t-r_t} H_{t-r_t}^\top  ( L_{t-1}^{t-r_t+1} + (L_{t-1}^{t-r_t+1})^\top) H_{t-r_t}H_{t-r_t}^\top [ (L_{t-1}^{t-r_t+1} )^\top ( \spadesuit_2^{(1)} + \spadesuit_2^{(2)}    )  L_{t-1}^{t-r_t+1} ]  ) \\
&\stackrel{\gamma_1}{\leq}  \tr(  H_{t-r_t} H_{t-r_t}^\top  ( L_{t-1}^{t-r_t+1} + (L_{t-1}^{t-r_t+1})^\top) H_{t-r_t}H_{t-r_t}^\top [ (L_{t-1}^{t-r_t+1} )^\top ( \spadesuit_2^{(1)} + \spadesuit_2^{(2)}    )  L_{t-1}^{t-r_t+1} ]  ) + \\
&\quad  \tr(  H_{t-r_t} H_{t-r_t}^\top  \|( L_{t-1}^{t-r_t+1} + (L_{t-1}^{t-r_t+1})^\top) \| I  H_{t-r_t}H_{t-r_t}^\top  [ (L_{t-1}^{t-r_t+1} )^\top ( \spadesuit_2^{(1)} + \spadesuit_2^{(2)}    )  L_{t-1}^{t-r_t+1} ]     ) + \\
&\quad \tr(  H_{t-r_t} H_{t-r_t}^\top  \|( L_{t-1}^{t-r_t+1} + (L_{t-1}^{t-r_t+1})^\top) \| I  H_{t-r_t}H_{t-r_t}^\top \| [ (L_{t-1}^{t-r_t+1} )^\top ( \spadesuit_2^{(1)} + \spadesuit_2^{(2)}    )  L_{t-1}^{t-r_t+1} ] \| I   ) \\
&\stackrel{\gamma_2}{\leq} \beta_t^3 (\frac{101}{100})^3*4*B_{\epsilon_t} r_t^2\mathcal{G}_t^2 \tr(  H_{t-r_t} H_{t-r_t}^\top ( \| L_{t-1}^{t-r_t+1} + (L_{t-1}^{t-r_t+1})^\top \| I + L_{t-1}^{t-r_t+1} + \\
&\quad (L_{t-1}^{t-r_t+1})^\top)  H_{t-r_t}H_{t-r_t}^\top    )    + \beta_t^3 (\frac{101}{100})^3*4*B_{\epsilon_t} r_t^2\mathcal{G}_t^2 *2*4r_t\mathcal{G}_t\beta_t \tr(  H_{t-r_t} H_{t-r_t}^\top     H_{t-r_t}H_{t-r_t}^\top   ) \\
&\leq    ( \beta_t^3 (\frac{101}{100})^3*4*B_{\epsilon_t} r_t^2\mathcal{G}_t^2*4*4r_t\mathcal{G}_t\beta_t    +  \beta_t^3 (\frac{101}{100})^3*4*B_{\epsilon_t} r_t^2\mathcal{G}_t^2 *2*4r_t\mathcal{G}_t\beta_t   )  \\
&\quad *\tr(  H_{t-r_t} H_{t-r_t}^\top  H_{t-r_t}H_{t-r_t}^\top ) \\
&= \beta_t^4 *72*(\frac{101}{100})^3 r_t^3\mathcal{G}_t^3B_{\epsilon_t}   \tr(  H_{t-r_t} H_{t-r_t}^\top  H_{t-r_t}H_{t-r_t}^\top )\\
&\leq  \beta_t^2 *5B_{\epsilon_t}\mathcal{G}_t r_t  \tr(  H_{t-r_t} H_{t-r_t}^\top  H_{t-r_t}H_{t-r_t}^\top )
\end{align*}
where the inequality $\gamma_1$ follows because the sum of the two added terms is nonnegative. The inequality $\gamma_2$ follows by combining the first two terms in the previous expression and noting that$$H_{t-r_t}H_{t-r_t}^\top H_{t-r_t}H_{t-r_t}^\top \| L_{t-1}^{t-r_t+1} + (L_{t-1}^{t-r_t+1})^\top \| + H_{t-r_t}H_{t-r_t}^\top ( L_{t-1}^{t-r_t+1} + (L_{t-1}^{t-r_t+1})^\top   ) H_{t-r_t}H_{t-r_t}^\top \succeq 0.$$ The last inequality follows from the step size conditions. This finalizes the analysis for the components in $\Gamma$ having three $L_{t-1}^{t-r_t+1}$ terms.

We now look at the components of $\Gamma$ with two $L_{t-1}^{t-r_t+1}$ terms. Their sum equals:
\begin{small}
\begin{align*}
&\tr\big( \big(      (L_{t-1}^{t-r_t+1} H_{t-r_t}H_{t-r_t}^\top + H_{t-r_t}H_{t-r_t}^\top L_{t-1}^{t-r-t+1}   )^2 +  (L_{t-1}^{t-r_t+1} H_{t-r_t}H_{t-r_t}^\top (L_{t-1}^{t-r_t+1})^\top + H_{t-r_t}H_{t-r_t}^\top )^2 \\
&- H_{t-r_t}H_{t-r_t}^\top H_{t-r_t}H_{t-r_t}^\top  - L_{t-1}^{t-r_t+1} H_{t-r_t} H_{t-r_t}^\top (L_{t-1}^{t-r_t+1})^\top L_{t-1}^{t-r_t+1}H_{t-r_t}H_{t-r_t}^\top (L_{t-1}^{t-r_t+1})^\top \big) ( \spadesuit_2^{(1)} + \spadesuit_2^{(2)} )\big)
\end{align*}
\end{small}
We look at a generic term of $\Gamma$ having exactly two $L_{t-1}^{t-r_t+1}$ terms: Let $$O_3 = \tr( H_{t-r_t} H_{t-r_t}^\top (L_{t-1}^{t-r_t+1})^\top L_{t-1}^{t-r_t+1} H_{t-r_t}H_{t-r_t}^\top  ( \spadesuit_2^{(1)} + \spadesuit_2^{(2)}) ).$$
Then, we have that,
\begin{align*}
O_3 &\leq d \| L_{t-1}^{t-r_t+1}\|^2 \| H_{t-r_t}H_{t-r_t}^\top\|^2 \| \spadesuit_2^{(1)} + \spadesuit_2^{(2)}) \| \\
&\leq d \beta_t^2 *4^2*r_t^2*\mathcal{G}_t^2*(\frac{101}{100} )^3*4*B_{\epsilon_t} \| (H_{t-r_t}H_{t-r_t}^\top)^2 \| \\
&\leq d \beta_t^2 17 \mathcal{G}_t B_{\epsilon_t}\tr(  H_{t-r_t}H_{t-r_t}^\top H_{t-r_t}H_{t-r_t}^\top   )
\end{align*}
The last inequality follows from a the step size conditions plus the fact that trace is larger than operator norm for a PSD matrix.

We now look at a generic term in $\Gamma$ with one $L_{t-1}^{t-r_t+1}$ term: Let $$O_4 = \tr( H_{t-r_t} H_{t-r_t}^\top (L_{t-1}^{t-r_t+1})^\top  H_{t-r_t}H_{t-r_t}^\top  ( \spadesuit_2^{(1)} + \spadesuit_2^{(2)}) ).$$
Then, we have that,
\begin{align*}
O_4 &\leq d \| L_{t-1}^{t-r_t+1}\| \| H_{t-r_t}H_{t-r_t}^\top\|^2 \| \spadesuit_2^{(1)} + \spadesuit_2^{(2)}) \|\\
&\leq d 4  r_t \mathcal{G}_t\beta_t  (\frac{101}{100})^3 * 4 B_{\epsilon_t}  \| (H_{t-r_t}H_{t-r_t}^\top)^2 \|  \\
&\leq 17 d \beta_t^2 r_t \mathcal{G}_t B_{\epsilon_t} \tr(  H_{t-r_t}H_{t-r_t}^\top H_{t-r_t}H_{t-r_t}^\top  )
\end{align*}
Since there is a single term of type $O_1$, four of type $O_2$, six of type $O_3$ and four of type $O_4$, we obtain the bound {whenever $t > t_0$}:
\begin{align*}
\tr \mathbb{E}\left[ H_{t-1}H_{t-1}^\top H_{t-1}H_{t-1}^\top(\spadesuit_2^{(1)} + \spadesuit_2^{(2)})\right] &\leq \beta_t^2( 5r_t\mathcal{A}_t + 55B_{\epsilon_t}\mathcal{G}_tr_t + 23d\mathcal{G}_tB_{\epsilon_t}r_t )\cdot\\
&\qquad\mathbb{E}\left[  \tr( H_{t-r_t}H_{t-r_t}^\top H_{t-r_t}H_{t-r_t}^\top) \right]
\end{align*}

Therefore we obtain the following recursion:
\begin{align*}
\tr( \mathbb{E}\left[ H_tH_t^\top H_tH_t^\top   \right]    ) &\leq \tr \mathbb{E}\left[ H_{t-1}H_{t-1}^\top H_{t-1}H_{t-1}^\top  \spadesuit \right]   \\
   &\leq  (1+4\beta_t \lambda_1 + 11 \beta_t^2 \max(\lambda_1^4, 1) + 8\beta_t^2 B_{\epsilon_t}^2) \tr \mathbb{E}\left[ H_{t-1}H_{t-1}^\top H_{t-1}H_{t-1}^\top\right] + \\
  &\quad \beta_t^2( 5r_t\mathcal{A}_t + 55B_{\epsilon_t}\mathcal{G}_tr_t + 23d\mathcal{G}_tB_{\epsilon_t}r_t   +    1(t\leq t_0) 2\beta_t (\frac{101}{100})^3  B_{\epsilon_t})\cdot\\
  &\quad\mathbb{E}\left[  \tr( H_{t-r_t}H_{t-r_t}^\top H_{t-r_t}H_{t-r_t}^\top) \right]
\end{align*}

Let $C^{(3)} $ be a constant such that:
\begin{equation*}
d r_t ( \mathcal{A}_t + B_{\epsilon_t}^2 + B_{\epsilon_t}\mathcal{G}_t ) C^{(3)} \geq    ( 5r_t\mathcal{A}_t + 55B_{\epsilon_t}\mathcal{G}_tr_t + 23d\mathcal{G}_tB_{\epsilon_t}r_t )  + 11\max(\lambda_1^4, 1) + 8B_{\epsilon_t}^2
\end{equation*}

Let $\{ \eta_i\} $ be a sequence of increasing upper bounds for $\mathbb{E}\left[ \tr( H_i H_i^\top H_i H_i^\top ) \right]$. In other words,
\begin{equation*}
\mathbb{E}\left[ \tr( H_i H_i^\top H_i H_i^\top ) \right] \leq \eta_i \text{   } \forall i
\end{equation*}
And $\eta_0 \leq \eta_1 \leq \eta_2 \leq \cdots $, where $\eta_0  = d$. Let $\mathcal{C}_t^{(3)}  = E_\epsilon + \mathcal{D}_t^{(3)} + 11\max(\lambda_1^4, 1)$. We can obtain a recursion of the form:

\begin{equation*}
\eta_t \leq (1+4\beta_t \lambda_1 + \beta_t^2  dr_t(\mathcal{A}_t + B_{\epsilon_t}^2 + B_{\epsilon_t}\mathcal{G}_t) C^{(3)} + 1(t\leq t_0)2\beta_t (\frac{101}{100})^3 B_{\epsilon_t})  \eta_{t-1}
\end{equation*}

We conclude by applying the inequality $1+x \leq \exp(x)$ for $x > 0$ and the initial condition $\eta_0 =d$:
\begin{equation*}
\eta_t \leq d\exp( \sum_{i=1}^t 4\lambda_1 \beta_i +   dr_i(\mathcal{A}_i + B_{\epsilon_i}^2 + B_{\epsilon_i}\mathcal{G}_i) C^{(3)} \beta_i^2 + \sum_{j=1}^{\min(t, t_0)}  \beta_j 2*(\frac{101}{100})^3 B_{\epsilon_j} )
\end{equation*}
\end{proof}

\begin{lemma}
\label{lemma::vari} For $t>0$, we have that
\begin{align*}
&\mathbb{E}\left[ (\tilde{u}_i^\top H_tH_t^\top \tilde{u}_i )^2  \right] \\
&\qquad\leq \|\tilde{u}_1\|_2^4\exp( \sum_{i=1}^t 4 \lambda_1 \beta_i + 11\lambda_1^2 \beta_t^2) + \|\tilde{u}_1\|_2^4\sum_{i=1}^t \Bigg( \left( \beta_i^2d r_i ( \mathcal{A}_i + B_{\epsilon_i}^2 + B_{\epsilon_i}\mathcal{G}_i) \mathcal{U}_2  +  1(i \leq t_0) \beta_i 4B_{\epsilon_i}    \right)\cdot \\
&\qquad\exp\Big( \sum_{j=1}^{i}4 \lambda_1 \beta_j + dr_j(\mathcal{A}_j + B_{\epsilon_j}^2 + \mathcal{G}_j B_{\epsilon_j}  ) \mathcal{C}^{(3)} \beta_j^2   + \sum_{j=i}^{\min(t, t_0)} \beta_j 2*(\frac{101}{100})^3 B_{\epsilon_j} \Big)  \Bigg)
\end{align*}
where $\tilde{u}_1$ is the unnormalized left eigenvector corresponding to the maximum eigenvalue $\lambda_1$ of $B^{-1}A$. As long as $\beta_t$ follows that $\| I + \beta_t B^{-1}A \|\leq \frac{101}{100}$, $\beta_t B_{\epsilon_t} < 1$

\end{lemma}

\begin{proof}
As in the previous lemma, we let $v = \nicefrac{\tilde{u}_1}{\|\tilde{u}_1 \|_2}$ denote the normalized left principal eigenvector. Let $H_t = (I + \beta_t G_t)H_{t-1} = (I + \beta_t B^{-1}A + \beta_t \epsilon_t)H_{t-1}$. The desired expectation can be written as:
\begin{align*}
&\mathbb{E}\left[ (v^\top H_tH_t^\top v )^2  \right] =  \mathbb{E}\left[   v^\top (I + \beta_t G_t) H_{t-1} H_{t-1}^\top (I + \beta_t G_t)^\top vv^\top (I+\beta_t G_t) H_{t-1}H_{t-1}^\top (I+\beta_tG_t)^\top v\right]\\
&= \mathbb{E}\left[  \underbrace{  v^\top (I + \beta_t B^{-1}A) H_{t-1} H_{t-1}^\top (I + \beta_t B^{-1}A)^\top vv^\top (I+\beta_t B^{-1}A) H_{t-1}H_{t-1}^\top (I+\beta_tB^{-1}A)^\top v}_{\Gamma_0} \right] \\
&\quad +\mathbb{E}\left[ \Gamma_1+ \Gamma_2 + \Gamma_3 + \Gamma_4 \right]
\end{align*}
where $\Gamma_i$ is the collection of terms in the expansion of $\mathbb{E}\left[(v^\top H_t H_t^\top  v)^2   \right]$ that have exactly $i$ terms of the form $\epsilon_t$.

Since $v$ is a left eigenvector of $B^{-1}A$, the term $\Gamma_0$ can be written as follows:
\begin{align*}
\mathbb{E}\left[ \Gamma_0  \right]&= (1+\beta_t \lambda_1)^4\mathbb{E}\left[  v^\top H_{t-1}H_{t-1}^\top vv^\top H_{t-1}H_{t-1}^\top v \right] \\
&\leq \exp( 4\lambda_1 \beta_t  + 11\lambda_1^2 \beta_t^2) \mathbb{E}\left[  v^\top H_{t-1}H_{t-1}^\top vv^\top H_{t-1}H_{t-1}^\top v \right]
\end{align*}

Now we bound the terms $\Gamma_i$ with $i \geq 2$. Each of these terms is formed of component terms with at least two $\beta_t\epsilon_t$ each. Let's look at a generic term like this one and bound it, for example one that has two terms of the form $\epsilon_t$:
\begin{align*}
| v^\top \beta_t \epsilon_t H_{t-1}H_{t-1}^\top \beta_t \epsilon_t vv^\top (I+ \beta_t B^{-1}A) H_{t-1}H_{t-1}^\top  (I + \beta_t B^{-1}A)^\top v| &\leq \beta_t^2 \|H_{t-1}H_{t-1}^\top \|^2B_{\epsilon_t}^2 \left(\frac{101}{100}\right)^2 \\
&\hspace{-5ex}\leq 2B_{\epsilon_t}^2\beta_t^2 \tr(H_{t-1}H_{t-1}^\top H_{t-1}H_{t-1}^\top)
\end{align*}

By a similar argument, and using the step size conditions $\beta_t \mathcal{B}_{\epsilon_t}  < 1$, we can bound each of the terms in $\Gamma_2, \Gamma_3$ and $\Gamma_4$ and obtain (using the fact that $\beta_t < 1$):
\begin{equation}
\Gamma_2 + \Gamma_3 + \Gamma_4 \leq \beta_t^2 B_{\epsilon_t}^2 \mathcal{U}_1 \tr( H_{t-1}H_{t-1}^\top H_{t-1}H_{t-1}^\top )
\end{equation}
For some universal constant $\mathcal{U}_1$ depending on $\frac{101}{100}$ and the number of component terms in $\Gamma_2, \Gamma_3$, and $\Gamma_4$. Therefore,
\begin{align*}
\mathbb{E}\left[  \Gamma_2 + \Gamma_3 + \Gamma_4  \right] &\leq \beta_t^2 B_{\epsilon_t}^2\mathcal{U}_1 \mathbb{E}\left[ \tr( H_{t-1}H_{t-1}^\top H_{t-1}H_{t-1}^\top    )   \right] \\
&\leq \beta_t^2 B_{\epsilon-t}^2\mathcal{U}_1 d\exp\left( \sum_{i=1}^{t-1}4 \lambda_1 \beta_i + dr_i (\mathcal{A}_i + B_{\epsilon_i}^2 + B_{\epsilon_i}\mathcal{G}_i  )\mathcal{C}^{(3)} \beta_i^2\right. \\
&\qquad\left.+ \sum_{j=1}^{\min(t, t_0)} 2*\beta_j (\frac{101}{100})^3B_{\epsilon_j} \right)
\end{align*}

\textbf{Bounding expectation of $\Gamma_1$:} We start by bounding the expectation of $\Gamma_1$ {whenever $t \leq t_0$.} Let's look at a generic term from $\Gamma_1$:
\begin{equation*}
\mathcal{Z} := v^\top (I+\beta_t B^{-1}A) H_{t-1}H_{t-1}^\top \beta_t \epsilon_t vv^\top (I+ \beta_t B^{-1}A) H_{t-1}H_{t-1}^\top  (I + \beta_t B^{-1}A)^\top v
\end{equation*}

We bound this term naively:
\begin{align*}
\|\mathcal{Z} \| &\leq \beta_t\| I + \beta_t B^{-1}A \|^3 \| H_{t-1}H_{t-1}^\top \|^2 B_{\epsilon_t} \\
&\leq    \beta_t \left(  \frac{101}{100} \right)^3 \tr( H_{t-1}H_{t-1}^\top H_{t-1}H_{t-1}^\top) B_{\epsilon_t}
\end{align*}
There are exactly $4$ terms of type $\mathcal{Z}$. Now we proceed to bound the expectation of $\Gamma_1$  whenever $t > t_0$: Let's look at a generic term from $\Gamma_1$:
\begin{equation}\label{eq::lemma4_2_eq1}
v^\top (I+\beta_t B^{-1}A) H_{t-1}H_{t-1}^\top \beta_t \epsilon_t vv^\top (I+ \beta_t B^{-1}A) H_{t-1}H_{t-1}^\top  (I + \beta_t B^{-1}A)^\top v
\end{equation}

In the same way as in previous lemmas, in order to obtain a bound for this term, we write $H_{t-1} = (I + L_{t-1}^{t-r_t+1}) H_{t-r_t}$ and substitute this equality in Equation \ref{eq::lemma4_2_eq1}. Recall that $\|L_{t-1}^{t-r_t+1}\| \leq 4r_t \mathcal{G}_t\beta_t$. In this expansion, we bound all terms that have at least one $L_{t-1}^{t-r_t+1}$ using a simple bound. Let's look at a generic such term and bound it:
\begin{equation}
\spadesuit := | v^\top (I+\beta_t B^{-1}A) L_{t-1}^{t-r_t+1} H_{t-r_t}H_{t-r_t}^\top \beta_t \epsilon_t vv^\top (I+ \beta_t B^{-1}A) H_{t-r_t}H_{t-r_t}^\top  (I + \beta_t B^{-1}A)^\top v|
\end{equation}
\begin{align*}
\spadesuit &\leq 4r_t \mathcal{G}_t \beta_t^2 \left( \frac{101}{100}\right)^3 \cdot \| H_{t-r_t}H_{t-r_t}^\top \|^2B_\epsilon \\
&\leq 2r_t     \mathcal{G}_t\beta_t^2 B_{\epsilon_t} \left( \frac{101}{100}\right)^3 \cdot\tr( H_{t-r_t}H_{t-r_t}^\top H_{t-r_t}H_{t-r_t}^\top  )
\end{align*}
And therefore:
\begin{align*}
\mathbb{E}\left[ \spadesuit \right] &\leq  2r_t     \mathcal{G}_t\beta_t^2 B_{\epsilon_t} \left( \frac{101}{100}\right)^3   d\exp\left( \sum_{i=1}^{t-r_t}4 \lambda_1 \beta_i + dr_i (\mathcal{A}_i + B_{\epsilon_i}^2 + B_{\epsilon_i}\mathcal{G}_i  )\mathcal{C}^{(3)} \beta_i^2 \right)\\
&\leq  2r_t     \mathcal{G}_t\beta_t^2 B_{\epsilon_t} \left( \frac{101}{100}\right)^3   d\exp\left( \sum_{i=1}^{t-1}4 \lambda_1 \beta_i + dr_i (\mathcal{A}_i + B_{\epsilon_i}^2 + B_{\epsilon_i}\mathcal{G}_i  )\mathcal{C}^{(3)} \beta_i^2 \right)
\end{align*}

Using the step size condition, $\beta_t \mathcal{G}_t r_t \leq \frac{1}{4}$, all of the remaining terms with at least one $L_{t-1}^{t-r_t+1}$ can be upper bounded by a expression of order $O(\beta_t^2 r_t \mathcal{G}_t B_{\epsilon_t}\tr(H_{t-r_t}H_{t-r_t}^\top H_{t-r_t}H_{t-r_t}^\top    ) ) $. This procedure will handle the terms in $\Gamma_1$ that after the subsitution $H_{t-1} = (I+ L_{t-1}^{t-r_t+1})H_{t-r_t}$ have at least one $L_{t-1}^{t-r_t+1}$.

The only terms remaining to bound are those coming from $\Gamma_1$, such that after substituting  $H_{t-1} = (I+ L_{t-1}^{t-r_t+1})H_{t-r_t}$ do not involve any $L_{t-1}^{t-r_t+1}$. Let's look at a generic such term and bound its expectation:
\begin{equation*}
 \diamondsuit := \mathbb{E}\left[ \underbrace{ v^\top (I+\beta_t B^{-1}A) H_{t-r_t}H_{t-r_t}^\top \beta_t \epsilon_t vv^\top (I+ \beta_t B^{-1}A) H_{t-r_t}H_{t-r_t}^\top  (I + \beta_t B^{-1}A)^\top v }_{\diamondsuit_1}\right]
\end{equation*}

Recall that $\| \mathbb{E}[ \epsilon_t | \mathcal{F}_{t-r_t}]\| \leq  \mathcal{A}_t \beta_t r_t  $. We bound $\diamondsuit$ by first bounding the norm of the conditional expectation of $\diamondsuit_1$:
\begin{align*}
\| \mathbb{E}\left[ \diamondsuit_1  | \mathcal{F}_{t-r_t} \right]  \| &\leq \beta_t^2 O(r_t) \left(\frac{101}{100}   \right)^3 \|H_{t-r_t}H_{t-r_t}^\top \|^2 \\
&\leq \beta_t^2 \mathcal{A}_t r_t \left(\frac{101}{100}   \right)^3 \tr(  H_{t-r_t}H_{t-r_t}^\top H_{t-r_t}H_{t-r_t}^\top    )
\end{align*}
And therefore:
\begin{small}
\begin{align*}
&\diamondsuit= \mathbb{E}\left[  \diamondsuit_1 \right] \leq \mathbb{E}\left[   \| \mathbb{E}\left[ \diamondsuit_1  | \mathcal{F}_{t-r_t} \right]  \|  \right] \\
&\qquad\leq \beta_t^2 \mathcal{A}_t r_t \left(\frac{101}{100}   \right)^3 d\exp\left( \sum_{i=1}^{t-r_t}4 \lambda_1 \beta_i + dr_i (\mathcal{A}_i + B_{\epsilon_i}^2 + B_{\epsilon_i}\mathcal{G}_i  )\mathcal{C}^{(3)} \beta_i^2  + \sum_{j=1}^{\min(t-r_t,t_0)}  \beta_j 2*(\frac{101}{100})^3 B_{\epsilon_j} \right) \\
&\qquad\leq \beta_t^2 \mathcal{A}_t r_t \left(\frac{101}{100}   \right)^3 d\exp\left( \sum_{i=1}^{t-1}4 \lambda_1 \beta_i + dr_i (\mathcal{A}_i + B_{\epsilon_i}^2 + B_{\epsilon_i}\mathcal{G}_i  )\mathcal{C}^{(3)} \beta_i^2  + \sum_{j=1}^{\min(t,t_0)}  \beta_j 2*(\frac{101}{100})^3 B_{\epsilon_j} \right)
\end{align*}
\end{small}
The last inequality follows from the results of \ref{lemma4_part1}. Combining all these bounds yields {for all $t$} we have:
\begin{align*}
\mathbb{E}\left[ \Gamma_1 + \Gamma_2 + \Gamma_3 + \Gamma_4  \right] &\leq \left( \beta_t^2d  r_t (\mathcal{A}_t + B_{\epsilon_t}^2 + B_{\epsilon_t}\mathcal{G}_t) \mathcal{U}_2 + 1(t \leq t_0) \beta_t 4* B_{\epsilon_t} ( \frac{101}{100} )^3  \right)\\
&\quad\exp\Big( \sum_{i=1}^{t-1}4 \lambda_1 \beta_i + dr_i (\mathcal{A}_i + B_{\epsilon_i}^2 
 + B_{\epsilon_i}\mathcal{G}_i  )\mathcal{C}^{(3)} \beta_i^2  + \sum_{j=1}^{\min(t,t_0)}  \beta_j 2*(\frac{101}{100} )^3    \Big)
\end{align*}
where $\mathcal{U}_2$ is an absolute constant depending on $\frac{101}{100}$, and the number of terms in $\Gamma_1, \Gamma_2, \cdots, \Gamma_4$.

Combining all these terms we get a recursion of the form:
\begin{align*}
\mathbb{E}\left[ (v^\top H_tH_t^\top v )^2  \right] &\leq \exp(4\lambda_1 \beta_t + 11\lambda_1^2 \beta_t^2 ) \mathbb{E}\left[ (v^\top H_{t-1}H_{t-1}^\top v )^2\right] + \Big( \beta_t^2 d r_t (\mathcal{A}_t + B_{\epsilon_t}^2 + B_{\epsilon_t}\mathcal{G}_t ) \mathcal{U}_2  + \\
&\quad 1(t\leq t_0) \beta_t B_{\epsilon_t} ( \frac{101}{100})^3  \Big)\exp\left( \sum_{i=1}^{t-1}4 \lambda_1 \beta_i + dr_i ( \mathcal{A}_i + B_{\epsilon_i}^2 + B_{\epsilon_i}\mathcal{G}_i)\mathcal{C}^{(3)} \beta_i^2 \right. \\
&\qquad\left.+ \sum_{j=1}^{\min(t, t_0)} \beta_j (\frac{101}{100})^3 B_{\epsilon_j}\right)
\end{align*}

After applying recursion on this equation we obtain:
\begin{align*}
\mathbb{E}\left[ (v^\top H_tH_t^\top v )^2  \right] &\leq \exp( \sum_{i=1}^t 4 \lambda_1 \beta_i + 11\lambda_1^2 \beta_t^2) \\
&+\sum_{i=1}^t \left( \beta_i^2d r_i ( \mathcal{A}_i + B_{\epsilon_i}^2 + B_{\epsilon_i}\mathcal{G}_i) \mathcal{U}_2  +  1(i \leq t_0) \beta_i 4B_{\epsilon_i}(\frac{101}{100})^3    \right)\exp\Big( \sum_{j=1}^{i}4 \lambda_1 \beta_j +\\
&\quad  dr_j(\mathcal{A}_j +  B_{\epsilon_j}^2 +\mathcal{G}_j B_{\epsilon_j}  ) \mathcal{C}^{(3)} \beta_j^2   + \sum_{j=i}^{\min(t, t_0)} \beta_j 2(\frac{101}{100})^3 B_{\epsilon_j} \Big)
\end{align*}
As desired.
\end{proof}

\section{Convergence Analysis and Main Result}
\label{sec:conv}
We reproduce the bounds that we will be requiring in this section from the previous ones. We begin by reporducing the lower bound of Lemma 5.3.
\begin{align}
\begin{split}\label{eq:lemma3_rep}
\frac{\E [\tilde{u}_1^\top H_nH_n^\top \tilde{u}_1]}{\|\tilde{u}_1\|_2^2} &\geq  \exp{\left( \sum_{t=1}^n 2\beta_t \lambda_1 - 4\beta_t^2 \lambda_1^2\right)} - \\
 &\quad \underbrace{d\sum_{t=1}^{n} c_1\Bigg( (\beta_t^2 r_t + \beta_t\mathbb{I}(t  \leq t_0)) \exp\left( \sum_{i=1}^{t} 2\beta_i \lambda_1 + c_2\beta_i^2 dr_i +c_3\sum_{i=1}^{t_0} \beta_i d \right) \Bigg)}_{(I)},
\end{split}
\end{align}
where we have merged previous explicit constants into $c_1, c_2$ and $c_3$, which throughout the course of this section might assume different values. Restating the bound from Lemma 5.4, we have,
\begin{align}
\begin{split}\label{eq:lemma4_rep}
\frac{\mathbb{E}\left[ (\tilde{u}_1^\top H_nH_n^\top \tilde{u}_1 )^2  \right]}{\|\tilde{u}_1\|_2^4} &\leq \exp\left(\sum_{t=1}^n 4 \lambda_1 \beta_t + 11\lambda_1^2 \beta_t^2\right) +\\
&\quad \underbrace{c_1\sum_{t=1}^n \Bigg( \left( d\beta_t^2 r_t   +  \mathbb{I}(t \leq t_0) \beta_t   \right)\exp\Big( \sum_{i=1}^{t}4 \lambda_1 \beta_i + c_2dr_i \beta_i^2   + c_3\sum_{i=1}^{t_0} \beta_i \Big)  \Bigg)}_{(II)}.
\end{split}
\end{align}
Note that as mentioned before in Section \ref{app:tatanos}, the term $r_t = O(\log^3(\beta_t^{-1}))$ and $t_0 = O(\log^3(d^2\beta))$. In the following, we substitute the step size $\beta_t = \frac{b}{d^2\beta +t}$, where $b, \beta$ are constants, implying that $r_t = O(\log^3(d^2\beta+t))$.

\textbf{Bounds on partial sums of series}:
We begin by obtaining bounds on partial sums of some series which will be useful in our analysis. We first prove the following upper bound:
\begin{equation}
  \label{eq:harmonic_upper}
\sum_{i=1}^t 4 \beta_i \lambda_1  = 4 b \lambda_1  \sum_{i=1}^t \frac{1}{d^2\beta +i} = 4b\lambda_1 \sum_{i = d^2\beta+1}^{d^2\beta+t}\frac{1}{i}  \leq 4b\lambda_1  \log{\left( \frac{d^2\beta + t}{d^2\beta}\right)}.
\end{equation}
We next have the following lower bound:
\begin{equation}
  \label{eq:harmonic_lower}
  \sum_{i=1}^t 4 \beta_i \lambda_1  = 4 b \lambda_1  \sum_{i=1}^t \frac{1}{d^2\beta +i} = 4b\lambda_1 \sum_{i = d^2\beta+1}^{d^2\beta+t}\frac{1}{i} \geq 4b\lambda_1  \log{\left( \frac{d^2\beta + t + 1}{d^2\beta +1}\right)}.
\end{equation}
We can obtain the following bound on the squared terms:
\begin{align*}
  c\sum_{i=1}^{t} \beta_i^2\log^3(d^2\beta +i) &= c \sum_{i=1}^{t} \frac{\log^3(d^2\beta +i)}{(d^2\beta+i)^2} \\
  &= c\sum_{i=d^2\beta+1}^{d^2\beta + t} \frac{\log^3(i)}{i^2} \leq c\int_{d^2\beta}^{\infty} \frac{\log^3(x)}{x^2}dx \leq c \frac{\log^3(d\beta)}{d^2\beta},
\end{align*}
where $c$ is a constant which changes with inequality. Next, we proceed by bounding the excess terms in the exponent corresponding to the summation over the $t_0$ terms.
\begin{equation}\label{eq:t0_bound}
  c\sum_{i=1}^{t_0} \beta_i \leq cb \log\left(\frac{d^2\beta + t_0}{d^2\beta}\right) \leq \frac{ct_0}{d^2\beta}\leq \frac{c\log^3(d\beta)}{d^2\beta}\leq \frac{c}{d},
\end{equation}
where the last inequality follows since $\frac{log^3(x)}{(x)} \leq 2$.

\textbf{Bounds on  $\E[v^\top H_nH_n^\top v]$ and $\E[(v^\top H_nH_n^\top v)^2]$}:
We first proceed by providing upper bounds on Term $(I)$ in \eqref{eq:lemma3_rep} and Term $(II)$ in \eqref{eq:lemma4_rep}.
\begin{small}
\begin{align*}
d\sum_{t=1}^{n} c_1\Bigg( (\beta_t^2 r_t + \beta_t\mathbb{I}(t  \leq t_0)) \exp\left( \sum_{i=1}^{t} 2\beta_i \lambda_1 + c_2\beta_i^2 dr_i +c_3\sum_{i=1}^{t_0} \beta_i d \right) \Bigg) &\leq c d \sum_{t = 1}^n  \Big( (\beta_t^2 r_t +\beta_t \mathbb{I}(t \leq t_0)) \\
&\quad *\left.\left(\frac{d^2\beta +t}{d^2\beta}\right)^{2b\lambda_1} \right).
\end{align*}
\end{small}
Similarly term $(II)$ by:
\begin{align*}
&c_1\sum_{t=1}^n \Bigg( \left( d\beta_t^2 r_t   +  \mathbb{I}(t \leq t_0) \beta_t   \right)\exp\Big( \sum_{i=1}^{t}4 \lambda_1 \beta_i + c_2dr_i \beta_i^2   + c_3\sum_{i=1}^{t_0} \beta_i \Big)  \Bigg)\\
&\qquad\leq c\sum_{t = 1}^n (d\beta_t^2r_t + \beta_t\mathbb{I}(t\leq t_0))\left(\frac{d^2\beta +t}{d^2\beta}\right)^{4b\lambda_1}.
\end{align*}

\begin{lemma}\label{lem:conv_mean_perturb}
For any $\delta_1 \in (0,1)$ and $n$ satisfying,
\begin{align*}
\frac{d^2\beta+n}{\log^{\frac{4}{\min(1, 2b\lambda_1)}}(d^2\beta+n)} \geq \max&\left(\left(\frac{\exp(\frac{c\lambda_1^2}{d^2})}{\delta_1} \right)^{1/2b\lambda_1}(d^2\beta+1),\right.\\ &\frac{cd\beta^{2b\lambda_1}}{\delta_1}\exp\left(\frac{c\lambda_1^2}{d^2}\right)\big(\left(1+\frac{1}{d^2\beta} \right)^{2b\lambda_1} +  
d^2\beta\big) ,\left. \frac{c\beta^2 d^3 \exp\left( \frac{c\lambda_1^2}{d^2}\right)}{\delta_1}\right)
\end{align*}
we have that
  \begin{equation*}
    \frac{\E [\tilde{u}_1^\top H_nH_n^\top \tilde{u}_1]}{\|\tilde{u}_1\|_2^2} \geq (1-\delta_1) \exp\left(\sum_{t=1}^n 2\beta_t\lambda_1 - 4\beta_t^2\lambda_1^2\right),
  \end{equation*}
  where $c$ depends polynomially on $b, \beta, \lambda_1$.
\end{lemma}
\begin{proof}
  We consider the term $\E [\tilde{u}_1^\top H_nH_n^\top \tilde{u}_1]$ from Equation \eqref{eq:lemma3_rep},
  \begin{align*}
    \frac{\E [\tilde{u}_1^\top H_nH_n^\top \tilde{u}_1]}{\|\tilde{u}_1\|_2^2} &\geq \exp(\sum_{t=1}^n 2\beta_t\lambda_1 - 4\beta_t^2\lambda_1^2) - c d \sum_{t = 1}^n (\beta_t^2 r_t +\beta_t \mathbb{I}(t \leq t_0))\left(\frac{d^2\beta +t}{d^2\beta}\right)^{2b\lambda_1}\\
    &= (1-\delta_1) \exp\left(\sum_{t=1}^n 2\beta_t\lambda_1 - 4\beta_t^2\lambda_1^2\right)  - c d \sum_{t = 1}^n (\beta_t^2 r_t +\beta_t \mathbb{I}(t \leq t_0))\left(\frac{d^2\beta +t}{d^2\beta}\right)^{2b\lambda_1}\\
    &\quad + \delta_1 \exp\left(\sum_{t=1}^n 2\beta_t\lambda_1 - 4\beta_t^2\lambda_1^2\right)\\
    &\geq (1-\delta_1) \exp\left(\sum_{t=1}^n 2\beta_t\lambda_1 - 4\beta_t^2\lambda_1^2\right)  - \frac{c}{d^{4b\lambda_1-1}} \sum_{t = 1}^n (\beta_t^2 r_t +\beta_t \mathbb{I}(t \leq t_0))\left({d^2\beta +t}\right)^{2b\lambda_1}\\
    &\quad + \delta_1 \exp\left(-\frac{c'\lambda_1^2}{d^2}\right)\left( \frac{d^2\beta+n+1}{d^2\beta+1}\right)^{2b\lambda_1}\\
    &\geq (1-\delta_1) \exp\left(\sum_{t=1}^n 2\beta_t\lambda_1 - 4\beta_t^2\lambda_1^2\right) - \frac{c}{d^{4b\lambda_1-1}} \sum_{t = 1}^n (\beta_t^2 r_t )\left({d^2\beta +t}\right)^{2b\lambda_1}\\
    &\quad + \delta_1 \exp\left(-\frac{c'\lambda_1^2}{d^2}\right)\left( \frac{d^2\beta+n+1}{d^2\beta+1}\right)^{2b\lambda_1} - \frac{c}{d^{4b\lambda_1-1}} \sum_{t = 1}^{t_0} \left({d^2\beta +t}\right)^{2b\lambda_1-1}\\
    &\stackrel{\zeta_1}{\geq} (1-\delta_1) \exp\left(\sum_{t=1}^n 2\beta_t\lambda_1 - 4\beta_t^2\lambda_1^2\right) - \frac{c}{d^{4b\lambda_1-1}} \sum_{t = 1}^n (\beta_t^2 r_t )\left({d^2\beta +t}\right)^{2b\lambda_1}\\
    &\quad + \delta_1 \exp\left(-\frac{c'\lambda_1^2}{d^2}\right)\left( \frac{d^2\beta+n+1}{d^2\beta+1}\right)^{2b\lambda_1} - \frac{c}{2b\lambda_1 d^{4b\lambda_1-1}}  \left({d^2\beta +\log^3(d\beta)}\right)^{2b\lambda_1}\\
    &\stackrel{\zeta_2}{\geq} (1-\delta_1) \exp\left(\sum_{t=1}^n 2\beta_t\lambda_1 - 4\beta_t^2\lambda_1^2\right) - \frac{c}{d^{4b\lambda_1-1}} \sum_{t = 1}^n (\beta_t^2 r_t )\left({d^2\beta +t}\right)^{2b\lambda_1}\\
    &\quad + \delta_1 \exp\left(-\frac{c'\lambda_1^2}{d^2}\right)\left( \frac{d^2\beta+n+1}{d^2\beta+1}\right)^{2b\lambda_1} - \frac{c\beta^{2b\lambda_1}{d}^{4b\lambda_1}}{ d^{4b\lambda_1-1}}  \\
    &\geq (1-\delta_1) \exp\left(\sum_{t=1}^n 2\beta_t\lambda_1 - 4\beta_t^2\lambda_1^2\right) - \frac{c}{d^{4b\lambda_1-1}} \sum_{t = 1}^n \log^3(d^2\beta+t) \left({d^2\beta +t}\right)^{2b\lambda_1-2}\\
    &\quad + \delta_1 \exp\left(-\frac{c'\lambda_1^2}{d^2}\right)\left( \frac{d^2\beta+n+1}{d^2\beta+1}\right)^{2b\lambda_1} - c\beta^{2b\lambda_1}d \\
    &\geq (1-\delta_1) \exp\left(\sum_{t=1}^n 2\beta_t\lambda_1 - 4\beta_t^2\lambda_1^2\right) - \frac{c \log^3(d^2\beta+n)}{d^{4b\lambda_1-1}} \sum_{t = 1}^n \left({d^2\beta +t}\right)^{2b\lambda_1-2}\\
    &\quad + \delta_1 \exp\left(-\frac{c'\lambda_1^2}{d^2}\right)\left( \frac{d^2\beta+n+1}{d^2\beta+1}\right)^{2b\lambda_1} - c\beta^{2b\lambda_1}d,
  \end{align*}
where $\zeta_1$ from using $\sum_{i=1}^n i^\gamma \leq n^{\gamma+1}/\gamma+1$ for $\gamma > -1$ and $\zeta_2$ follows from the fact that $\log^3(x) \leq c x$. We now consider the following three cases:

\textbf{Case 1: $2b\lambda_1 < 1$}\\
In this case we can lower bound the term $\frac{\E [\tilde{u}_1^\top H_nH_n^\top \tilde{u}_1]}{\|\tilde{u}_1\|_2^2}$ as,
\begin{align*}
  \frac{\E [\tilde{u}_1^\top H_nH_n^\top \tilde{u}_1]}{\|\tilde{u}_1\|_2^2} &\geq (1-\delta_1) \exp\left(\sum_{t=1}^n 2\beta_t\lambda_1 - 4\beta_t^2\lambda_1^2\right) - \frac{c \log^3(d^2\beta+n)}{d^{4b\lambda_1-1}(d^2\beta)^{(1-2b\lambda_1)}} \\
  &\quad + \delta_1 \exp\left(-\frac{c'\lambda_1^2}{d^2}\right)\left( \frac{d^2\beta+n+1}{d^2\beta+1}\right)^{2b\lambda_1} - c\beta^{2b\lambda_1}d\\
  &\geq (1-\delta_1) \exp\left(\sum_{t=1}^n 2\beta_t\lambda_1 - 4\beta_t^2\lambda_1^2\right) - \frac{c \beta^{2b\lambda_1}\log^3(d^2\beta+n)}{d\beta} \\
  &\quad + \delta_1 \exp\left(-\frac{c'\lambda_1^2}{d^2}\right)\left( \frac{d^2\beta+n+1}{d^2\beta+1}\right)^{2b\lambda_1} - c\beta^{2b\lambda_1}d\\
  &\geq (1-\delta_1) \exp\left(\sum_{t=1}^n 2\beta_t\lambda_1 - 4\beta_t^2\lambda_1^2\right) - c d\beta^{2b\lambda_1}\log^3(d^2\beta+n) \\
  &\quad + \delta_1 \exp\left(-\frac{c'\lambda_1^2}{d^2}\right)\left( \frac{d^2\beta+n+1}{d^2\beta+1}\right)^{2b\lambda_1}\\
  &\stackrel{\zeta_1}{\geq} (1-\delta_1) \exp\left(\sum_{t=1}^n 2\beta_t\lambda_1 - 4\beta_t^2\lambda_1^2\right),
\end{align*}
where $\zeta_1$ follows by using that
\begin{equation*}
  \frac{d^2\beta+n}{\log^{3/2b\lambda_1}(d^2\beta+n)} \geq \left(\frac{cd}{\delta_1} \right)^{1/2b\lambda_1}(d^2\beta+1).
\end{equation*}
\textbf{Case 2: $2b\lambda_1 > 1$}\\
In this case, we can lower bound the term $\frac{\E [\tilde{u}_1^\top H_nH_n^\top \tilde{u}_1]}{\|\tilde{u}_1\|_2^2}$ as,
\begin{align*}
   \frac{\E [\tilde{u}_1^\top H_nH_n^\top \tilde{u}_1]}{\|\tilde{u}_1\|_2^2} &\geq (1-\delta_1) \exp\left(\sum_{t=1}^n 2\beta_t\lambda_1 - 4\beta_t^2\lambda_1^2\right) - \frac{c \log^3(d^2\beta+n)}{d^{4b\lambda_1-1}} \frac{(d^2\beta+n)^{2b\lambda_1-1}}{2b\lambda_1 -1}\\
    &\quad + \delta_1 \exp\left(-\frac{c'\lambda_1^2}{d^2}\right)\left( \frac{d^2\beta+n+1}{d^2\beta+1}\right)^{2b\lambda_1} - c\beta^{2b\lambda_1}d\\
    &\geq  \left( \frac{d^2\beta+n}{d^2\beta+1}\right)^{2b\lambda_1}\Big(\delta_1 \exp\left(-\frac{c'\lambda_1^2}{d^2}\right) - c\beta^{2b\lambda_1}d\left( \frac{d^2\beta+1}{d^2\beta+n}\right)\\
    &\quad - cd\beta^{2b\lambda_1}\left( 1+\frac{1}{d^2\beta}\right)^{2b\lambda_1} \frac{\log^3(d^2\beta+n)}{d^2\beta+n}\Big) + (1-\delta_1) \exp\left(\sum_{t=1}^n 2\beta_t\lambda_1 - 4\beta_t^2\lambda_1^2\right)\\
    &\stackrel{\zeta_1}{\geq}(1-\delta_1) \exp\left(\sum_{t=1}^n 2\beta_t\lambda_1 - 4\beta_t^2\lambda_1^2\right),
\end{align*}
where $\zeta_1$ follows by using that
\begin{equation*}
  \frac{d^2\beta+n}{\log^3(d^2\beta+n)} \geq \frac{cd\beta^{2b\lambda_1}}{\delta_1}\exp\left(\frac{c\lambda_1^2}{d^2}\right)\left(\left(1+\frac{1}{d^2\beta} \right)^{2b\lambda_1} + d^2\beta\right)
\end{equation*}
\textbf{Case 3: $2b\lambda_1 =1$}\\
In this case, we can lower bound the term $\frac{\E [\tilde{u}_1^\top H_nH_n^\top \tilde{u}_1]}{\|\tilde{u}_1\|_2^2}$ as,
\begin{align*}
    \frac{\E [\tilde{u}_1^\top H_nH_n^\top \tilde{u}_1]}{\|\tilde{u}_1\|_2^2} &\geq (1-\delta_1) \exp\left(\sum_{t=1}^n 2\beta_t\lambda_1 - 4\beta_t^2\lambda_1^2\right) - \frac{c \log^4(d^2\beta+n)}{d} \\
    &\quad + \delta_1 \exp\left(-\frac{c'\lambda_1^2}{d^2}\right)\left( \frac{d^2\beta+n+1}{d^2\beta+1}\right) - c\beta d\\
    &\stackrel{\zeta_1}{\geq} (1-\delta_1) \exp\left(\sum_{t=1}^n 2\beta_t\lambda_1 - 4\beta_t^2\lambda_1^2\right),
\end{align*}
where $\zeta_1$ follows from using
\begin{equation*}
  \frac{d^2\beta+n}{\log^4(d^2\beta+n)} \geq \frac{c\beta^2 d^3 \exp\left( \frac{c\lambda_1^2}{d^2}\right)}{\delta_1} .
\end{equation*}
\end{proof}

\begin{lemma}\label{lem:conv_var_perturb}
  For any $\delta_2 \in (0,1)$ and $n$ satisfying,
  \begin{equation*}
\frac{d^2\beta+n}{\log^{4\min(1, 1/4b\lambda_1)}(d^2\beta+n)} \geq \max\left(\frac{c(d^2\beta+1)}{(\delta_2\log^3(d\beta))^{\frac{1}{4b\lambda_1}}}, \frac{c^{4b\lambda_1}}{\delta_2}(d^2\beta+n)\right),
\end{equation*}
we have that,
  \begin{equation*}
  \frac{\mathbb{E}\left[ (\tilde{u}_1^\top H_nH_n^\top \tilde{u}_1 )^2  \right]}{\|\tilde{u}_1\|_2^4} \leq (1+\delta_2) \exp\left(\sum_{t=1}^n 4 \lambda_1 \beta_t + 11\lambda_1^2 \beta_t^2\right),
\end{equation*}
  where $c$ depends polynomially on $b, \beta, \lambda_1, \Delta_\lambda$.
\end{lemma}
\begin{proof}
  We consider the term $\mathbb{E}\left[ (\tilde{u}_1^\top H_nH_n^\top \tilde{u}_1 )^2  \right]$ from Equation \eqref{eq:lemma4_rep},
  \begin{align*}
    \frac{\mathbb{E}\left[ (\tilde{u}_1^\top H_nH_n^\top \tilde{u}_1 )^2  \right]}{\|\tilde{u}_1\|_2^4}
    &\leq \exp\left(\sum_{t=1}^n 4 \lambda_1 \beta_t + 11\lambda_1^2 \beta_t^2\right)
    + c\sum_{t = 1}^n (d\beta_t^2r_t + \beta_t\mathbb{I}(t\leq t_0))\left(\frac{d^2\beta +t}{d^2\beta}\right)^{4b\lambda_1}\\
    &= c\sum_{t = 1}^n (d\beta_t^2r_t + \beta_t\mathbb{I}(t\leq t_0))\left(\frac{d^2\beta +t}{d^2\beta}\right)^{4b\lambda_1} -\delta_2 \exp\left(\sum_{t=1}^n 4 \lambda_1 \beta_t + 11\lambda_1^2 \beta_t^2\right)\\
    &\quad + (1+\delta_2) \exp\left(\sum_{t=1}^n 4 \lambda_1 \beta_t + 11\lambda_1^2 \beta_t^2\right)\\
    &=c\sum_{t = 1}^n (d\beta_t^2r_t)\left(\frac{d^2\beta +t}{d^2\beta}\right)^{4b\lambda_1} -\delta_2 \exp\left(\sum_{t=1}^n 4 \lambda_1 \beta_t + 11\lambda_1^2 \beta_t^2\right)\\
    &\quad + cb\sum_{t = 1}^{t_0}\frac{1}{d^2\beta+t} \left(\frac{d^2\beta +t}{d^2\beta}\right)^{4b\lambda_1}
    + (1+\delta_2) \exp\left(\sum_{t=1}^n 4 \lambda_1 \beta_t + 11\lambda_1^2 \beta_t^2\right)\\
    &= c\sum_{t = 1}^n (d\beta_t^2r_t)\left(\frac{d^2\beta +t}{d^2\beta}\right)^{4b\lambda_1} -\delta_2 \exp\left(\sum_{t=1}^n 4 \lambda_1 \beta_t + 11\lambda_1^2 \beta_t^2\right)\\
    &\quad + \frac{cb}{(d^2\beta)^{4b\lambda_1}}\sum_{t = 1}^{t_0}{(d^2\beta+t)^{4b\lambda_1-1}}
    + (1+\delta_2) \exp\left(\sum_{t=1}^n 4 \lambda_1 \beta_t + 11\lambda_1^2 \beta_t^2\right)\\
    &\stackrel{\zeta_1}{\leq} c\sum_{t = 1}^n (d\beta_t^2r_t)\left(\frac{d^2\beta +t}{d^2\beta}\right)^{4b\lambda_1} -\delta_2 \exp\left(\sum_{t=1}^n 4 \lambda_1 \beta_t + 11\lambda_1^2 \beta_t^2\right) + \frac{c^{4b\lambda_1}}{\lambda_1}\\
    &\quad + (1+\delta_2) \exp\left(\sum_{t=1}^n 4 \lambda_1 \beta_t + 11\lambda_1^2 \beta_t^2\right)\\
    &\leq \frac{cdb^2}{(d^2\beta)^{4b\lambda_1}}\sum_{t = 1}^n \log^3(d^2\beta+t)(d^2\beta+t)^{4b\lambda_1-2} -\delta_2 \exp\left(\sum_{t=1}^n 4 \lambda_1 \beta_t + 11\lambda_1^2 \beta_t^2\right) \\
    &\quad + \frac{c^{4b\lambda_1}}{\lambda_1} + (1+\delta_2) \exp\left(\sum_{t=1}^n 4 \lambda_1 \beta_t + 11\lambda_1^2 \beta_t^2\right)\\
    &\leq \frac{cdb^2\log^3(d^2\beta +n)}{(d^2\beta)^{4b\lambda_1}}\sum_{t = 1}^n (d^2\beta+t)^{4b\lambda_1-2} -\delta_2 \exp\left(\sum_{t=1}^n 4 \lambda_1 \beta_t + 11\lambda_1^2 \beta_t^2\right) + \frac{c^{4b\lambda_1}}{\lambda_1}\\
    &\quad + (1+\delta_2) \exp\left(\sum_{t=1}^n 4 \lambda_1 \beta_t + 11\lambda_1^2 \beta_t^2\right)\\
    &\leq \frac{cdb^2\log^3(d^2\beta +n)}{(d^2\beta)^{4b\lambda_1}}\sum_{t = 1}^n (d^2\beta+t)^{4b\lambda_1-2} -\delta_2 \left(\frac{d^2\beta+n+1}{d^2\beta+1}\right)^{4b\lambda_1}
    + \frac{c^{4b\lambda_1}}{\lambda_1}\\
    &\quad + (1+\delta_2) \exp\left(\sum_{t=1}^n 4 \lambda_1 \beta_t + 11\lambda_1^2 \beta_t^2\right),
  \end{align*}
  where $\zeta_1$ follows by using the fact that $\sum_{i=1}^n i^\gamma \leq n^{\gamma+1}/\gamma+1$ for $\gamma > -1$. We consider now the following three cases as before:

  \textbf{Case 1: $4b\lambda_1 < 1$}\\
  In this case, we can upper bound the term $\frac{\mathbb{E}\left[ (\tilde{u}_1^\top H_nH_n^\top \tilde{u}_1 )^2  \right]}{\|\tilde{u}_1\|_2^4}$ as,
  \begin{align*}
    \frac{\mathbb{E}\left[ (\tilde{u}_1^\top H_nH_n^\top \tilde{u}_1 )^2  \right]}{\|\tilde{u}_1\|_2^4}
    &\leq \frac{cb^2\log^3(d^2\beta+n)}{d\beta} -\delta_2 \left(\frac{d^2\beta+n+1}{d^2\beta+1}\right)^{4b\lambda_1}
    + \frac{c^{4b\lambda_1}}{\lambda_1}\\
    &\quad + (1+\delta_2) \exp\left(\sum_{t=1}^n 4 \lambda_1 \beta_t + 11\lambda_1^2 \beta_t^2\right)\\
    &\stackrel{\zeta_1}{\leq} (1+\delta_2) \exp\left(\sum_{t=1}^n 4 \lambda_1 \beta_t + 11\lambda_1^2 \beta_t^2\right),
  \end{align*}
  where $\zeta_1$ follows from using that
  \begin{equation*}
    \frac{d^2\beta + n}{\log^{\frac{3}{4b\lambda_1}}(d^2\beta+n)} \geq \frac{c(d^2\beta+1)}{(\delta_2\log^3(d\beta))^{\frac{1}{4b\lambda_1}}}.
  \end{equation*}

  \textbf{Case 2: $4b\lambda_1 > 1$}\\
  In this case, we can upper bound the term $\frac{\mathbb{E}\left[ (\tilde{u}_1^\top H_nH_n^\top \tilde{u}_1 )^2  \right]}{\|\tilde{u}_1\|_2^4}$ as,
  \begin{align*}
      \frac{\mathbb{E}\left[ (\tilde{u}_1^\top H_nH_n^\top \tilde{u}_1 )^2  \right]}{\|\tilde{u}_1\|_2^4}
      &\leq \frac{cdb^2\log^3(d^2\beta +n)}{(d^2\beta)^{4b\lambda_1}}\sum_{t = 1}^n (d^2\beta+t)^{4b\lambda_1-2} -\delta_2 \left(\frac{d^2\beta+n+1}{d^2\beta+1}\right)^{4b\lambda_1}
      + \frac{c^{4b\lambda_1}}{\lambda_1}\\
      &\quad + (1+\delta_2) \exp\left(\sum_{t=1}^n 4 \lambda_1 \beta_t + 11\lambda_1^2 \beta_t^2\right)\\
      &\stackrel{\zeta_1}{\leq} \frac{cdb^2\log^3(d^2\beta +n)}{(d^2\beta)^{4b\lambda_1}}(d^2\beta+n)^{4b\lambda_1 -1} -\delta_2 \left(\frac{d^2\beta+n+1}{d^2\beta+1}\right)^{4b\lambda_1}
      + \frac{c^{4b\lambda_1}}{\lambda_1}\\
      &\quad + (1+\delta_2) \exp\left(\sum_{t=1}^n 4 \lambda_1 \beta_t + 11\lambda_1^2 \beta_t^2\right)\\
      &\stackrel{\zeta_2}{\leq}(1+\delta_2) \exp\left(\sum_{t=1}^n 4 \lambda_1 \beta_t + 11\lambda_1^2 \beta_t^2\right),
  \end{align*}
  where $\zeta_2$ follows by using that
  \begin{equation*}
    \frac{d^2\beta+n}{\log^3(d^2\beta+n)} \geq \frac{c^{4b\lambda_1}}{\delta_2}(d^2\beta+1).
  \end{equation*}

  \textbf{Case 3: $4b\lambda_1=1$}\\
  In this case, we can upper bound the term $\frac{\mathbb{E}\left[ (\tilde{u}_1^\top H_nH_n^\top \tilde{u}_1 )^2  \right]}{\|\tilde{u}_1\|_2^4}$ as,
  \begin{align*}
      \frac{\mathbb{E}\left[ (\tilde{u}_1^\top H_nH_n^\top \tilde{u}_1 )^2  \right]}{\|\tilde{u}_1\|_2^4}
      &\leq \frac{cdb^2\log^3(d^2\beta +n)}{(d^2\beta)}\sum_{t = 1}^n (d^2\beta+t)^{-1} -\delta_2 \left(\frac{d^2\beta+n+1}{d^2\beta+1}\right)
      + \frac{c}{\lambda_1}\\
      &\quad + (1+\delta_2) \exp\left(\sum_{t=1}^n 4 \lambda_1 \beta_t + 11\lambda_1^2 \beta_t^2\right)\\
      &\leq \frac{cdb^2\log^3(d^2\beta +n)}{(d^2\beta)}\log\left(\frac{d^2\beta+n}{d^2\beta}\right) -\delta_2 \left(\frac{d^2\beta+n+1}{d^2\beta+1}\right)
      + \frac{c}{\lambda_1}\\
      &\quad + (1+\delta_2) \exp\left(\sum_{t=1}^n 4 \lambda_1 \beta_t + 11\lambda_1^2 \beta_t^2\right)\\
      &\stackrel{\zeta_1}{\leq} (1+\delta_2) \exp\left(\sum_{t=1}^n 4 \lambda_1 \beta_t + 11\lambda_1^2 \beta_t^2\right),
  \end{align*}
  where $\zeta_1$ holds due to
  \begin{equation*}
    \frac{d^2\beta+n}{\log^{4}(d^2\beta+n)} \geq \frac{cd}{\delta_2}.
  \end{equation*}
\end{proof}

\subsection{Convergence Theorem}
We begin by restating the bound obtained on $\mathbb{E}\left[ \tr(V_\perp^\top H_t H_t^\top V_\perp)\right]$ in Lemma \ref{lemma2},
{
\begin{align}\label{eq:lemma2_rep}
&\mathbb{E}\left[ \tr(V_\perp^\top H_n H_n^\top V_\perp)\right]\leq \exp \left(\sum_{t = 1}^n 2\beta_t\lambda_2 + \beta_t^2 \lambda_2^2 \right)\nonumber\\
&\quad \left(\tr(V_\perp V_\perp^\top) + cd\| V_\perp V_\perp^\top\|_2 \sum_{t=1}^n \left(r_t\beta_t^2 + \mathbb{I}(t\leq t_0)\beta_td \right)\exp\left(2\sum_{i=1}^t\beta_i(\lambda_1-\lambda_2) + cd\beta_i^2r_i + c\sum_{i=1}^{t_0}\beta_i d \right)\right)\nonumber \\
&\stackrel{\zeta_1}{\leq}\exp \left(\sum_{t = 1}^n 2\beta_t\lambda_2 + \beta_t^2 \lambda_2^2 \right)\Big(\tr(V_\perp V_\perp^\top) \\
&\quad + cd\| V_\perp V_\perp^\top\|_2 \sum_{t=1}^n \left(r_t\beta_t^2 + \mathbb{I}(t\leq t_0)\beta_td \right)\exp\left(2\sum_{i=1}^t\beta_i(\lambda_1-\lambda_2) + cd\beta_i^2r_i \right)\Big),
\end{align}
}
where $\zeta_1$ follows from using Equation \eqref{eq:t0_bound}.
\begin{theorem}[Convergence Theorem]\label{thm:conv_pn}
Let $\delta>0$ and the step sizes $\beta_i = \frac{b}{d^2\beta + i}$. The output $v_n$ of Algorithm \ref{alg:spge} for $n$ satisfying the assumption in Lemma \ref{lem:conv_mean_perturb} and \ref{lem:conv_var_perturb} is an $\epsilon$-approximation to $u_1$ with probability atleast $1-\delta$ where,
\begin{align*}
  \hspace{-6ex}\underbrace{\sin_B^2(u_1, v_n)}_{\epsilon}  &\leq \frac{d\|V_{\perp}V_{\perp}^\top\|_2}{Q}\exp\left(5\lambda_1^2\sum_{t=1}^n\beta_t^2\right)\Big(\exp\left(-2\Delta_\lambda\sum_{t=1}^n \beta_t \right) \\
  &\quad+ c\sum_{t=1}^n \left(r_t\beta_t^2 + \mathbb{I}(t\leq t_0)\beta_td \right)\exp\left( -2\Delta_\lambda\sum_{i=t+1}^n\beta_i\right)\Big),
\end{align*}
where $\Delta_\lambda = \lambda_1 - \lambda_2$ and
\begin{equation*}
  Q = \frac{2\delta^2\|\tilde{u}_1\|_2^2}{(2+\epsilon_1)c\log(1/\delta)}\left( 1-\frac{1}{\sqrt{\delta}}\sqrt{ (1+\epsilon_1)\exp{\left(19\sum_{t=1}^n \beta_t^2 \lambda_1^2 \right)}-1}\right),
\end{equation*}
The constant $c$ occuring in the equations, as before depends polynomially on problem dependent paramters $b, \lambda_1, \Delta_{\lambda}$ and the parameters $\frac{\delta_1}{2} = \delta_2 = \frac{\epsilon_1}{2+\epsilon_1}$.
\end{theorem}
\begin{proof}
  First, using the Chebychev's inequality, we have:
  \begin{equation*}
    \mathbb{P}\left[\left\vert \tilde{u}_1^\top H_nH_n^\top \tilde{u}_1 - \E\left[\tilde{u}_1^\top H_nH_n^\top \tilde{u}_1\right] \right\vert \geq\frac{1}{\delta}\sqrt{\text{Var}[\tilde{u}_1^\top H_nH_n^\top \tilde{u}_1]} \right] \leq \delta.
  \end{equation*}
  With probability greater than $1-\delta$, we have,
  \begin{align}
    &\tilde{u}_1^\top H_nH_n^\top \tilde{u}_1 \geq \E\left[\tilde{u}_1^\top H_nH_n^\top \tilde{u}_1\right] -\frac{1}{\sqrt{\delta}} \sqrt{\text{Var}\left[\tilde{u}_1^\top H_nH_n^\top \tilde{u}_1\right]}\nonumber\\
    &\quad= \E\left[\tilde{u}_1^\top H_nH_n^\top \tilde{u}_1 \right]\left(1- \frac{1}{\sqrt{\delta}}\sqrt{\frac{\E[(\tilde{u}_1^\top H_nH_n^\top \tilde{u}_1)^2 ]}{\E[\tilde{u}_1^\top H_nH_n^\top \tilde{u}_1]^2} -1} \right)\label{eq:conv_bnd_denom_pn}
  \end{align}
  Now, using Lemma \ref{lem:conv_var_perturb}, we have that,
  \begin{equation}
    \label{eq:bnd_var_pn}
  \frac{\mathbb{E}\left[ (\tilde{u}_1^\top H_nH_n^\top \tilde{u}_1 )^2  \right]}{\|\tilde{u}_1\|_2^4} \leq (1+\delta_2) \exp\left(\sum_{t=1}^n 4 \lambda_1 \beta_t + 11\lambda_1^2 \beta_t^2\right)
  \end{equation}
  and using Lemma \ref{lem:conv_mean_perturb}, we have,
  \begin{equation*}
    \frac{\E [\tilde{u}_1^\top H_nH_n^\top \tilde{u}_1]}{\|\tilde{u}_1\|_2^2} \geq (1-\delta_1) \exp\left(\sum_{t=1}^n 2\beta_t\lambda_1 - 4\beta_t^2\lambda_1^2\right),
  \end{equation*}
  squaring the above, we obtain,
  \begin{equation}\label{eq:bnd_mean_square_pn}
    \frac{\E [\tilde{u}_1^\top H_nH_n^\top \tilde{u}_1]^2}{\|\tilde{u}_1\|_2^4} \geq (1-\delta_1') \exp\left(\sum_{t=1}^n 4\beta_t\lambda_1 - 8\beta_t^2\lambda_1^2\right),
  \end{equation}
  where $\delta_1' = 2\delta_1$. Setting $\delta_1'=\delta_2 = \frac{\epsilon_1}{2+\epsilon_1}$ and substituting bounds \eqref{eq:bnd_var_pn} and \eqref{eq:bnd_mean_square_pn} in \eqref{eq:conv_bnd_denom_pn}, we obtain,
  \begin{equation*}
    \tilde{u}_1^\top H_nH_n^\top \tilde{u}_1 \geq \frac{2\| \tilde{u}_1\|_2^2}{2+\epsilon_1}\exp{\left(\sum_{t=1}^n 2\beta_t \lambda_1 - 4\beta_t^2 \lambda_1^2\right)}\left( 1-\frac{1}{\sqrt{\delta}}\sqrt{ (1+\epsilon_1)\exp{\left(19\sum_{t=1}^n \beta_t^2 \lambda_1^2 \right)}-1}\right).
  \end{equation*}
  Further, using the Equation \eqref{eq:lemma2_rep} along with Markov's inequality, we have with probability atleast $1-\delta$
\begin{align*}
\tr(V_\perp^\top H_n H_n^\top V_\perp)&\leq \frac{1}{\delta}\exp \left(\sum_{t = 1}^n 2\beta_t\lambda_2 + \beta_t^2 \lambda_2^2 \right)\Big(\tr(V_\perp V_\perp^\top) \\
&\quad + cd\| V_\perp V_\perp^\top\|_2 \sum_{t=1}^n \left(r_t\beta_t^2 + \mathbb{I}(t\leq t_0)\beta_td \right)\exp\left(2\sum_{i=1}^t\beta_i(\lambda_1-\lambda_2) + cd\beta_i^2r_i \right)\Big).
\end{align*}
Combining the above with Lemma \ref{lem:sin}, we have that the output $v_n$ of Algorithm \ref{alg:spge} is an $\epsilon$-approximation to $u_1$ with probability atleast $1-\delta$,
\begin{align*}
  \epsilon &\leq \frac{c\log(1/\delta)(2+\epsilon_1)}{2\delta \|\tilde{u}_1 \|_2^2}\frac{\exp{\left(\sum_{t=1}^n -2\beta_t \lambda_1 + 4\beta_t^2 \lambda_1^2\right)}\tr(V_\perp^\top H_nH_nV_\perp)}{\left( 1-\frac{1}{\sqrt{\delta}}\sqrt{ (1+\epsilon_1)\exp{\left(19\sum_{t=1}^n \beta_t^2 \lambda_1^2 \right)}-1}\right)}\\
  &\leq \frac{d\|V_{\perp}V_{\perp}^\top\|_2}{Q}\exp\left(5\lambda_1^2\sum_{t=1}^n\beta_t^2\right)\Big(\exp\left(-2\Delta_\lambda\sum_{t=1}^n \beta_t \right) \\
  &\quad + c\sum_{t=1}^n \left(r_t\beta_t^2 + \mathbb{I}(t\leq t_0)\beta_td \right)\exp\left( -2\Delta_\lambda\sum_{i=t+1}^n\beta_i\right)\Big),
\end{align*}
where $\Delta_\lambda = \lambda_1 - \lambda_2$ and
\begin{equation*}
  Q = \frac{2\delta^2\|\tilde{u}_1\|_2^2}{(2+\epsilon_1)c\log(1/\delta)}\left( 1-\frac{1}{\sqrt{\delta}}\sqrt{ (1+\epsilon_1)\exp{\left(19\sum_{t=1}^n \beta_t^2 \lambda_1^2 \right)}-1}\right).
\end{equation*}
\end{proof}


\subsection{Main Result} 
In this section, we state our main theorem and instantiate the parameters of our algorithm.
\begin{theorem}[Main Result]\label{thm:main_thm_pn}
Fix any $\delta > 0$ and $\epsilon_1 > 0$. Suppose that the step sizes are set to $\alpha_t = \frac{c}{\log(d^2\beta+t)} $ and $\beta_t = \frac{\gamma}{\Delta_\lambda(d^2\beta+t)}$ for $\gamma > 1/2$ and
\begin{equation*}
  \beta = \max \left(\frac{20\gamma^2 \lambda_1^2}{\Delta_\lambda^2d^2\log\left(\frac{1+\delta/100}{1+\epsilon_1} \right)}, \frac{200\left(     \frac{R}{\mu}+\frac{R^3}{\mu^2}+ \frac{R^5}{\mu^3}\right)\log(    1+\frac{R^2}{\mu}+ \frac{R^4}{\mu^2}    )}{\delta \Delta_\lambda^2} \right).
\end{equation*}
Suppose that the number of samples $n$ satisfy the assumptions of Lemma  \ref{lem:conv_mean_perturb} and \ref{lem:conv_var_perturb}. Then, the output $v_n$ of Algorithm \ref{alg:spge} satisfies,
  \begin{align*}
\sin_B^2(u_1, v_n) &\leq \frac{(2+\epsilon_1)cd \|\sum_{i=1}^d\tilde{u}_i\tilde{u}_i^\top \|_2\log\left(\frac{1}{\delta}\right)}{\delta^2\|\tilde{u}_1 \|_2^2}\Big( \left(\frac{d^2\beta+1}{d^2\beta+n+1} \right)^{2\gamma} + \frac{c\gamma^2\log^3(d^2\beta+n)}{\Delta_\lambda^2(d^2\beta+n+1)} \\
&\quad + \frac{cd}{\Delta_\lambda}\left(\frac{d^2\beta + \log^3(d^2\beta)}{d^2\beta + n +1} \right)^{2\gamma}\Big),
  \end{align*}
  with probability at least $1-\delta$ with $c$ depending polynomially on parameters of the problem $\lambda_1, \kappa_B, R, \mu$. The parameters  $\delta_1, \delta_2$ are set as $\delta_1 = \frac{\epsilon_1}{2(2+\epsilon_1)}$ and $\delta_2 = \frac{\epsilon_1}{2+\epsilon_1}$.
\end{theorem}
\begin{proof}
  With the step size $\beta_t = \frac{b}{d^2\beta +t}$, we set the parameter $b = \frac{\gamma}{\lambda_1 - \lambda_2}$ and thus we get $\beta_t = \frac{\gamma}{\Delta_{\lambda}(d^2\beta+t)}$. Now, we have that
  \begin{equation*}
    \sum_{t = 1}^n \beta_t^2 \leq \frac{\gamma^2}{\Delta_\lambda^2d^2\beta}
  \end{equation*}
  and using the assumption that $\frac{\gamma^2\lambda_1^2}{\Delta_\lambda^2d^2\beta} \leq \frac{1}{19}\log\left(\frac{1+\frac{\delta}{100}}{1+\epsilon_1} \right)$, we obtain,
  \begin{equation}\label{eq:main_bnd_0}
    \sqrt{((1+\epsilon_1)\exp\left(19\sum_{t=1}^n \beta_t^2\lambda_1^2 \right)-1)} \geq \frac{9}{10}\qquad \Rightarrow \qquad Q \geq \frac{c\delta^2\|\tilde{u}_1\|_2^2}{(2+\epsilon_1)\log(1/\delta)}.
  \end{equation}
  Using previous bounds on sums of partial harmonic sums, we have that,
  \begin{equation*}
    \sum_{t =1}^n \beta_t \geq \frac{\gamma}{\Delta_\lambda}\log\left( \frac{d^2\beta+n+1}{d^2\beta+1}\right)\qquad\text{and}\qquad\sum_{i =t+1}^n \beta_i \geq \frac{\gamma}{\Delta_\lambda}\log\left( \frac{d^2\beta+n+1}{d^2\beta+t+1}\right).
  \end{equation*}
  Using these bounds, we obtain,
  \begin{equation}\label{eq:main_bnd_1}
    \exp\left(-2\Delta_\lambda \sum_{t=1}^n\beta_t\right) \leq \left( \frac{d^2\beta+1}{d^2\beta+n+1}\right)^{2\gamma}.
  \end{equation}
  In order to bound the remaining terms from Theorem \ref{thm:conv_pn}, we note that,
  \begin{align}\label{eq:main_bnd_2}
    &c\sum_{t=1}^n \left(r_t\beta_t^2 + \mathbb{I}(t\leq t_0)\beta_td \right)\exp\left( -2\Delta_\lambda\sum_{i=t+1}^n\beta_i\right)\nonumber \\
    &\quad \leq c\sum_{t=1}^n \left(r_t\beta_t^2 + \mathbb{I}(t\leq t_0)\beta_td \right)\left( \frac{d^2\beta+t+1}{d^2\beta + n + 1} \right)^{2\gamma}\nonumber \\
    &\quad \leq c\sum_{t=1}^n\frac{r_t\gamma^2}{(\Delta_\lambda)^2 (d^2\beta+t)^2} \left(\frac{d^2\beta + t+1}{d^2\beta + n+1} \right)^{2\gamma} + cd\sum_{t=1}^{t_0}\frac{\gamma}{\Delta_\lambda (d^2\beta+t)} \left(\frac{d^2\beta + t+1}{d^2\beta + n+1} \right)^{2\gamma}\nonumber\\
    &\quad \leq \frac{c\gamma^2\log^3(d^2\beta+n)}{\Delta_\lambda^2(2\gamma-1)(d^2\beta+n+1)} + \frac{cd}{\Delta_\lambda}\left(\frac{d^2\beta + \log^3(d^2\beta)}{d^2\beta + n +1} \right)^{2\gamma},
  \end{align}
  where the last bounds holds for any $\gamma > 1/2$. Substituting bounds \eqref{eq:main_bnd_0},\eqref{eq:main_bnd_1} and \eqref{eq:main_bnd_2} in the result of Theorem \ref{thm:conv_pn}, we obtain that the output $v_n$ of Algorithm \ref{alg:spge} satisfies,
  \begin{align*}
    \sin_B^2(u_1, v_n) &\leq \frac{(2+\epsilon_1)cd \|\sum_{i=1}^d\tilde{u}_i\tilde{u}_i^\top \|_2\log\left(\frac{1}{\delta}\right)}{\delta^2\|\tilde{u}_1 \|_2^2}\Big( \left(\frac{d^2\beta+1}{d^2\beta+n+1} \right)^{2\gamma} + \frac{c\gamma^2\log^3(d^2\beta+n)}{\Delta_\lambda^2(d^2\beta+n+1)} \\
    &\quad + \frac{cd}{\Delta_\lambda}\left(\frac{d^2\beta + \log^3(d^2\beta)}{d^2\beta + n +1} \right)^{2\gamma}\Big).
  \end{align*}
\end{proof}

\section{Auxiliary Properties}
\subsection{Useful Trace Inequalities}
In this section we enumerate some useful inequalities.
\begin{lemma}
\begin{enumerate}
\item $\langle A, B \rangle  \leq \langle A, C \rangle$ for PSD matrices $A,B,C$ with $B \preceq C$.
\item $\tr( A^\top B) \leq \frac{1}{2}\tr(A^\top A + B^\top B)$ for all matrices $A,B \in \mathbb{R}^{m\times n}$.
\end{enumerate}
\end{lemma}

As a consequence:
\begin{corollary}\label{corollary::trace_cor1}
$\langle A, B \rangle  \leq \langle A, C \rangle$ for a PSD matrix $A$ and $B \preceq C$, with $B$ and $C$ symmetric.
\end{corollary}
\begin{proof}
If $B$ is PSD, the result follows immediately from the previous lemma. Otherwise let $\lambda_{min }$ be the smallest eigenvalue of $B$. Let $B' = B + |\lambda_{min}| I $ and $C' = C+ | \lambda_{min}|I$. The matrices $B'$ and $C'$ are PSD and satisfy $B' \preceq C'$. The result follows by applying the lemma above and rearranging the terms.
\end{proof}

\subsection{Useful spectral norm Inequalities}
In this section we enumerate some useful inequalities.

\begin{lemma}
If $0 \preceq B \preceq C$ and symmetric then $ 0\preceq ABA^\top  \preceq  A C A^\top$.
\end{lemma}
As a consequence:

\begin{corollary}
If $0 \preceq B \preceq C$ and symmetric then $\| ABA^\top \| \leq \|ACA^\top\|$.
\end{corollary}

\subsection{Properties concerning Eigenvectors of $B^{-1}A$}
\label{sec:eigenprop}
In this subsection, we highlight some important properties concerning the left and right eigenvectors of the matrix under consideration $B^{-1}A$.

As before, we let $\tilde{u}_1, \ldots , \tilde{u}_d$ denote the left eigenvetors and $u_1, \ldots, u_d$ denote the right eigenvectors of $B^{-1}A$.

\begin{lemma}\label{lemma::right_eigenvectors_orthogonal}
The right eigenvectors of the matrix $B^{-1}A$ satisfy the following:
\begin{equation*}
u_i^\top B u_j =0 \quad\text{ if  } i \neq j.
\end{equation*}
\end{lemma}

\begin{proof}
Consider the symmetric matrix $C = B^{-1/2}AB^{-1/2}$. Let $u_1^C, \ldots, u_d^C$ be the eigenvectors of $C$. Notice that if $u_i^C$ is an eigenvector of $C$ with eigenvalue $\lambda_i$, then
\begin{equation*}
B^{-1/2} ( B^{-1/2}A B^{-1/2}) u_i^C = \lambda_i B^{-1/2}  u_i^C,
\end{equation*}
implying that $B^{-1/2}u_i^C$ is a right eigenvector of $B^{-1}A$, $u_i$.
Therefore the eigenvector of $C$ are related to the righteigenvectors of $B^{-1}A$ as $B^{1/2}u_i = u_i^C$. Further, since the matrix $C$ is symmetric, its eigencvectors can be taken to form an orthogonal basis, and hence,
\begin{equation*}
(u_i^C)^\top u_j^C = u_i^\top B u_j = 0 \quad\text{ if  } i \neq j.
\end{equation*}
\end{proof}

\begin{lemma}\label{lemma::orthogonality_left_right}
Let $u_1$ denote the top right eigenvector of $B^{-1}A$. Then,
\begin{equation*}
\tilde{u}_i^\top u_1 = 0 \quad \text{ for all } i \geq 2,
\end{equation*}
where $\tilde{u}_i$ represent the left eigenvectors of the matrix $B^{-1}A$.
\end{lemma}

\begin{proof}
We begin by noting that the left and right eigenvectors of the matrix $B^{-1}A$ are related as $\tilde{u}_i = Bu_i$, which follows from,
\begin{equation*}
 (B^{-1}A)B^{-1} \tilde{u}_i = B^{-1}(A B^{-1}) \tilde{u}_i = \lambda_i B^{-1} \tilde{u}_i
\end{equation*}
As a consequence $B^{-1} \tilde{u}_i$ is a right eigenvector of $B^{-1}A$ and the lemma now follows from using Lemma \ref{lemma::right_eigenvectors_orthogonal}.

\end{proof}

As a consequence of Lemma \ref{lemma::orthogonality_left_right}, we have the following corollary relating the orthogonal subspace of $u_1$ to the left eigenvectors $\tilde{u}_2, \ldots, \tilde{u}_d$.

\begin{corollary}\label{eigenproperties::corollary_v_perp}
If $\lambda_1$ has multiplicity $1$, the space orthogonal to $u_1$ is spanned by the vectors $\{ \tilde{u}_2,\ldots, \tilde{u}_d \}$.
\end{corollary}

\end{document}